\documentclass{informs3}

\OneAndAHalfSpacedXI

\usepackage{xcolor}

\usepackage[colorlinks=true, allcolors=blue]{hyperref}
\usepackage{booktabs}

\usepackage{graphicx}
\usepackage{subcaption}
\usepackage{enumitem}

\usepackage{float}

\newcommand\E{\mathbb{E}}        \usepackage{xspace}

\usepackage{multirow}
\usepackage{natbib}
\usepackage{array}
\usepackage{colortbl}
\usepackage{arydshln}
 \bibpunct[, ]{(}{)}{,}{a}{}{,}%
 \def\bibfont{\small}%

\TheoremsNumberedThrough    
\ECRepeatTheorems

\EquationsNumberedThrough    
\MANUSCRIPTNO{}

\begin{document}

\RUNAUTHOR{Li et al.} 

\RUNTITLE{Robust Learning under Distributional Shifts}

\TITLE{Statistical Properties of Robust Learning under Distributional Shifts}

\ARTICLEAUTHORS{%
\AUTHOR{Zhiyi Li}
\AFF{National University of Singapore, \EMAIL{e1632488@u.nus.edu}} 
\AUTHOR{Xiaojie Mao}
\AFF{Tsinghua University, \EMAIL{maoxj@sem.tsinghua.edu.cn}}
\AUTHOR{Yunbei Xu}
\AFF{National University of Singapore, \EMAIL{yunbei@nus.edu.sg}}
\AUTHOR{Ruohan Zhan}
\AFF{University College London, \EMAIL{ruohan.zhan@ucl.ac.uk}}
} 

\ABSTRACT{%
Distributional shifts arise when the target deployment environment differs from the source environment that generated the training data. Robust learning frameworks such as Distributionally Robust Optimization (DRO) and Robust Satisficing (RS) aim to address this challenge, yet their finite-sample guarantees under such shifts, and their systematic comparison, remain underexplored: existing analyses typically establish guarantees either in the source environment or for adversarial worst-case performance over an ambiguity set.
This paper instead studies  generalization error in the target environment---the excess loss under the shifted target distribution. Our contributions are threefold. First, we derive finite-sample generalization error bounds in the shifted target environment for both DRO and RS. These bounds explicitly characterize the trade-off between reduced sensitivity to shift and the regularization penalty induced by each method's robustness hyperparameter, and they 
 avoid the curse of dimensionality associated with Wasserstein empirical concentration.
Second, when partial shift information such as shift magnitude or direction is available, we propose information-directed hyperparameter calibrations and compare the two methods given the same information. Under these calibrations, and in the partial-information regimes we study, DRO and RS exhibit complementary theoretical and empirical behavior.
Finally, we apply the  framework to a network lot-sizing problem, using it to interpret how robust policies respond to positive shifits in the demand distribution. 
Together, these results fill a gap in understanding the statistical properties of robust learning methods under distributional shifts and provide a principled basis for comparing DRO and RS.
}%

\KEYWORDS{Distributional Shifts; Distributionally Robust Optimization; Robust Satisficing; Finite-Sample Generalization Error Bounds; Hyperparameter Calibration}

\maketitle

\section{Introduction}
Modern machine-learning and data-driven decision systems are often trained in
one source environment and deployed in a different target environment. When the corresponding distributions differ, the resulting distribution shift may
take the form of covariate shift \citep{quinonero2022dataset}, label shift
\citep{garg2020unified}, or temporal and environmental change
\citep{sugiyama2012machine}. Standard learning methods that optimize
performance under the source distribution, such as Empirical Risk Minimization
(ERM), offers no protection against this mismatch. Distribution shift
can substantially degrade target performance \citep{koh2021wilds} in critical
domains, including medical imaging diagnostics
\citep{zech2018variable,yu2022external} and autonomous-driving perception 
\citep{sakaridis2021acdc,dong2023robustness}. More broadly, uncertainty about
the deployment distribution affects operational decisions in renewable-energy
systems \citep{li2023data,huang2023distributionally}, supply-chain and
disaster-relief planning
\citep{chen2023designing,deng2023distributionally,wang2023risk}, and
fairness-sensitive learning under changing group representation
\citep{hashimoto2018fairness,sagawa2020distributionally}.

Robust learning methods seek to protect decisions against such  distribution shifts. Distributionally Robust Optimization (DRO) stands out as a widely studied approach \citep[e.g.,][]{hu2013kullback,bayraksan2015data,long2023robust,gao2023finite,gao2023distributionally,esfahani2015data,blanchet2019robust,lee2018minimax,an2021generalization,wang2023equivalence,li2024statistical,lam2019recovering}. It minimizes the worst-case expected loss over distributions within a prescribed Wasserstein radius \(r\) of the empirical source  distribution.
Thus, \(r\) specifies the largest departure from the source distribution that DRO treats as plausible and protects against: a small \(r\) may leave the target distribution outside the set, whereas a large \(r\) guards against a broader range of shifts at the cost of greater conservatism.
A recently proposed method, Robust Satisficing (RS) \citep{long2023robust}, uses a different robustness semantic. 
Rather than specifying a set of plausible shifted distributions, RS fixes a reference expected-loss level $\tau$. 
It requires the expected loss under any distribution exceeding this baseline to grow at most proportionally with the distribution’s distance from a nominal distribution; the proportionality constant \(k_\tau\), called the fragility measure, captures the resulting sensitivity to distributional shift. Hence, \(\tau\) sets the baseline performance requirement: increasing \(\tau\) relaxes this requirement and can reduce fragility, while decreasing \(\tau\) demands a lower baseline loss but generally entails greater fragility.

To assess  robust learning methods, empirical studies often evaluate their performance in the target environment \citep{liu2024rethinking,koh2021wilds,gulrajani2021search}. For example, \citet{liu2024rethinking} benchmark DRO methods in target environments induced by diverse types of distribution shifts. 
In contrast,   established theory has focused less on target-environment guarantees. As summarized by \citet{blanchet2025distributionally}, existing DRO bounds control either no-shift generalization in the source environment or adversarial worst-case performance over an ambiguity set.
The resulting DRO generalization bounds are monotone in the Wasserstein radius \(r\), so enlarging the ambiguity set only worsens the bound, which obscures the practical robustness-conservatism trade-off induced by setting the radius $r$.
 For RS, \citet{li2024statistical} provide a recent target-environment analysis, but their bound suffers from the curse of dimensionality associated with Wasserstein empirical concentration and likewise cannot show the \(\tau\)-driven trade-off.

This work addresses the gap by developing finite-sample bounds for target-environment generalization guarantees under distributional shifts. The bounds avoid the curse of dimensionality associated with the Wasserstein empirical concentration. They also explicitly characterize how the hyperparameters of DRO and RS trade off robustness against conservatism. Specifically, we use target-environment generalization error, namely excess loss under the shifted distribution, as the performance metric and derive finite-sample upper bounds for both methods. These bounds show how each hyperparameter balances (i) reduced sensitivity to distribution shift, which improves target-environment performance, against (ii) the regularization cost of enforcing stronger robustness.
Moreover, because $r$ and $\tau$ control robustness in fundamentally different ways, our results reveal the distinct mechanisms through which DRO and RS handle distributional shifts.

As a benchmark for our robustness results,  we also analyze the standard ERM  method. We derive an upper bound on its  generalization error in the target environment, and complement this with a matching  example showing that the linear dependence on the source-target shift magnitude cannot be improved. 
Against this baseline, we show that robust methods do better: for DRO and RS, the shift-induced terms in the target-environment generalization error admit more favorable characterizations than their ERM counterparts, through two distinct mechanisms detailed in Section~\ref{subsec contributions}.
These results differ in scope from \citet{esfahani2015data,blanchet2019robust,gao2023finite}, which  study the no-shift setting; our bounds instead are explicitly shift-aware and reveal the potential  advantage of robust methods in the target environment. 
This distinction highlights both the novelty of our results and the technical challenges underlying the analysis.

We next compare DRO and RS under distributional shifts. A fair comparison is challenging because the two paradigms encode robustness through different hyperparameter semantics: the DRO radius \(r\) represents distributional coverage, whereas the RS threshold \(\tau\) represents a loss tolerance. Moreover, when no shift information is available beyond the training sample, reliable learning under arbitrary unknown shifts is generally impossible~\citep{sutter2021robust}. 
We therefore introduce partial shift information, such as its magnitude or direction, and ask how the two  methods exploit this information. 
 Specifically, we construct a shift-information-directed  calibration that maps the same input into the DRO radius \(r\) and the RS threshold \(\tau\). This differs from the optimization-based hyperparameter correspondence of \citet{wang2023equivalence}, which tunes the two hyperparameters so that the resulting optimal solutions coincide. Our results show that the  methods play complementary roles: DRO is favored when the shift magnitude is known but its direction is not, whereas RS can be favored when the direction is known and the magnitude is sufficiently under-specified.
 Finally, we instantiate the framework in a network lot-sizing problem under positive demand shifts, a setting in which the shift direction is known but its magnitude may be misspecified. Robust methods then substitute preventive upfront inventory for corrective second-stage transshipment and emergency ordering, and the preferred method depends on the cost structure and the degree of magnitude misspecification.

The remainder of the paper is organized as follows. 
Section \ref{sec:related} provides the literature review.
Section \ref{sec:preliminary} introduces the problem setup. 
Section \ref{sec: theoretical bounds} presents the main generalization bounds.
Section \ref{sec: partial} develops a shift-informed calibration for DRO and RS hyperparameters and establishes theoretical comparisons under partial shift information, supported by simulations.
Section \ref{sec:lot_sizing} applies the framework to a network lot-sizing problem. 
Section \ref{sec:conclusion} concludes.

\subsection{Our Methods and Contributions} \label{subsec contributions}
In this paper, we reinterpret DRO and RS in the context of distributional shifts. 
Our analysis evaluates optimizers trained on the source distribution by their excess loss  under the target distribution, which we refer to as the \emph{generalization error}.
This perspective differs from the conventional evaluation,
which measures excess loss under the source distribution itself and thus says little about  performance once the distribution shifts
\citep{esfahani2015data,blanchet2019robust,gao2023finite}. 
It also differs from adversarial-distribution formulations \citep{lee2018minimax,an2021generalization}, in which performance is measured by the worst-case loss over an ambiguity set rather than by the loss under a specific target distribution. 
Focusing on this target-based criterion, we derive generalization error bounds for DRO and RS, providing a theoretical basis for understanding their robustness under distributional shifts.

Specifically, for DRO we establish the following bound that quantifies the trade-off in choosing the radius hyperparameter $r$: 
    \[
    \mathcal{R}_{\text{DRO}} \leq {\underbrace{L \cdot  \inf_{P
    \in \mathcal{B}(P_{S},r)}d_W(P,P_{T})}_{\text{Sensitivity to shift $C_1(r)$}}} + \underbrace{C_2(r)}_{\text{Regularization penalty}}+C_{n,DRO}.
    \]
Here \(\mathcal{R}_{\mathrm{DRO}}\) is the  generalization error of the DRO optimizer in the target environment, with \(P_S\) and \(P_T\) being the source and target distributions. The set \(\mathcal{B}(P_S,r)\) is the type-I Wasserstein ball of radius \(r\) around \(P_S\), $d_W(\cdot,\cdot)$ is the type-I Wasserstein distance, and \(L\) is the Lipschitz constant in Assumption~\ref{f-lipschitz}.
 This bound decomposes into three components: (i) sensitivity to shift $C_1(r)$, which decreases with $r$ and reflects how well the ambiguity set covers the target distribution $P_T$, (ii) regularization penalty  $C_2(r)$, which grows with $r$ and 
reflects the extra cost incurred when enforcing robustness, and (iii) a residual term $C_{n,DRO}$, which is independent of \(r\) and contains statistical estimation error and the source-target optimal-value gap (defined in Section~\ref{sec: theoretical bounds}). 
The bound highlights the central trade-off in choosing $r$: larger $r$ reduces $C_1(r)$ by improving coverage for the target distribution but increases the penalty from $C_2(r)$, while smaller $r$ reduces $C_2(r)$ but risks excluding the target distribution.

For RS we establish a parallel generalization error bound that characterizes the trade-off in choosing the reference value parameter $\tau$:
\[
    \mathcal{R}_{\text{RS}} \leq \underbrace{k_\tau \cdot d_W(P_{S}, P_{T})}_{\text{Sensitivity to shift $C_3(\tau)$}} + \underbrace{C_4(\tau)}_{\text{Regularization penalty}}+C_{n,RS}.
    \]
Here  $\mathcal{R}_{\mathrm{RS}}$ is the  generalization error of the RS optimizer in the target environment, and $k_{\tau}$ is the fragility measure introduced by \citet{long2023robust} and is part of the optimized RS solution. The RS bound similarly decomposes into three components: (i) sensitivity to shift $C_3(\tau)$, which decreases with $\tau$ as the coefficient $k_\tau$ becomes smaller; (ii) regularization penalty $C_4(\tau)$, which grows with $\tau$ and captures the satisficing cost imposed by the RS formulation; and (iii) a residual term $C_{n,RS}$, which has the same order as the corresponding term in the DRO bound and is independent of $\tau$. The trade-off is explicit: larger $\tau$ reduces the shift sensitivity but increases the regularization penalty, while smaller $\tau$ keeps the regularization penalty low but leaves the model more sensitive to distributional shifts. We also note that the statistical error component in $C_{n,RS}$ scales as $n^{-1/2}$ under a bounded diameter assumption and avoids the Wasserstein concentration rate that causes the curse of dimensionality in prior work \citep{li2024statistical}.

To better understand how DRO and RS achieve robustness against distributional shifts, we present the bound for ERM below as a baseline:
\[
\mathcal{R}_{\mathrm{ERM}}
\;\le\;
\underbrace{L \cdot d_W(P_{S}, P_{T})}_{\text{Sensitivity to shift}}
\;+\;
C_{n,ERM},
\]
where \(\mathcal{R}_{\mathrm{ERM}}\) denotes the target-environment generalization error of the ERM solution trained on a source sample from \(P_S\) and evaluated under \(P_T\).
Relative to this baseline, both robust methods reduce the sensitivity-to-shift component of the upper bound. Specifically, DRO attenuates the shift distance by replacing it with the minimum distance between the ambiguity set and the target distribution, while RS reduces the multiplicative factor on the distance from $L$ to $k_\tau$ (which is smaller than $L$ by Lemma \ref{k leq L}). However, this robustness comes at a cost, with DRO incurring a regularization penalty $C_2(r)$ that increases with $r$ and RS incurring a satisficing regularization cost $C_4(\tau)$ governed by the threshold $\tau$. Together, these clarify the distinct trade-off mechanisms by which the two methods achieve robustness under distributional shifts.

The above results motivate a closer comparison of DRO and RS under distributional shifts. However, when no shift information is available, it is challenging to calibrate either method's hyperparameter in a principled way, rendering a direct head-to-head comparison uninformative. We therefore focus on settings with partial shift information (e.g., shift magnitude or shift direction). This information is mapped into each method's native hyperparameter semantics: the DRO radius \(r\) represents distributional coverage, whereas the RS threshold \(\tau\) specifies a loss tolerance. The resulting information-directed hyperparameter calibration provides a common informational input for comparing the two methods' generalization bounds. We consider two scenarios: (i) the shift magnitude is known but the direction is unknown, and (ii) the shift direction is known (through a distribution family) but the magnitude is unknown.
The first scenario favors DRO because the radius \(r\) can directly encode the known shift magnitude. In the second scenario, RS can yield a tighter  bound when the shift magnitude is sufficiently under-specified.  These analyses are further validated through simulations on a two-product newsvendor problem.

We finally apply our analysis to  network lot-sizing, a representative operations problem. 
We consider upward demand shifts, in which target demand is on average higher than source demand. The shift direction is assumed to be known because shortages are costly, whereas its magnitude is uncertain and may be misspecified.
 We find that the relative performance of DRO and RS depends on both the misspecification of the shift magnitude and the ratio of initial to  emergency ordering costs. When initial ordering is relatively cheap, RS attains a lower total cost than DRO when the shift magnitude is largely under-specified. As the initial-ordering cost increases, this advantage narrows  and eventually reverses once the magnitude is   over-specified. 
To trace the source of these differences, we decompose total cost into initial-ordering and operational components.  We then compare the shift-information-calibrated hyperparameter pair with the optimization-based correspondence curve,  showing how the DRO radius and RS reference value drive the two methods' differing sensitivities to distributional shifts.

\section{Related Work}

\label{sec:related}
Our work relates to a broad stream of research in robust optimization \citep{hu2013kullback,bayraksan2015data,long2023robust,gao2023finite,gao2023distributionally,esfahani2015data,blanchet2019robust,lam2019recovering,lee2018minimax,an2021generalization,wang2023equivalence,li2024statistical}. 
We highlight three particularly relevant lines of work: statistical guarantees for DRO, Robust Satisficing models, and the connection between DRO and RS.

\subsection{Statistical Guarantees for Distributionally Robust Optimization}

A substantial body of work has examined the statistical properties of Distributionally Robust Optimization (DRO). Early studies constructed ambiguity sets using moment conditions, such as requiring distributions to match low-order moments \citep{delage2010distributionally,goh2010distributionally}, though these approaches lacked formal asymptotic consistency guarantees \citep{shafieezadeh2019regularization}. Subsequent research shifted to ambiguity sets defined through statistical distances. For instance, \citet{duchi2021statistics} showed that DRO with $f$-divergence admits an asymptotic expansion as the empirical loss plus a variance regularization term. Relatedly, \citet{lam2019recovering} studied empirical divergence-based DRO and showed how suitably calibrated divergence balls can recover classical asymptotic statistical guarantees. The Wasserstein metric, however, has become the central tool in statistical analysis on DRO. 
\citet{esfahani2015data} provided a comprehensive treatment of Wasserstein DRO, proving that setting the radius $r$ on the order of $n^{-\frac{1}{\max\{2,d\}}}$ yields finite-sample guarantees and asymptotic consistency, where $d$ is the dimension of the random vector over which the Wasserstein distance is defined. Since the rate deteriorates to $n^{-1/d}$ in dimensions $d>2$, these guarantees inherit the curse of dimensionality from the Wasserstein concentration bounds between empirical and true distributions \citep{fournier2015rate}.

To address this limitation, \citet{shafieezadeh2019regularization} introduced structure into the model class and showed that for structured settings such as linear models,  a radius of $r = O\left(\frac{1}{\sqrt{n}}\right)$  is sufficient for finite-sample guarantees. From a more nonparametric perspective, \citet{blanchet2019robust} applied an  empirical likelihood reformulation of DRO. Their framework relaxed the requirement that the ambiguity set contains the true distribution $P_S$, requiring instead only a distribution $P$ that produces the same optimal parameter as $P_S$. This relaxation enabled asymptotic confidence intervals with order $O\left(\frac{1}{\sqrt{n}}\right)$.
More recently, \citet{gao2023finite} introduced localized Rademacher complexity to obtain dimension-free finite-sample guarantees. 

Despite these advances, a critical limitation remains. Most existing generalization bounds require the DRO radius $r_n$ to vanish as the sample size $n \to \infty$. This vanishing-radius regime effectively positions DRO as a tool for mitigating overfitting, rather than a methodology for handling distribution shifts of non-negligible magnitude. Recent exceptions begin to consider non-vanishing radii; for example, \citet{azizian2023exact} developed dual concentration bounds that allow non-vanishing radii and provide robustness to shifts without dimensional dependence. 
A different line of work, including \citet{lee2018minimax} and \citet{an2021generalization}, studies DRO with non-vanishing radii under an adversarial evaluation framework.  Our proof in Section~\ref{sec: theoretical bounds} builds on some techniques from \citet{lee2018minimax}. However, those results control the adversarial worst-case loss over an ambiguity set, rather than characterizing the trade-off with respect to a given target distribution. As a result, they do not directly characterize the target-environment performance of source-trained optimizers.

Overall, prior DRO guarantees either focus on vanishing radii for controlling sampling error or study non-vanishing radii through adversarial worst-case loss. These results do not directly characterize target-environment performance for a given distribution shift of interest. Meanwhile, their results do not capture the trade-off that a larger $r$ may better cover the target distribution but also induce greater conservatism.

\subsection{Robust Satisficing Framework}\label{sec: 2.2}

Robust Satisficing (RS) was first proposed by \citet{long2023robust} as a framework to achieve robustness without the excessive conservatism of minimax approaches, by adopting a satisficing principle. 
The concept has since gained traction \citep[e.g.,][]{ramachandra2021robust,li2024statistical}.
Existing research on RS has primarily focused on optimization tractability and applications. \citet{long2023robust} and \citet{ramachandra2021robust} analyzed the dual formulation of RS in risk-based linear optimization, linear optimization with recourse, and conic optimization. RS has also been integrated into  applications across operations research and reinforcement learning, 
including portfolio optimization \citep{long2023robust}, robust Bayesian satisficing models \citep{saday2023robust}, and satisficing Markov decision processes \citep{ruan2023robust}.

Statistical analysis of RS remains underexplored.
Recently, \citet{li2024statistical} provided the first statistical guarantees for RS, analyzing the performance of optimizers trained on source data when evaluated in shifted target environments. Their results derived a generalization bound in which the generalization error scales linearly with the Wasserstein distance between the source and target distributions.
However, this guarantee has two critical limitations. First, the bounds remain subject to the curse of dimensionality, scaling polynomially with data dimension and limiting applicability in high-dimensional settings. Second, similar to existing results for DRO, the analysis does not capture the  trade-off induced by the parameter $\tau$. Increasing $\tau$ monotonically loosens their resulting bound without providing certifiable robustness gains. Our work  addresses both of these challenges.

\subsection{Connection between DRO and RS}\label{sec:hyperp corres}

\citet{wang2023equivalence} present  early research on connecting DRO and RS, where they establish an optimization-based hyperparameter correspondence between these two robust methods. Specifically, for Wasserstein DRO and RS models built on a convex objective with uncertain quantities, they show that the two formulations share the same solution family: one can choose the DRO hyperparameter \(r\) and the RS hyperparameter \(\tau\) so that the two models yield the same optimal solution. The authors further illustrate this correspondence by plotting the mapping between the two hyperparameters. This correspondence is derived in a \emph{post-hoc} manner: given the sample, one varies the RS threshold \(\tau\), solves the RS optimization problem, and then recovers the corresponding DRO radius \(r_\tau\) that leads to the same solution.  The correspondence holds at the optimization level and does not reflect the distinct mechanisms of these two methods in robustness control.

In contrast, we study the connection between DRO and RS through the lens of distributional shift. In our framework, partial shift information is used to calibrate the hyperparameters before solving either robust problem. Given the same shift information, we propose model-specific ways to choose the DRO radius \(r\) and the RS reference value \(\tau\) according to their native semantics. Our analysis therefore yields an \emph{a priori} shift-information-directed hyperparameter calibration, rather than a post-hoc mapping obtained by solving the optimization problem of the other model. This calibration emphasizes a learning-theoretic perspective and clarifies how incorporating shift information into hyperparameter selection affects the target-environment trade-off bounds of both methods.

\section{Preliminaries}\label{sec:preliminary}

In this section, we set up the problem by first introducing the key evaluation metric under distributional shifts. We then present the mathematical formulations of the two robust learning methods that serve as our primary objects of analysis: Distributionally Robust Optimization (DRO) and Robust Satisficing (RS), alongside the baseline method of Empirical Risk Minimization (ERM). We conclude this section by formally stating our goal.

\subsection{Generalization Error under Distributional Shifts}\label{sec:generalization definition}

Let $x \in \mathcal{X}$ be the decision variable, where $\mathcal{X}$ is a decision constraint set.
Samples are represented by a $d$-dimensional random variable $z \in \mathcal{Z}$, where $\mathcal{Z} \subseteq \mathbb{R}^d$ denotes the instance space. 
Let
$\mathcal P(\mathcal Z)$ denote the set of all  probability distributions on the instance set $\mathcal Z$.
We use the function $f(x,z)$ to quantify the cost or loss associated with the decision $x$ at the sample value $z$. 
The performance of a decision $x$ under a target distribution $P_T$ is evaluated through the expected target loss $\mathbb{E}_{P_T}[f(x,z)]$.

We focus on the distribution shift setting. Given $n$ i.i.d.~samples $\{z_i\}_{i=1}^n$ drawn from a source distribution $P_S$, 
an algorithm learns an estimated decision $\hat{x} \in \mathcal{X}$.
Then $\hat{x}$ is evaluated by its generalization error, defined as the excess loss under the target distribution:
\begin{equation}
    \mathcal{R}_{P_T}(\hat{x}) := \mathbb{E}_{P_T}[f(\hat{x}, z)] - J_T,
    \label{eq:transfer_error}
\end{equation}
where $J_T := \inf_{x \in \mathcal{X}} \mathbb{E}_{P_T}[f(x,z)]$ is the minimal expected loss under the   target distribution. We similarly introduce $ J_S := \inf_{x \in \mathcal{X}} \mathbb{E}_{P_S}[f(x,z)]$ as the minimal expected loss under the   source distribution and $x_S \in \argmin_{x \in \mathcal{X}} \mathbb{E}_{P_S}[f(x,z)]$ as a corresponding minimizer. The main problem is to learn a decision $\hat x$ from the data generated by the source distribution $P_S$ to achieve a low generalization error $\mathcal{R}_{P_T}(\hat{x})$ on the target distribution $P_T$.

\subsection{Data-Driven Robust Methods}\label{sec:data-driven}
We first present the baseline method for statistical learning and then introduce the two data-driven robust methods that serve as the focus of our analysis throughout the paper.

\subsubsection{Baseline: Empirical Risk Minimization (ERM).}
The standard approach in statistical learning is Empirical Risk Minimization (ERM), defined as
\begin{equation}
    \min_{x \in \mathcal{X}}
    \;\mathbb{E}_{\hat{P}_n}[f(x,z)],
    \label{ERM}  
\end{equation}
where the empirical distribution $\hat{P}_n = \frac{1}{n}\sum_{i=1}^n \delta_{z_i}$ is constructed from the sample set $\{z_i\}_{i=1}^n$ drawn i.i.d. from the source distribution $P_S$, and $\delta_{z}$ denotes the Dirac point mass at $z$.
ERM achieves asymptotic consistency when the sample size $n$ is large  and no distribution shift is present. However, it is prone to overfitting in small-sample regimes and fragile under distribution shifts ($P_S \neq P_T$), motivating the need for robust alternatives \citep{esfahani2015data,long2023robust}. In what follows, we use ERM as the baseline for both theoretical and empirical comparisons with the robust methods introduced below.

\subsubsection{Distributionally Robust Optimization (DRO).}

The most widely studied robust method is Distributionally Robust Optimization (DRO). 
We focus on the formulation based on the type-I Wasserstein distance: for two distributions $Q_1,Q_2 \in \mathcal P(\mathcal Z)$, let
$\Pi(Q_1,Q_2)$ denote the set of couplings of $Q_1$ and $Q_2$, i.e., the set of joint distributions whose marginals agree with $Q_1, Q_2$. The type-I
Wasserstein distance induced by the norm $\|\cdot\|$ on $\mathcal Z$ is defined as
\[
d_W(Q_1,Q_2)
:=
\inf_{\pi\in\Pi(Q_1,Q_2)}
\int_{\mathcal Z\times\mathcal Z}
\|z-z'\|\,\pi(dz,dz').
\]
Unlike ERM, which optimizes only over the empirical distribution $\hat{P}_n$, DRO minimizes the worst-case expected loss over an ambiguity set centered at a nominal distribution $Q$, typically taken to be a Wasserstein ball
$\mathcal{B}(Q,r) := \{ P \in \mathcal{P}(\mathcal{Z}) : d_W(P, Q) \le r \}$. Following convention, we take the empirical distribution $\hat{P}_n$ as the nominal distribution and set $Q = \hat{P}_n$, so that the ambiguity set is $\mathcal{B}(\hat{P}_n, r)$. The decision is then obtained by solving:
\begin{align}
&\min_{x \in \mathcal{X}}\;\max_{P \in \mathcal{B}(\hat{P}_n, r)} \;\mathbb{E}_P[f(x,z)].\label{eq:DRO}
\end{align}
The hyperparameter $r$, referred to as the ``radius'', defines the size of this ambiguity set and is central to controlling robustness:  larger $r$ values expand the set of candidate distributions and provide stronger robustness, while smaller $r$ values reduce conservatism but risk excluding the true target distribution from the ambiguity set.

The performance of DRO is typically evaluated by two types of performance metrics (see Section~2.2.3 of \citet{blanchet2025distributionally} for a review). One is the  generalization error in the source environment 
\citep{esfahani2015data,shafieezadeh2019regularization,gao2023finite},
\begin{align}\label{eq:no_shift}
    \mathbb{E}_{P_S}\!\big[f(\hat{x}, z)\big]- J_S,
\end{align}
which assumes {\it no distributional shift} (i.e., $P_T=P_S$) and therefore focus primarily on overfitting to finite-sample data within a single environment. Another is the following  performance metric from an adversarial perspective \citep{lee2018minimax,an2021generalization}:
\begin{align}\label{eq:adversarial}
\sup_{P \in \mathcal{B}(P_S,r)} \mathbb{E}_{P}\!\big[f(\hat{x}, z)\big]
    \;-\;
    \inf_{x \in \mathcal{X}} \sup_{P \in \mathcal{B}(P_S,r)} \mathbb{E}_{P}\!\big[f(x,z)\big].
\end{align}

In contrast, this paper focuses on the target-environment generalization error  defined in Eq.~\eqref{eq:transfer_error}, which measures the suboptimality of a decision learned in the source environment when it is deployed in a potentially different target environment.
This makes it the  more appropriate criterion for studying robustness under distributional shifts. Our theory in Section~\ref{sec: theoretical bounds}  bounds this generalization error for  DRO  and characterizes the trade-off induced by the hyperparameter $r$. 

\subsubsection{Robust Satisficing (RS).}
Another increasingly popular robust method is the Robust Satisficing (RS) model, recently proposed by \citet{long2023robust}. 
RS has gained increasing attention for its ability to avoid the excessive conservatism of minimax formulations by adopting a satisficing principle. The model is defined as
\begin{align}
k_{\tau}=&\min \; k \nonumber \\  \text { s.t. } & \mathbb{E}_{P}[f(x,z)]-\tau \leq k d_W(P,\hat{P}_n), \quad \forall P \in\mathcal{P}(\mathcal{Z}), \label{eq:rs} \\ & x \in \mathcal{X}, \quad k \geq 0. \nonumber
\end{align}
Unlike DRO, RS does not restrict attention to a local ambiguity set and instead imposes a global performance constraint over $\mathcal{P}(\mathcal{Z})$. As pointed out by \citet{long2023robust}, the optimal value $k_\tau$, measuring  the model's \emph{fragility},  is the worst-case excess above the threshold $\tau$ normalized by the distributional distance from $\hat P_n$. A smaller $k_\tau$ (or fragility) means that threshold violations grow more slowly as the distribution moves away from the nominal distribution, and therefore indicates greater robustness. Increasing $\tau$ relaxes the constraint and can reduce the fragility measure $k_\tau$.

The RS's hyperparameter to control robustness is the reference value $\tau$, which can be interpreted as an anticipated cost in domain-specific applications. RS enforces robustness by adopting a satisficing strategy: deviations of the expected loss beyond $\tau$ scale proportionally with the distributional distance from the empirical distribution $\hat{P}_n$.
Choosing $P = \hat{P}_n$ in \eqref{eq:rs} shows that any feasible decision $x$ must satisfy $\mathbb{E}_{\hat{P}_n}[f(x,z)] \leq \tau$.  Therefore, RS feasibility requires $\inf_{x\in\mathcal X}\mathbb{E}_{\hat{P}_n}[f(x,z)] \le \tau$. 
Motivated by this lower bound, we consider
\begin{equation}
\label{eq:tau_reparametrization}
    \tau_\epsilon = \inf_{x\in \mathcal{X}}\mathbb{E}_{\hat{P}_n}[f(x,z)]+\epsilon,
\end{equation}
where $\epsilon \geq 0$ serves as a tolerance parameter that quantifies the allowable excess empirical loss.
Our theory in Section~\ref{sec: theoretical bounds}  shows that 
increasing $\epsilon$ (hence $\tau$) lowers the shift-related fragility but enlarges the satisficing regularization penalty, whereas decreasing $\epsilon$ (hence $\tau$) keeps the  penalty small at the cost of making the RS model more vulnerable to distributional shifts.

\subsection{Goal} 

We aim to characterize generalization error bounds for DRO and RS under distributional shifts and to reveal their distinct robustness mechanisms. These bounds elucidate  the trade-off in setting each method's  robustness hyperparameter: reducing sensitivity to shift can improve target-environment performance, but it also introduces a regularization or satisficing cost. This analysis  forms the theoretical foundation of our paper and is detailed in Section~\ref{sec: theoretical bounds}.
Building on this trade-off structure, Section~\ref{sec: partial} compares the two  methods under partial shift information. We ask how the robustness hyperparameters of DRO and RS can be calibrated from the same partial  information, and how such model-specific calibration  shapes target-environment performance. We consider  two  regimes: known shift magnitude with unknown direction, and known shift direction with unknown magnitude. The results are further consolidated by numerical simulations on representative operations problems.

\section{Generalization Error Bounds under Distributional Shift}\label{sec: theoretical bounds}

This section presents our main results. We characterize the generalization error bounds for the optimizers produced by the two robust methods under distributional shifts, and benchmark these against the corresponding bound for ERM. We begin by outlining the regularity conditions commonly used in the literature \citep{esfahani2015data,lee2018minimax}.

\begin{assumption}[Regularity]\label{assump:regularity}
We assume:
\begin{enumerate}[label=(\alph*),ref=\theassumption(\alph*),itemsep=0pt, topsep=0pt, parsep=0pt, partopsep=0pt]
  \item \textbf{Bounded $\mathcal{Z}$.}
  The instance space $\mathcal{Z}$ is bounded:
  $\operatorname{diam}(\mathcal{Z})=\sup_{z,z'\in \mathcal{Z}} ||z-z'|| < \infty$.
  \label{bounded z}

  \item \textbf{Bounded functions.}
  The loss function  $f(x,z)$ is lower semicontinuous in $x \in \mathcal{X}$ and is uniformly bounded, i.e.,
  $0 \le f(x,z) \le M < \infty,\ \forall x \in \mathcal{X},\ z \in \mathcal{Z}$.
  \label{bounded f}

  \item \textbf{Lipschitz continuity of loss.}
  The loss function is Lipschitz in $z$, uniformly over $x \in \mathcal{X}$:
  \[
    \sup_{z \neq z'}\frac{|\,f(x,z)-f(x,z')\,|}{\|z-z'\|}\le L,\ \ \forall x \in \mathcal{X}.
  \]
  \label{f-lipschitz}
\end{enumerate}
\end{assumption}
\vspace{-0.7cm}

Assumption~\ref{bounded z} is  imposed mainly to facilitate the analysis and may  plausibly be relaxed to sub-Gaussian random variables.
Indeed, our numerical experiments do not  enforce it,  using Gaussian distributions with unbounded support, which suggests that the  applicability of the robust methods extends beyond this assumption.
Assumption \ref{f-lipschitz} is essential for the dual reformulation of the type-I Wasserstein distance (see \cite{fournier2015rate} for details) and is satisfied by many commonly used loss functions including newsvendor loss, Huber loss, logistic loss, and so on.

\subsection{Baseline: Empirical Risk Minimization (ERM)}

We first characterize the generalization error of the ERM method in \eqref{ERM} to establish a baseline for comparison with the robust methods. The decision space $\mathcal{X}$ induces a loss class $\mathcal{A} := \{ z \mapsto f(x,z) : x \in \mathcal{X} \}$, whose complexity is measured by the entropy integral
\(
\mathcal{C}(\mathcal{A}) := \int_{0}^{\infty} \sqrt{\log \mathcal{N}(\mathcal{A}, \|\cdot\|_{\infty}, u)} \, du,
\)
where $\mathcal{N}(\mathcal{A}, \|\cdot\|_{\infty}, u)$ denotes the $\ell_\infty$-norm covering number of $\mathcal{A}$ at radius $u$. For many standard decision spaces, $\mathcal{C}(\mathcal{A})$ is finite.
Below we provide the corresponding upper bound for ERM.

\begin{proposition}[ERM, Generalization upper bound]\label{f-space, ERM bound}
If Assumption~\ref{assump:regularity} holds, then with probability at least $1-\delta$,
\begin{align}
\label{eq:erm_bound}
     \mathcal{R}_{P_T}(\hat{x}_{ERM})\leq  \underbrace{L\cdot d_W(P_S,P_T)}_{\text{Sensitivity to shift}}+\underbrace{[J_S-J_T]}_{\text{Environment gap}}+ \underbrace{\frac{24}{\sqrt{n}}\mathcal{C}(\mathcal{A})+2M\sqrt{\frac{\log(2/\delta)}{2n}}}_{\text{Statistical learning error}}, \ \ \ \forall P_T.
\end{align}
\end{proposition}

This bound consists of three components. 
The first, $L \cdot d_W(P_S, P_T)$, is the \textit{sensitivity-to-shift} term, which quantifies the  discrepancy between the source and target distributions. This is the key term that both robust methods will improve upon at the cost of introducing regularization penalty from robustness.
The second, $J_S-J_T$, represents the difference between the minimal expected losses under the source and target distributions, which we refer to as the \textit{environment gap}. This term captures the irreducible gap in the optimal losses between the source and target distributions, independent of any particular learning algorithm, and  will also appear in the generalization bounds for DRO and RS. 
The third is the \textit{statistical learning error}, a residual term that decreases at the standard $n^{-1/2}$ rate as the sample size grows; similar terms of the same rate will also appear in the generalization bounds for the two robust methods. In what follows, we use Proposition \ref{f-space, ERM bound} as the baseline to benchmark the generalization bounds of DRO and RS.

\subsection{Generalization Bound for Distributionally Robust Optimization (DRO)}

We now characterize the generalization error bound for DRO. 
Let $\hat{x}_{{DRO}}$ denote the  optimizer obtained by solving the DRO problem in Eq. \eqref{eq:DRO}. 
Following \cite{gao2023finite}, we define the DRO regularizer as:
\begin{align}
\label{eq:wasserstein_regularizer}
    \Lambda_r(Q,x)= \sup_{P\in \mathcal{B}(Q,r)}\mathbb{E}_P[f(x,z)]-\mathbb{E}_Q[f(x,z)].
\end{align}
This regularizer measures the deviation between the worst-case loss within the Wasserstein ball centered at $Q$ and the loss under $Q$ itself. By construction,  $\Lambda_r(Q,x)$ is non-decreasing in  $r$.

\begin{theorem}[DRO, Generalization upper bound]\label{DRO bound}
If Assumption~\ref{assump:regularity} holds, then with probability at least $1-\delta$,
\begin{align}\label{eq:dro_bound}
     \mathcal{R}_{P_T}(\hat{x}_{DRO}) \leq & \underbrace{L\cdot \inf_{P\in \mathcal{B}(P_S,r)} d_W(P_T,P)}_{\text{Sensitivity to shift}}
     + \underbrace{\Lambda_r(P_S,x_S)}_{\text{Regularization penalty}} \\ \nonumber& +\underbrace{[J_S-J_T]}_{\text{Environment gap}}+ \underbrace{\frac{48}{\sqrt{n}}\mathcal{C}(\mathcal{A})+ \frac{72L\cdot diam(\mathcal{Z})}{\sqrt{n}}+2M\sqrt{\frac{\log(4/\delta)}{2n}}}_{\text{Statistical learning error}}, \ \ \ \forall P_T.
\end{align}

\end{theorem}

This upper bound consists of four components.  The first is the \textit{sensitivity-to-shift}: $L\cdot \inf_{P\in \mathcal{B}(P_S,r)} d_W(P_T,P)$, which measures the discrepancy between the target distribution and the closest distribution in the Wasserstein ball centered at the source distribution. Setting $r=0$ recovers the shift term in the ERM bound \eqref{eq:erm_bound}, while $r>0$ allows DRO to improve upon it. In particular, once $r \ge d_W(P_S, P_T)$ so that the target distribution $P_T$ lies in the Wasserstein ball $\mathcal{B}(P_S, r)$, this shift term becomes zero. The second is the \textit{regularization penalty}: $\Lambda_r(P_S,x_S)$, which is defined in \eqref{eq:wasserstein_regularizer} and does not appear in the ERM bound. This term reflects the additional cost introduced by robustness; it vanishes when $r=0$, aligning DRO with ERM. The third is the common \textit{environment gap} term, $J_S-J_T$, which also appears in the ERM bound in \eqref{f-space, ERM bound}. The fourth is the statistical learning error, a residual term that  decreases at the standard $n^{-1/2}$ rate.

\begin{remark}[The First DRO's Shift-Aware Trade-off Bound on the Radius $r$] Theorem \ref{DRO bound} provides the first target-environment generalization error bound in the literature that explicitly characterizes the trade-off induced by the DRO radius $r$. 
The shift term $L \cdot \inf_{P\in \mathcal{B}(P_S,r)} d_W(P_T,P)$ decreases as $r$ increases, because enlarging the Wasserstein ball makes it more likely to cover the target distribution $P_T$. By optimizing against the worst case over the ambiguity set, DRO protects the decision against all target distributions that lie inside $\mathcal{B}(P_S,r)$, so only shifts that fall outside this covered region continue to contribute to the error. 
In contrast, the regularization term $\Lambda_r(P_S,x_S)$ grows with $r$, since a larger Wasserstein ball amplifies the worst-case loss deviation. This term therefore reflects the cost of robustness: enforcing worst-case performance over a larger ambiguity set makes the decision more conservative.
The generalization error bound in Theorem \ref{DRO bound} hence highlights the fundamental trade-off in setting $r$ appropriately under distributional shifts.
\end{remark}

\subsection{Generalization Bound for Robust Satisficing (RS) }
We now turn to the generalization error bound of RS. Let $\hat{x}_{RS}$ and $k_{\tau_\epsilon}$ denote the optimal solution and optimal value obtained from RS in Eq. \eqref{eq:rs}. The following result from \citet{li2024statistical} characterizes the magnitude of the  fragility measure $k_{\tau_\epsilon}$ relative to the Lipschitz constant $L$. We restate it here because it is useful in interpreting the RS generalization error bound that we present next.

\begin{lemma}[\cite{li2024statistical}]\label{k leq L}
Under Assumption~\ref{f-lipschitz}, we have $k_{\tau_\epsilon}\leq L$. 
\end{lemma}

Next we show the generalization upper bound for RS.
\begin{theorem}[RS, Generalization upper bound]\label{RS bound}
If Assumption~\ref{assump:regularity} holds,  then with probability at least $1-\delta$,
    \begin{align}\label{eq:rs_bound}
    \mathcal{R}_{P_T}(\hat{x}_{RS})\leq  ~&\underbrace{k_{\tau_\epsilon} \cdot d_W(P_S,P_T)}_{\text{Sensitivity to shift}} +\underbrace{\epsilon}_{\text{Regularization penalty}}\\
   &+\underbrace{[J_S-J_T]}_{\text{Environment gap}} + \underbrace{\frac{48}{\sqrt{n}}\mathcal{C}(\mathcal{A})+ \frac{48L\cdot diam(\mathcal{Z})}{\sqrt{n}}+2M\sqrt{\frac{\log(2/\delta)}{2n}}}_{\text{Statistical learning error}}, \ \ \ \forall P_T.\nonumber
\end{align}
\end{theorem}

This upper bound consists of four components. The first is the \textit{sensitivity-to-shift} $k_{\tau_\epsilon}\cdot d_W(P_S,P_T)$, which quantifies the distributional discrepancy. 
Unlike ERM, the RS framework introduces a hyperparameter-dependent multiplicative factor $k_{\tau_\epsilon}$ in the shift term. By Lemma \ref{k leq L}, this factor $k_{\tau_\epsilon}$ is no larger than the Lipschitz constant $L$, thereby improving on the ERM bound.
The second is the \textit{regularization penalty}, $\epsilon$, which is the tolerance value specified in the RS model \eqref{eq:rs} and reflects the additional cost of robustness.
The third is the \textit{environment gap}, $J_S-J_T$, which also appears in the ERM bound \eqref{eq:erm_bound} and the DRO bound \eqref{eq:dro_bound}. 
The fourth is the statistical learning error, a residual term that decays at the standard $n^{-1/2}$ rate.

\begin{remark}[RS shift-aware trade-off]
    Theorem \ref{RS bound} highlights the trade-off in setting the RS hyperparameter \(\epsilon\). On one hand, enlarging \(\epsilon\) (and hence \(\tau\)) reduces the fragility parameter \(k_{\tau_\epsilon}\), thereby decreasing the coefficient of the shift term. On the other hand, the satisficing regularization penalty grows with \(\epsilon\). Thus, RS achieves robustness by attenuating the sensitivity coefficient on the source-target distance, whereas DRO achieves robustness by reducing the effective distance from the target distribution to the ambiguity set. The two methods therefore exhibit different robustness mechanisms and different trade-offs.
\end{remark}

\begin{remark}[Breaking the Curse of Dimensionality]

We improve upon \citet{li2024statistical} by avoiding the Wasserstein concentration rate that leads to the curse of dimensionality in their bound. In our result, the statistical error depends on the instance space through $\mathrm{diam}(\mathcal Z)/\sqrt n$ instead of  a direct concentration bound for $d_W(P_S,\hat P_n)$. Although $\mathrm{diam}(\mathcal Z)$ may itself depend on the dimension, the sample-size rate remains $O(n^{-1/2})$, improving over the rate $O(n^{-\min\{1/d,1/2\}})$ that arises from direct Wasserstein concentration. This improvement is made possible because our analysis controls deviations of expected losses rather than the Wasserstein distance between $P_S$ and $\hat P_n$ itself.
\end{remark}

\begin{remark}[Sharpened RS Bounds in the Absence of Distributional Shifts]
Our result in Eq. \eqref{eq:rs_bound} can be sharpened in the absence of distributional shift. When $P_S=P_T$, $\hat{x}_{RS}$ reduces to the $\epsilon$-approximate ERM defined in Equation 3.3 of \cite{zhang2023mathematical}. To see that, with the hyperparameter choice in \eqref{eq:tau_reparametrization}, substituting $\hat{P}_n$
 into \eqref{eq:rs} implies that the empirical loss of 
$\hat{x}_{RS}$ is at most $\epsilon$ larger than the optimal empirical loss. By Proposition 4.20 in \citet{zhang2023mathematical}, under Assumption \ref{bounded f}, with probability at least $1-\delta$, we have:
    \begin{align}\label{RS bound without shift}
    \mathcal{R}_{P_T}(\hat{x}_{RS}) \leq \epsilon + \frac{24}{\sqrt{n}}\mathcal{C}(\mathcal{A})+2M\sqrt{\frac{\log(2/\delta)}{2n}}.
\end{align}

Comparing \eqref{RS bound without shift} with the generic RS bound \eqref{eq:rs_bound}, we see that  the term involving the instance diameter $diam(\mathcal{Z})$ disappears without distributional shift and therefore removes the need for  the boundedness assumption on $diam(\mathcal{Z})$ (Assumption \ref{bounded z}) and the Lipschitz assumption (Assumption \ref{f-lipschitz}). The additional term in \eqref{eq:rs_bound} thus arises as an artifact of handling distributional shifts in our proof, and   the regularity conditions (Assumptions~\ref{bounded z} and~\ref{f-lipschitz}) help to control how decision performance transfers between the source and target distributions.
In cases without shifts, however, the problem reduces to classical ERM analysis \citep{zhang2023mathematical}, where only a uniform bound  on the loss is needed. 

\end{remark}  

\subsection{An  Illustrating Example}

Now we provide a simple example to contextualize the generalization error bounds presented above, which serves two roles. First, it shows that the ERM sensitivity-to-shift term \(L\,d_W(P_S,P_T)\) can be attained, up to the common environment gap. Second, it illustrates the two robust trade-off mechanisms: DRO reduces the residual distance from the target distribution to the Wasserstein ball but pays a radius-induced regularization penalty, whereas RS lowers the coefficient multiplying the source-target distance but pays a tolerance penalty.

In this example, let \(\mathcal X=[0,1]\) and \(\mathcal Z=[-B,B]\)  be a sufficiently large bounded interval. We consider the loss \(f(x,z)=|1-xz|\), where \(x\in\mathcal X\) is the decision variable and \(z\in\mathcal Z\). The source distribution for the data variable $z$ is the Dirac measure at value $z = 1$, i.e.,  \(P_S=\delta_1\). The loss is Lipschitz in \(z\) with constant \(L=1\), and its minimum is \(J_S=0\) attained by $x_S = 1$. 
This example is free from finite-sample uncertainty, so the optimizers in this example have simple closed forms. ERM chooses \(\hat x_{\mathrm{ERM}}=x_S = 1\), because the source mass is located at \(z=1\). For \(r<1\), DRO also chooses \(\hat x_{\mathrm{DRO}}=1\), coinciding with the ERM solution. RS instead chooses \(\hat x_{\mathrm{RS}}=1-\epsilon\) for \(\epsilon<1\), sacrificing source fit in order to reduce its fragility coefficient to \(k_{\tau_\epsilon}=1-\epsilon\). We note that when $r \ge 1$ and $\epsilon \ge 1$,  DRO and RS have degenerate solutions \(\hat x_{\mathrm{DRO}} = \hat x_{\mathrm{RS}}= 0\), so we omit them subsequently. 

The following result first shows the sharpness of the generic ERM bound derived in Proposition \ref{f-space, ERM bound}. This guarantee holds for an arbitrary target distribution \(P_T\) supported on \(\mathcal Z\), and shows that the ERM shift term in \eqref{eq:erm_bound}, including the coefficient \(L\), is tight up to the common environment gap.

\begin{proposition}[Sharpness of the ERM shift term]\label{LB for ERM}
For the loss example described above, for any target distribution \(P_T\) supported on \(\mathcal Z\), we have
\[
\mathcal{R}_{P_T}(\hat x_{\mathrm{ERM}})
=
[J_S-J_T]+L\,d_W(P_S,P_T).
\]
\end{proposition}

For ease of exposition, we next restrict the target distribution to the Dirac family \(P_T=\delta_a\), where \(a \in \mathcal Z\) indexes the location of the target Dirac mass. Moving \(a\) away from \(1\) therefore represents a distributional shift away from the source environment, yielding \(d_W(P_S,P_T)=|a-1|\). The target optimum is 
\[ J_T = \inf_{x\in[0,1]} |1-xa| = \begin{cases} 1, & a\le 0,\\ 1-a, & 0\le a\le 1,\\ 0, & a\ge 1. \end{cases} \]
We write the signed environment gap as \(\Delta_{\mathrm{env}}(a):=J_S-J_T=-J_T\).

Table~\ref{tab:dirac_actual_bounds}  reports, for each method's optimizer, the actual target-environment risk computed  under \(P_T=\delta_a\) and the corresponding bound expression  by instantiating \eqref{eq:erm_bound}, \eqref{eq:dro_bound}, and \eqref{eq:rs_bound}. The ``Comment'' column shows the sensitivity-to-shift term and regularization penalty term associated. 

\begin{table}[t]
\centering
\begin{tabular}{c|c|c}
\hline
Quantity & Formula & Comment \\
\hline
ERM actual risk &
\(\Delta_{\mathrm{env}}(a)+|a-1|\) &
 \\
ERM bound &
\(\Delta_{\mathrm{env}}(a)+|a-1|\) &
Sensitivity \(|a-1|\); penalty \(0\) \\
\hline 
DRO actual risk (\(r<1\)) &
\(\Delta_{\mathrm{env}}(a)+|a-1|\) &
 \\
DRO bound (\(r<1\)) &
\(\Delta_{\mathrm{env}}(a)+(|a-1|-r)_+ + r\) &
Sensitivity \((|a-1|-r)_+\); penalty \(r\) \\
\hline 
RS actual risk ($\epsilon < 1$) &
\(\Delta_{\mathrm{env}}(a)+|1-(1-\epsilon)a|\) &
 \\
RS bound ($\epsilon < 1$) &
\(\Delta_{\mathrm{env}}(a)+(1-\epsilon)|a-1|+\epsilon\) &
Sensitivity \((1-\epsilon)|a-1|\); penalty \(\epsilon\) \\
\hline
\end{tabular}
\caption{Actual risks and theoretical bounds for different methods in the Dirac example.}
\label{tab:dirac_actual_bounds}
\end{table}

The table highlights the two distinct trade-off mechanisms of the two robust methods. For DRO, the shift-dependent part changes from the ERM distance \(|a-1|\) to the residual distance \((|a-1|-r)_+\). Increasing \(r\) therefore reduces the part of the target shift that remains outside the protected ball \(\mathcal B(P_S,r)\), but the same increase raises the regularization penalty \(r\). For RS, increasing \(\epsilon\) raises the tolerance penalty, but lowers the slope of the shift term from \(1\) to \(1-\epsilon\). Thus, DRO trades residual shift for radius regularization, whereas RS trades source-side tolerance for a smaller shift sensitivity.

 Figure~\ref{loss_comparison} further shows that the established generalization error bounds above effectively control the true losses under the target environment, after accounting for the common environment gap $J_S-J_T$.
 Specifically, define the ``adjusted risk'' as the realized target loss of the  optimizer after removing the  environment gap, and the ``adjusted bound'' as the corresponding upper bound given by the sensitivity-to-shift term plus regularization-penalty term.
 Figure~\ref{loss_comparison} confirms that these adjusted bounds provide valid upper envelopes for the corresponding adjusted risks. The ERM bound is exact in this example; the DRO and RS bounds are conservative but still meaningful, with DRO displaying the residual-shift versus radius-penalty trade-off and RS showing a slower deterioration under distributional shift because its shift coefficient is smaller.

\begin{figure}[H]
    \centering    \includegraphics[width=\linewidth]{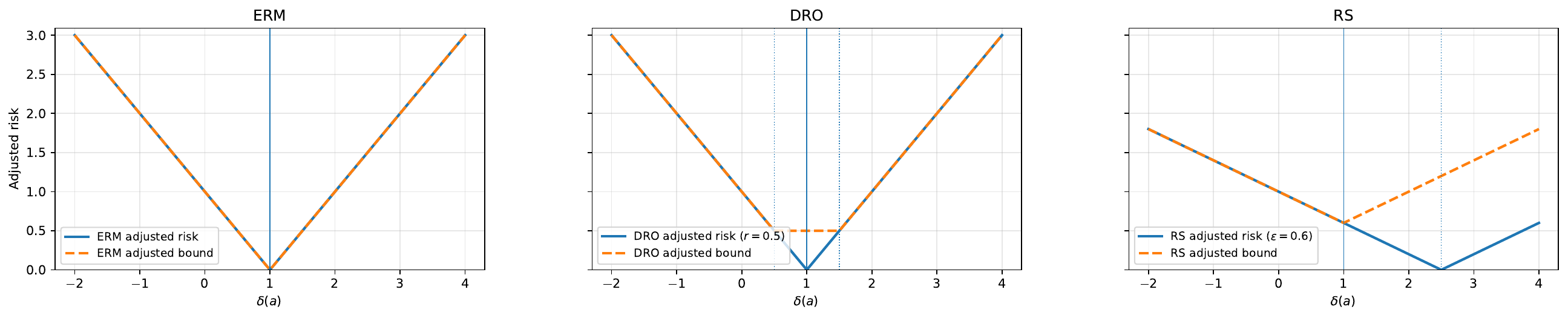}
    \caption{Adjusted risks and adjusted bounds of different methods in the Dirac example with target distribution \(P_T=\delta_a\) for $a \in [-2, 4]$. We consider DRO with  $r = 0.5$ and RS with  $\epsilon = 0.6$ as examples. }
\label{loss_comparison}
\end{figure}

\section{Comparative Statics under Partial Shift Information}\label{sec: partial}

We now compare how DRO and RS respond to distributional shifts. The generalization bounds in Section~\ref{sec: theoretical bounds} show that, relative to ERM,
DRO improves robustness by enlarging the ambiguity set through the radius $r$,
thereby reducing sensitivity to distribution shift but incurring an additional
regularization penalty. RS, in contrast, controls shift-related fragility through
the satisficing threshold $\tau$, at the cost of allowing a prescribed level of
suboptimality. These two methods enforce robustness via distinct mechanisms by setting corresponding 
hyperparameters, so their relative performance depends
critically on how those hyperparameters are chosen.

However, direct comparison between these two methods is difficult without any information about the target shift, 
since training data alone does not offer a principled criterion for setting these two robustness hyperparameters in a consistent and comparable manner.
We therefore focus on settings in which partial information about the
distributional shift is available. This information provides a common input for
calibrating hyperparameters for both methods, while still allowing each  to use that input according to its
own native interpretation. In particular, we consider two regimes: (i) the known shift
magnitude, measured by the type-I Wasserstein distance between the source and
target distributions, but unknown shift direction; and (ii) the known
shift direction within a parametric distribution family, but unknown shift
magnitude. These two regimes capture complementary forms of partial
shift information and allow us to compare DRO and RS under controlled and
mechanism-respecting calibrations.

Specifically, we consider the two terms of the generalization error bounds in
Section~\ref{sec: theoretical bounds} that capture a robustness trade-off: a sensitivity-to-shift
term, which captures the part of the target loss bound driven by the discrepancy
between the source distribution \(P_S\) and the target distribution \(P_T\), and a
regularization-penalty term, which captures the additional cost incurred by the
robustness mechanism. 
For DRO, these terms are naturally indexed by the
ambiguity radius \(r\). For RS, the sensitivity-to-shift component is indexed
by the reference value  \(\tau\), and its regularization
penalty is \(\epsilon(\tau):=
\tau-\inf_{x\in\mathcal X}\mathbb E_{\hat P_n}[f(x,z)]\), for any feasible threshold
\(\tau\ge \inf_{x\in\mathcal X}\mathbb E_{\hat P_n}[f(x,z)]\).\footnote{Although the bound in
Section~\ref{sec: theoretical bounds} is stated using the parameterization
\(\tau_\epsilon=\inf_{x\in\mathcal X}\mathbb E_{\hat P_n}[f(x,z)]+\epsilon\),
the analyses in this section are more naturally expressed in terms of the reference
threshold \(\tau\) itself.}
Table~\ref{tab:tradeoff-components} summarizes the corresponding terms by
\(\mathrm{Sen}_{\mathrm{DRO}}\), \(\mathrm{Reg}_{\mathrm{DRO}}\),
\(\mathrm{Sen}_{\mathrm{RS}}\), and \(\mathrm{Reg}_{\mathrm{RS}}\), and each method's trade-off term is defined as:
\begin{equation}\label{eq:TO-components}
\mathrm{TO}_{\mathrm{DRO}}(r)
=
\mathrm{Sen}_{\mathrm{DRO}}(r)
+
\mathrm{Reg}_{\mathrm{DRO}}(r),
\qquad
\mathrm{TO}_{\mathrm{RS}}(\tau)
=
\mathrm{Sen}_{\mathrm{RS}}(\tau)
+
\mathrm{Reg}_{\mathrm{RS}}(\tau).
\end{equation}
 Next we compare DRO and RS through these trade-off 
terms, after calibrating $r$ and $\tau$ using the same available partial shift 
information. Figure~\ref{pic:dro_vs_rs} summarizes the main conclusions.

\begin{table}[t]
\centering
\begin{tabular}{c|c|c}
\toprule
Method & Sensitivity to distribution shift & Regularization penalty \\
\midrule
DRO &
$\mathrm{Sen}_{\mathrm{DRO}}(r)
= L\cdot \inf_{P\in \mathcal{B}(P_S,r)} d_W(P_T,P)$
&
$\mathrm{Reg}_{\mathrm{DRO}}(r)
= \Lambda_r(P_S,x_S)$
\\[6pt]
RS &
$\mathrm{Sen}_{\mathrm{RS}}(\tau)
= k_{\tau}\cdot d_W(P_S,P_T)$
&
$\mathrm{Reg}_{\mathrm{RS}}(\tau)
= \epsilon(\tau)$
\\
\bottomrule
\end{tabular}
\caption{Sensitivity-to-shift and regularization-penalty components of the trade-off terms.}
\label{tab:tradeoff-components}
\end{table}

\begin{figure}[h]
    \centering
    \includegraphics[width=0.8\linewidth]{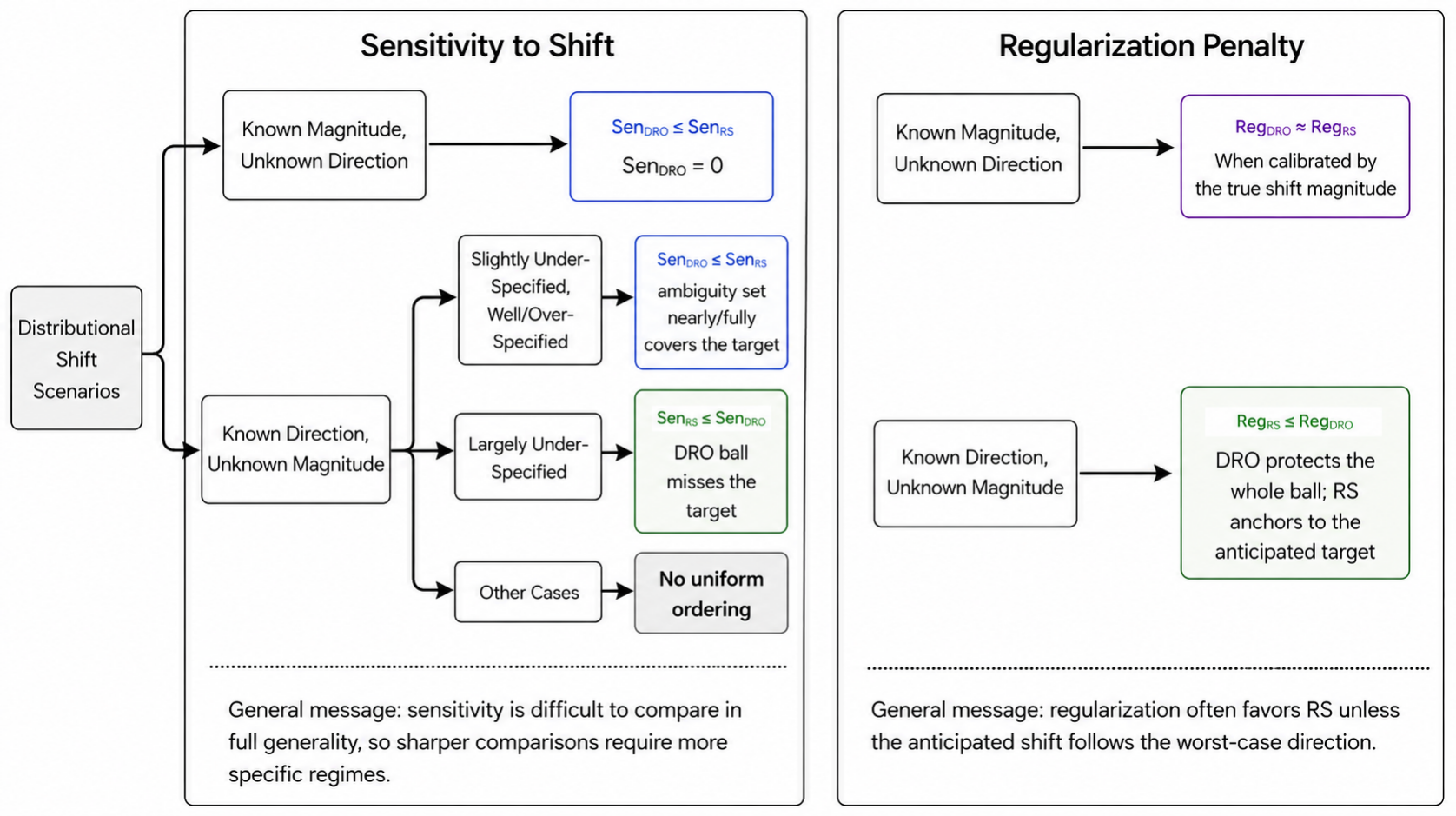}
    \caption{Comparison between DRO and RS when calibrating the hyperparameters to different shift information. The relations \(\leq\)  and \(\approx\) are understood up to vanishing error terms. }
    \label{pic:dro_vs_rs}
\end{figure}

We  acknowledge that  the hyperparameter calibration rules studied below
may not be universally optimal in practice, and   a decision maker may not  always observe
the partial shift information assumed here. However, these calibrations serve as
a systematic device for isolating how the distinct robustness mechanisms of DRO
and RS affect target-environment performance under the same informational input. Our analyses complements the optimization-based hyperparameter
correspondence literature by focusing on a learning-theoretic comparison under
controlled information regimes.

\subsection{Scenario I: Shift with Known Magnitude but Unknown Direction}
\label{sec:known_shift}

We begin with the setting where the shift magnitude $r = d_W(P_S, P_T)$ is known but the shift direction is unknown.  Under this information, the natural calibration for the DRO hyperparameter is to set
the ambiguity radius equal to that known shift magnitude: $\mathcal{B}(\hat P_n,r):=\{P\in\mathcal P(\mathcal Z):d_W(P,\hat P_n)\le r\}$.  This calibration reflects the native semantics of DRO: it protects against all candidate target distributions consistent with the known magnitude.
For RS, the hyperparameter $\tau$ has a different meaning: it is a reference value of anticipated loss. To set $\tau$'s value, 
we translate the known shift magnitude into a reference value via the candidate target distribution induced by that magnitude. Since no directional information is available, we evaluate the ERM solution over all candidate target distributions in the Wasserstein ball and take the worst-case loss as the reference value: $\tau_r=\sup_{P\in \mathcal{B}(\hat P_n,r)}\mathbb{E}_P[f(\hat x_{\mathrm{ERM}},z)]$, where $r$ is the known shift magnitude (and coincides with the DRO radius). 
This choice maps the same magnitude information into the RS hyperparameter by asking what loss level the source-trained ERM solution may have to tolerate over the set of target distributions consistent with that information. 
By doing so, we set the hyperparameters of DRO and RS with the same magnitude information while respecting the distinct semantics of $r$ and $\tau$.

We now provide a quantitative comparison of the trade-off terms of DRO and RS presented in \eqref{eq:TO-components}. 
We introduce  additional regularity assumptions for the empirical optimizer around the
source-population optimizer. 
These assumptions are standard regularity and identification conditions and hold for common loss functions such as newsvendor, logistic, and Huber losses under standard boundedness conditions and suitable nondegeneracy conditions on the source distribution $P_S$ (see Appendix~\ref{app:examples_assumptions}).

\begin{assumption}\label{assump:strong convex and lipschitz}
We assume:
\begin{enumerate}[label=(\alph*),ref=\theassumption(\alph*)]
  \item \textbf{Decision-Lipschitz loss:}
  There exists a constant $L'>0$ and a norm $\|\cdot\|_{\mathcal{X}}$ on the decision space such that $\big|f(x_1,z)-f(x_2,z)\big| ~\le~ L'\,\|x_1-x_2\|_{\mathcal{X}}$ for any $z\in\mathcal{Z}$ and $x_1, x_2\in\mathcal{X}$.
  \item \textbf{Quadratic growth around the source optimum:}
  Let $\mathcal{E}(x):=\mathbb{E}_{P_S}[f(x,z)]$. There exists $\alpha>0$ such that  for all $x \in \mathcal{X}$, $\mathcal{E}(x) ~\ge~ \mathcal{E}(x_S)  + \frac{\alpha}{2}\,\|x-x_S\|_{\mathcal{X}}^2$.
\end{enumerate}
\end{assumption}

The following proposition compares the sensitivity-to-shift terms and regularization-penalty terms for DRO and RS defined in Table \ref{tab:tradeoff-components} under our hyperparameter calibration scheme above.

\begin{proposition}
\label{prop:known_shift}
Under Assumptions~\ref{assump:regularity} and~\ref{assump:strong convex and lipschitz}, with probability at least \(1-\delta\), we have
\[
\begin{gathered}
0 = \mathrm{Sen}_{\mathrm{DRO}}(r)
\leq \mathrm{Sen}_{\mathrm{RS}}(\tau_r)
= k_{\tau_r} r,\\
\mathrm{Reg}_{\mathrm{RS}}(\tau_r) - \rho_n(\delta)
\leq \mathrm{Reg}_{\mathrm{DRO}}(r)
\leq \mathrm{Reg}_{\mathrm{RS}}(\tau_r) + \rho_n(\delta),
\end{gathered}
\]
where \(\tau_r\) denotes the RS threshold calibrated under shift radius \(r\),
\(\mathrm{Sen}_{\mathrm{DRO}}\), \(\mathrm{Reg}_{\mathrm{DRO}}\),
\(\mathrm{Sen}_{\mathrm{RS}}\), and \(\mathrm{Reg}_{\mathrm{RS}}\) are defined in
Table~\ref{tab:tradeoff-components}, and
\(
\rho_n(\delta)
=
4L'\sqrt{\frac{\mathfrak G_n(\delta)}{\alpha}}
+
\frac{24L\,\mathrm{diam}(\mathcal Z)}{\sqrt n}
+
2M\sqrt{\frac{\log(8/\delta)}{2n}}
\)
is the statistical error with
\(
\mathfrak G_n(\delta)
:=
\frac{24}{\sqrt n}\mathcal C(\mathcal A)
+
M\sqrt{\frac{\log(8/\delta)}{2n}}.
\) Therefore, for the trade-off terms defined in \eqref{eq:TO-components}, we have $\mathrm{TO}_{\mathrm{DRO}}(r)\leq \mathrm{TO}_{\mathrm{RS}}(\tau_r)-k_{\tau_r}r+\rho_n(\delta)$.
\end{proposition}
 
Proposition~\ref{prop:known_shift} compares trade-off terms  between DRO and RS when their hyperparameters are calibrated  with the shift magnitude. The regularization-penalty components of the two methods are asymptotically comparable. In particular, the DRO regularization penalty \(\mathrm{Reg}_{\mathrm{DRO}}(r)=\sup_{P\in \mathcal{B}(P_S,r)}\mathbb{E}_P[f(x_S,z)]-\mathbb{E}_{P_S}[f(x_S,z)]\) matches the RS penalty \(\mathrm{Reg}_{\mathrm{RS}}(\tau_r)=\sup_{P\in \mathcal{B}(\hat P_n,r)}\mathbb{E}_{P}[f(\hat x_{\mathrm{ERM}},z)]-\mathbb{E}_{\hat P_n}[f(\hat x_{\mathrm{ERM}},z)]\) up to the statistical error \(\rho_n(\delta)\). This is because, as \(n\rightarrow\infty\), \(\hat P_n\) converges to \(P_S\), and \(\hat x_{\mathrm{ERM}}\) converges to \(x_S\) in norm under Assumption~\ref{assump:strong convex and lipschitz}.

The key difference lies in the sensitivity-to-shift components. For DRO, the sensitivity term \(\mathrm{Sen}_{\mathrm{DRO}}(r)=L\cdot \inf_{P\in \mathcal{B}(P_S,r)} d_W(P_T,P)\) equals zero, since setting $r = d_W(P_S, P_T)$ to the exact shift magnitude  ensures that \(P_T\) lies within the set $\mathcal{B}(P_S,r)$. In contrast, the RS sensitivity term \(\mathrm{Sen}_{\mathrm{RS}}(\tau_r)=k_{\tau_r}\cdot d_W(P_S,P_T)\) is equal to \(k_{\tau_r}r\) and  does not vanish.

Combining these two components, the aggregate trade-off comparison in Proposition~\ref{prop:known_shift} shows that DRO is favored 
by the non-vanishing sensitivity gap \(k_{\tau_r}r\), up to the statistical error \(\rho_n(\delta)\). This stems from DRO's ability to directly incorporate the shift magnitude into its ambiguity set and absorb the target distribution into its worst-case robustness framework.

\subsection{Scenario II: Shift with Known Direction but Unknown  Magnitude}
\label{sec:known_direction}
We now consider the complementary setting in which the direction of the distributional shift is known, while its magnitude remains unspecified. Let $\{P_t\}_{t\ge 0}$ denote the family of distributions induced by the known shift direction, where $t$ indexes the shift magnitude. 
We introduce the following assumption to regularize the distribution family $\{P_t\}_{t\geq 0}$.
\begin{assumption}[Monotonicity]\label{monotonicity}
    For all $0\leq u\leq t$, $d_W(P_{0}, P_{u})\leq d_W(P_{0},P_t)$, where $P_0\equiv P_S$ denotes the source distribution.
\end{assumption}

This monotonicity condition ensures that the Wasserstein distance between $P_t$ and the source distribution $P_0\equiv P_S$ increases monotonically with the shift magnitude index $t$. It holds  for a variety of distribution families, including the example given below; Appendix~\ref{sec:monotonicity_assump} provides further examples satisfying Assumption~\ref{monotonicity}, such as distributions characterized by stochastic differential equations.

\begin{example}[Parameter shift within a distribution family]
Consider a parametric family of distributions 
$\{P_\theta:\theta\in\Theta\}$, where $\theta$ may be vector- or matrix-valued.
A distributional shift of the form $P_t=P_{\theta_t}$ can then be viewed as 
a parameter shift. For instance, consider the multivariate Gaussian shift
\[
    P_0=\mathcal{N}(\boldsymbol{\mu}_0,\boldsymbol{\Sigma}),
    \qquad
    P_t=\mathcal{N}
    (\boldsymbol{\mu}_0+t\boldsymbol{v},\boldsymbol{\Sigma}),
\]
where $\boldsymbol{\mu}_0,\boldsymbol{v}\in\mathbb{R}^d$ and 
$\boldsymbol{\Sigma}\in\mathbb{S}_+^d$ remains fixed. In this case,
$\theta_t=(\boldsymbol{\mu}_0+t\boldsymbol{v},\boldsymbol{\Sigma})$ and
\(
    d_W(P_0,P_t)=t\|\boldsymbol{v}\|,\)
so Assumption~\ref{monotonicity} is satisfied.
\end{example}

We  denote the true but unknown distribution shift magnitude by $t_T$, so that the target distribution is $P_T \equiv P_{t_T} $. However, since the true shift magnitude $t_T$ is unknown, we instead specify a ``nominal" magnitude $t$ that may deviate from the truth $t_T$, which
singles out a ``nominal" target distribution \(P_t\) from the known family. This $P_t$  serves as a natural object for both robust methods to calibrate their hyperparameters:
DRO can straightforwardly set \(r_t=d_W(P_S,P_t)\), suggesting that the ambiguity set should extend exactly far enough to cover the nominal target distribution indexed by \(t\);
RS instead calibrates its anticipated loss threshold via 
 $\tau_t=\mathbb E_{P_t}[f(\hat x_{\mathrm{ERM}},z)]$, i.e., the  loss of the source-trained ERM under the nominal target $P_t$.\footnote{Throughout this comparison, we focus on candidate targets under which the source-trained ERM solution deteriorates in the specified direction. This is the main regime in which adopting a robust method is substantively motivated to handle distribution shifts. In this regime, the calibrated reference level $\tau_t$ is not below the relevant empirical source optimum $\inf_{x\in\mathcal X}\mathbb{E}_{\hat{P}_n}[f(x,z)]$, so the feasibility of the corresponding RS formulation is not an issue.}  
Thus, the calibration rules \(r_t\) and \(\tau_t\) are obtained from the same hypothetical target, but each is expressed in the native language of the corresponding method: distributional coverage for DRO and anticipated reference loss level for RS.

We first compare the regularization-penalty terms of both methods. 

\begin{proposition}\label{prop:reg_comparison}
Under Assumptions~\ref{assump:regularity} and~\ref{assump:strong convex and lipschitz}, with probability at least \(1-\delta\), we have
\[
    \mathrm{Reg}_{\mathrm{RS}}(\tau_t)+\sup_{P\in \mathcal{B}(P_S,r_t)}\mathbb{E}_P[f(x_S,z)]-\mathbb{E}_{P_t}\!\big[f(x_S,z)\big]
    \leq
    \mathrm{Reg}_{\mathrm{DRO}}(r_t)
    +
    \bar{\rho}_n(\delta),
\]
where \(\mathrm{Reg}_{\mathrm{DRO}}\) and \(\mathrm{Reg}_{\mathrm{RS}}\) are defined in Table~\ref{tab:tradeoff-components}, and \(\bar{\rho}_n(\delta)=2L'\sqrt{\frac{\bar{\mathfrak G}_n(\delta)}{\alpha}}+\bar{\mathfrak G}_n(\delta)\) with \(\bar{\mathfrak G}_n(\delta)=\frac{24}{\sqrt n}\mathcal C(\mathcal A)+M\sqrt{\frac{\log(2/\delta)}{2n}}\).
\end{proposition}

Proposition~\ref{prop:reg_comparison} shows that the DRO regularization penalty is asymptotically larger than the RS regularization penalty by a nonnegative gap $\sup_{P\in \mathcal{B}(P_S,r_t)}\mathbb{E}_P[f(x_S,z)]-\mathbb{E}_{P_t}\!\big[f(x_S,z)\big]$, which captures the difference between the worst-case loss of the source optimizer over the ambiguity set $\mathcal{B}(P_S,r_t)$ and its loss under the nominal target \(P_t\). This gap could be substantial if \(P_t\) is far from the worst-case distribution within \(\mathcal B(P_S,r_t)\), capturing the potential cost of DRO's hedging against the worst-case.

The sensitivity-to-shift terms of DRO and RS, however, do not admit a uniform ordering in general. 
We therefore consider three cases, distinguished by whether the nominal shift magnitude $t$ is ``under-specified" ($t<t_T$), ``well-specified" ($t=t_T$), or ``over-specified" ($t>t_T$), for which a sharper comparison can be made.

\begin{proposition}\label{prop:underspecified}
    Suppose Assumption~\ref{monotonicity} holds, \(P_T=P_{t_T}\), \(r_t=d_W(P_S,P_t)\). The following sensitivity comparisons hold: (i) If $\frac{d_W(P_t,P_S)}{d_W(P_T,P_S)}
    \leq
    1-\frac{k_{\tau_t}}{L}$, 
    then $\mathrm{Sen}_{\mathrm{RS}}(\tau_t)
    \leq
    \mathrm{Sen}_{\mathrm{DRO}}(r_t)$; (ii) If \(t\ge t_T\), then $0 = \mathrm{Sen}_{\mathrm{DRO}}(r_t)
    \leq
    \mathrm{Sen}_{\mathrm{RS}}(\tau_t)$.
    \end{proposition}

Part (i) corresponds to cases where the distribution shift magnitude is sufficiently under-specified (i.e., $t$ is sufficiently smaller than $t_T$), which formally states that
\(\frac{d_W(P_t,P_S)}{d_W(P_T,P_S)}\leq 1-\frac{k_{\tau_t}}{L}<1\).\footnote{A sufficient condition for the under-specified cases is \(\frac{d_W(P_t,P_S)}{d_W(P_T,P_S)}\leq 1-\frac{k_{\tau_0}}{L}\)  if \(\mathbb E_{P_t}[f(\hat x_{\mathrm{ERM}},z)]\) increases monotonically in \(t\) (i.e., as the distribution keeps shifting, the ERM learned on the source continuously deteriorates).}
Combining it with Proposition~\ref{prop:reg_comparison},  the resulting trade-off bound  favors RS  up to the statistical error \(\bar\rho_n(\delta)\).

In contrast,  part (ii) of Proposition~\ref{prop:underspecified} favors DRO for the sensitivity-to-shift term, when the shift magnitude is well-specified or over-specified and thus the ambiguity set already covers the true target distribution. However, Proposition~\ref{prop:reg_comparison} shows that the regularization-penalty term
 favors RS, since DRO still hedges against the worst-case distribution in the entire ball
$\mathcal{B}(P_S,r_t)$, whereas RS is anchored to the nominal target $P_t$. Therefore, in the
well-specified and over-specified regimes, the comparison of the trade-off terms between
$\mathrm{TO}_{\mathrm{DRO}}(r_t)$ and $\mathrm{TO}_{\mathrm{RS}}(\tau_t)$ is undecided.
However, as $t$
becomes increasingly over-specified, 
the gap between the regularization penalty terms ($\mathrm{Reg}_{\mathrm{DRO}}(r_t)$ and $\mathrm{Reg}_{\mathrm{RS}}(\tau_t)$) likely  increases, which    explains the deterioration  of DRO's numerical performance  in Sections~\ref{subsection:lr-know-direction} and \ref{sec:lot_sizing}.  

We finally note that  when the shift direction aligns with the worst-case distributional direction, referred to as the \emph{adversarial scenario}, Appendix \ref{sec:adversarial} shows that DRO may have better trade-off performance because it is designed to handle worst-case distributions.

\subsection{Simulation Studies}\label{sec:simulation}

We  present  a two-product Gaussian newsvendor problem to numerically validate the theoretical findings. 
With ERM as the baseline, we evaluate the performance of DRO and RS under two partial-information regimes: (i) known shift magnitude with unknown direction, and (ii) known shift direction with unknown magnitude. In both regimes, the hyperparameters of DRO and RS are calibrated from the available shift information following rules in Sections~\ref{sec:known_shift} and~\ref{sec:known_direction}.

\subsubsection{Setup.}
Consider a two-product Gaussian newsvendor problem. Let the random variable be the demand vector
\(z=(z_1,z_2)\in\mathbb R^2\). The source distribution of  demand $z$ is
\(
    P_S=N(\mu_S,\Sigma),
\)
where the covariance matrix \(\Sigma\) is fixed. From the source distribution \(P_S\), we observe
\(n\) i.i.d.~samples \(\{z_i\}_{i=1}^n\) and estimate the source mean by
\(
    \hat\mu=\frac1n\sum_{i=1}^n z_i .
\)
We use
\(
    P_{\hat\mu}=N(\hat\mu,\Sigma)
\)
 as the distribution used in ERM, the center of the Wasserstein ball in the DRO framework and the fitted nominal distribution in the RS framework.
For a decision \(x=(x_1,x_2)\), we consider the two-product newsvendor loss
\[
    f(x,z)
    =
    \sum_{j=1}^2
    h_j(x_j-z_j)_+
    +
    b_j(z_j-x_j)_+,
\]
where \(h_j\) is the overage cost and \(b_j\) is the underage cost of product \(j\). 
Following the standard newsvendor interpretation, we take underage costs to be larger than overage costs, reflecting that unmet demand is typically more costly than leftover inventory. 
The  set of feasible decisions is  
\(
    \mathcal X
    =
    \{x\in\mathbb R_+^2:x_1+x_2\le C\},
\)
where \(C\) is the total capacity. For any mean vector \(\mu\), define the loss
\[
    L(x,\mu)
    :=
    \mathbb E_{z\sim N(\mu,\Sigma)}[f(x,z)].
\]
In our numerical experiments, we set
\(
    \mu_S=(50,50)^\top,
    \Sigma
    =
    2^2
    \begin{pmatrix}
    1 & 0.3\\
    0.3 & 1
    \end{pmatrix},
 \)
and use
\(
    h=(1,1)^\top,
    b=(3,5)^\top,
    C=105
\) unless otherwise specified.

We model distributional shifts through a Gaussian mean-shift family.\footnote{In contrast, the  type-I Wasserstein DRO under unrestricted shifts degenerates to the ERM problem for the newsvendor loss \citep[Remark~6.7]{esfahani2015data}.} Specifically, any target distribution considered in the numerical study takes the form $P_T=N(\mu_T,\Sigma)$,  where the covariance matrix \(\Sigma\) remains fixed and the discrepancy between the source and target distributions arises only through the mean shift \(\mu_S\rightarrow \mu_T\), yielding $    d_W\bigl(N(\mu_S,\Sigma),N(\mu_T,\Sigma)\bigr)
    =
    \|\mu_S-\mu_T\|_2.$ This mean-shift representation allows us to separate the magnitude and direction of a distributional shift: given  a unit direction \(u\) and a magnitude \(t\ge 0\), the corresponding shifted distribution can be written as $P_t=N(\mu_t,\Sigma)$ for $\mu_t=\mu_S+tu$, 
with \(P_0=P_S\). In Scenario I below, we  fix the known shift magnitude and sample an unknown direction to create distributional shift. In Scenario II, we fix the direction and vary the specified magnitude.

We consider the following three methods:
\begin{enumerate}
    \item ERM (baseline):
    \(
        \hat x_{\mathrm{ERM}}
        =
        \arg\min_{x\in\mathcal X} L(x,\hat\mu).
    \)

    \item Gaussian-family DRO:
    \(
        \hat x_{\mathrm{DRO}}(r)
        =
        \arg\min_{x\in\mathcal X}
        \sup_{\|\mu-\hat\mu\|_2\le r}
        L(x,\mu).
    \) That is, the corresponding Wasserstein ball $\mathcal B_G(P_{\hat\mu},r)$ is  centered at \(P_{\hat\mu}\) and defined as
\[
    \mathcal B_G(P_{\hat\mu},r)
    :=
    \{P_\mu=N(\mu,\Sigma):\|\mu-\hat\mu\|_2\le r\}.
\]
    
    \item Gaussian-family RS: for a satisficing threshold \(\tau\), we solve
    \[
    \begin{aligned}
        \hat x_{\mathrm{RS}}
        \in
        \arg\min_{x\in\mathcal X,\ k\ge0}
        \quad & k \\
        \text{s.t.}\quad
        & L(x,\mu)
        \le
        \tau+k\|\mu-\hat\mu\|_2.
        \\
    \end{aligned}
    \]
\end{enumerate}
We evaluate them by the expected loss of their optimizers in the target environment,
    $L(\hat x,\mu_T)
    =
    \mathbb E_{P_T}[f(\hat x,z)] $,
under two distributional shift scenarios with separate partial shift information.

\subsubsection{Scenario I: known shift magnitude.}\label{sec: simulation_Known Shift Magnitude}

With a known shift magnitude, measured as the Wasserstein distance between the source and target Gaussian distributions, we set this value as the radius $r$ in the DRO framework. 
For RS, following the calibration rule discussed in Section~\ref{sec:known_shift}, we choose the satisficing threshold \(\tau_r\) to reflect the worst-case performance of the ERM solution over the Gaussian-family ball within radius \(r\):
\begin{equation}
    \tau_r
    =
    \sup_{P\in\mathcal B_G(P_{\hat\mu},r)}
    \mathbb E_{P}[f(\hat x_{\mathrm{ERM}},z)]
    =
    \sup_{\|\mu-\hat\mu\|_2\le r}
    L(\hat x_{\mathrm{ERM}},\mu),
    \label{eq:tau_known_mag_newsvendor}
\end{equation}
where the last equality follows from the definition of \(\mathcal B_G(P_{\hat\mu},r)\).

We now evaluate the performance of each optimizer on the target distribution. Specifically, given a shift magnitude \(r\), we conduct Monte Carlo simulations with \(500\) replications. In each replication, we sample a unit vector \(u\in\mathbb S^1\) to represent the shift direction and construct the target distribution as
\(
    P_T(u,r)=N(\mu_S+r u,\Sigma).
\)
By construction, for every sampled direction,
\(
    d_W(P_S,P_T(u,r))=r.
\)
We then evaluate ERM, DRO, and RS by their target expected losses under \(P_T(u,r)\). 

\begin{figure}[t]
    \centering
    \includegraphics[width=0.85\linewidth]{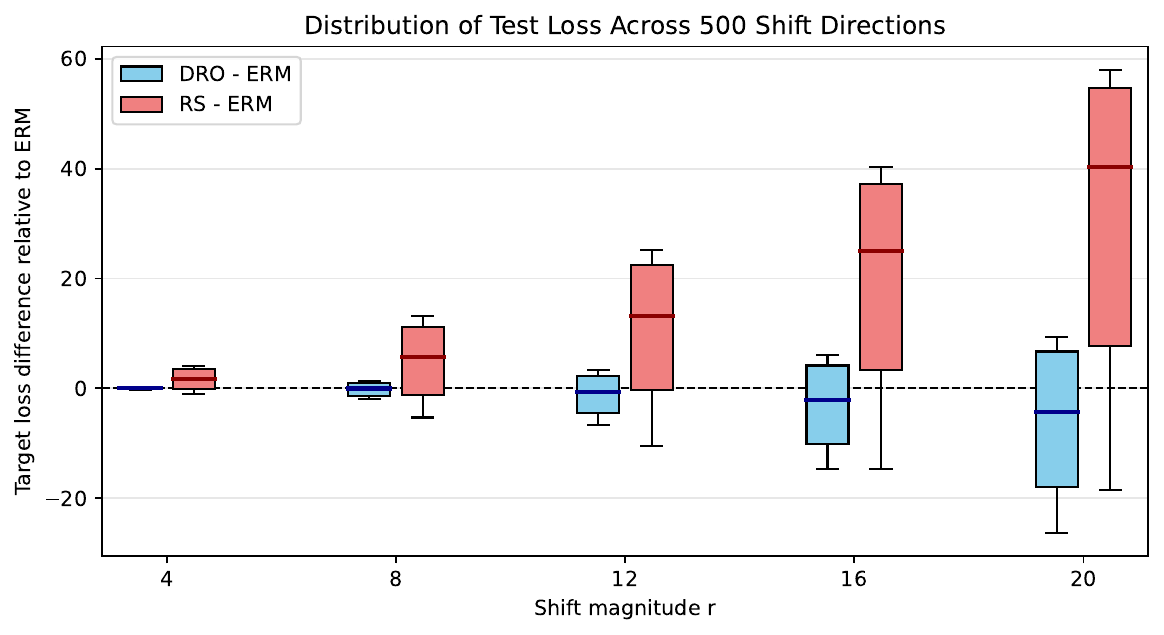}
    \caption{Distribution of loss differences relative to ERM across 500 randomly
sampled shift directions in the known-magnitude regime. For each shift magnitude $r$, the boxplots summarize the target-loss differences of DRO and RS relative to ERM, with both methods calibrated using
the known magnitude. Negative values indicate lower target loss than ERM.}
    \label{fig:known_mag_newsvendor}
\end{figure}

Figure~\ref{fig:known_mag_newsvendor} reports the target-loss differences of DRO and RS relative to ERM across  \(500\) simulated shift directions.  DRO has overall lower realized target loss than RS across the reported magnitudes, which is consistent with Proposition~\ref{prop:known_shift}: when only the shift magnitude is available, DRO is  favored as it uses the magnitude directly through its ambiguity radius, whereas RS is calibrated through a worst-case ERM-loss threshold over the same Gaussian-family ball.

Moreover, RS with the calibrated hyperparameter performs overall  worse than ERM. It also significantly deteriorates as the shift magnitude increases, while DRO remains  more competitive. The DRO--ERM comparison reflects the asymmetric cost structure of the newsvendor problem. Since product \(2\) has a larger underage cost, increasing the DRO radius leads the robust solution to order more of product \(2\). This adjustment is not beneficial for every possible shift direction. If the target shift mainly increases the demand for product \(1\), allocating more inventory to product \(2\) may hurt performance. However, if the target shift increases the demand for product \(2\), the same adjustment can substantially reduce shortage loss. Because the shortage cost of product \(2\) is high, the improvement in these directions is larger than the loss incurred in directions where the adjustment is less useful. This explains why the DRO boxplots move downward relative to zero as the shift magnitude increases, even though the improvement is not uniform across all directions.
\subsubsection{Scenario II:  known shift direction.}\label{subsection:lr-know-direction}

Let \(u\in\mathbb S^1\) denote the known mean-shift direction and consider the shifted distribution family $P_t=N(\mu_t,\Sigma)$ where $\mu_t=\mu_S+t u$ and \(t\) quantifies the shift magnitude. 
We set the true target distribution as
\(
    P_T=P_{t_T},
\)
with true shift magnitude \(t_T=16\). Since the actual shift magnitude is unknown, we consider nominal magnitudes 
$t=\alpha t_T$, with $
    \alpha\in\left\{\frac13,\frac23,1,2,3,4\right\}$,
which correspond to under-specified ($\alpha<1$), well-specified ($\alpha=1$), and over-specified  ($\alpha>1$) regimes.
For DRO, we use the radius $r_t$ corresponding to the 
nominal magnitude $t$, which equals $t$ under the Gaussian mean-shift parameterization.
For RS, since the shift direction is known, we calibrate the satisficing threshold by evaluating the ERM solution at the nominal shifted mean $\tau_{t}
=
L\bigl(\hat x_{\mathrm{ERM}},\,\hat\mu+t u\bigr)$.

We focus on demand-increasing directions, i.e., directions in the first quadrant, which corresponds to the demand-increasing shift as the  main robustness concern  in the newsvendor problem.\footnote{In contrast, demand-decreasing shifts mainly create overage losses, which are less costly under cost choice \(b_j>h_j\).} 
Since product \(2\) has a larger underage cost, robust methods tend to allocate more protection to product \(2\). Shifts that mainly increase the demand of product \(2\) are more aligned with the robust adjustment, while shifts that mainly increase the demand of product \(1\) are less aligned with it. To capture whether the known demand-shift direction is favorable or unfavorable to the inventory protection induced by the robust methods, we  consider   shift directions falling into the  two regimes below:
\begin{itemize}
    \item \textit{High alignment}: the demand shift direction is close to \(e_2=(0,1)^\top\), with
    \(
        \langle u,e_2\rangle\ge 0.9,
    \)
    so the target demand mainly increases for product \(2\);

    \item \textit{Low alignment}: the demand shift direction is close to \(e_1=(1,0)^\top\), with
    \(
        \langle u,e_1\rangle\ge 0.9,
    \)
    so the target demand mainly increases for product \(1\).
\end{itemize}
For each regime, we sample \(50\) directions and compute the loss differences of DRO and RS relative to ERM across the nominal values of $t$ (or corresponding \(\alpha\)). The target distribution in each replication is
\(
    P_T=N(\mu_S+t_T u,\Sigma),
\)
while the hyperparameters of DRO and RS are calibrated using the pre-specified magnitude \(t = \alpha t_T\) and the known direction \(u\). 

\begin{figure}[t]
    \centering
    \includegraphics[width=\linewidth]{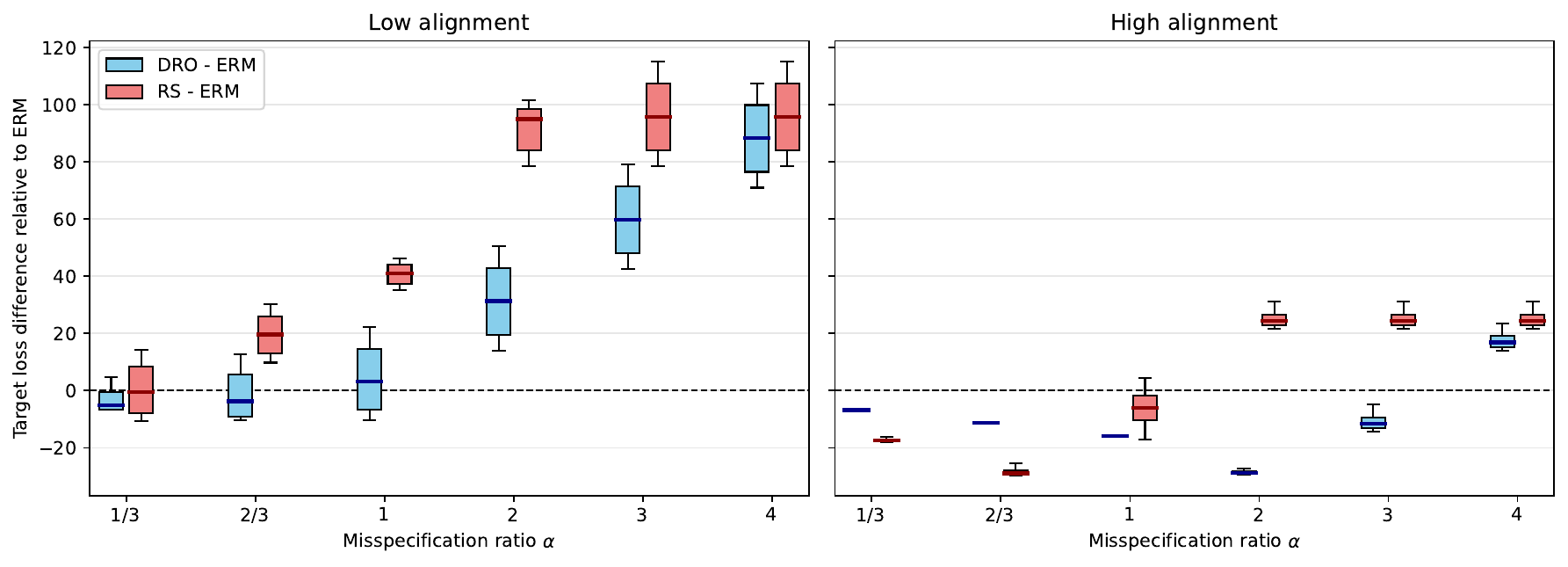}
    \caption{Loss differences relative to ERM across 50 randomly sampled shift directions in each alignment regime under the known-direction setting. DRO and RS are calibrated using a specified shift magnitude, while the true magnitude is fixed. Directions are sampled from two demand-increasing regimes: low alignment, where demand mainly increases for product \(1\), and high alignment, where demand mainly increases for product \(2\). The cases \(\alpha<1\), \(\alpha=1\), and \(\alpha>1\) correspond to under-specification, correct specification, and over-specification, respectively.}
    \label{fig:known_direction_newsvendor}
\end{figure}

Figure~\ref{fig:known_direction_newsvendor} shows that the performance of robust methods depends strongly on whether the known direction agrees with the protection induced by robustness. In the low-alignment regime, the target demand mainly increases for product \(1\), while the robust methods tend to protect product \(2\). As a result, the robust adjustment is directionally mismatched. When the specified magnitude is small, DRO remains close to ERM and can slightly improve upon it. As the nominal magnitude increases, however, this mismatch is amplified: RS deteriorates rapidly, and DRO also becomes increasingly worse under large over-specification.

In the high-alignment regime, the target demand mainly increases for product \(2\), which is the product that robust methods tend to protect because of its larger underage cost. Hence being robust is directionally favorable. When the shift magnitude is largely under-specified, both robust methods can improve upon ERM in this regime, and RS can be better than DRO. 
When the magnitude is over-specified, DRO in turn becomes more favorable than RS. 
These patterns are consistent with Propositions~\ref{prop:reg_comparison} and~\ref{prop:underspecified}: when the shift magnitude is sufficiently under-specified, RS can be favored over DRO; 
once the specified radius exceeds the true shift magnitude, the DRO ambiguity set covers the target distribution and its sensitivity term vanishes, while 
 RS still retains a non-vanishing sensitivity term of the form \(k_{\tau_t}t_T\) and is outperformed by DRO.
However, as the nominal magnitude becomes too large, the trade-off changes. 
According to the analysis below  Proposition \ref{prop:underspecified}, increasing the specified magnitude enlarges the ambiguity set used by DRO and may substantially increase its worst-case regularization penalty. 
In this regime, even though the DRO sensitivity term remains zero after the target is covered, the growing regularization gap becomes dominant. 
This is reflected in the upward movement of the DRO boxplots under  over-specification.

\section{Application to  Network Lot-sizing}\label{sec:lot_sizing} 

We now  apply our analytical framework to a network lot-sizing problem, in which the shift direction is known (demand shifts upward), but its magnitude is not.  

\subsection{Setup}\label{sec:lot_sizing_setup}
\subsubsection*{The network lot-sizing problem.}
Consider  $N$ geographically dispersed stores. Before the demand is realized, an initial inventory allocation is made across the stores. After observing the realized demand, the system can either rebalance inventory through transportation between stores or place emergency local orders to satisfy unmet demand. The total system cost consists of three components: the initial-ordering cost, the transportation cost from transshipment, and the emergency ordering cost.

The initial ordering decision is denoted as $x=(x_i)_{i\in[N]}$, where $x_i$ represents the pre-stocked quantity at store $i$ and $[N]:=\{1,\ldots,N\}$. Each store has a capacity limit $\delta=(\delta_i)_{i\in[N]}$, satisfying $x_i\leq \delta_i$ for all $i\in[N]$. 
After the realization of demand, the system makes a recourse decision $(y,w)$ to satisfy demand, where $y=(y_{ij})_{i,j\in[N]}$ denotes the amount transshipped from store $i$ to store $j$, and $w=(w_i)_{i\in[N]}$ represents the emergency order at store $i$. The cost parameters are as follows: $c=(c_i)_{i\in[N]}$ for per-unit initial orders, $l=(l_i)_{i\in[N]}$ for per-unit emergency orders, and $d=(d_{ij})_{i,j\in[N]}$ for per-unit transportation costs between stores. For brevity, we denote $d_i^\top y_i:=\sum_{j\in[N]}d_{ij}y_{ij}$.

The total system cost given an initial order $x$ and a realized demand vector $\boldsymbol z=(z_i)_{i\in[N]}$ is determined from the following
two-stage linear optimization model:
\begin{align}\label{lot_sizing}
    f(x,\boldsymbol z)
    =
    c^\top x
    +
    \min_{(y,w)\in\mathcal Y_{(x,\boldsymbol z)}}
    \left\{
        \sum_{i\in[N]}d_i^\top y_i+l^\top w
    \right\},
\end{align}
where the first-stage decision satisfies
\(x \in \mathcal{X} := \{x \in \mathbb{R}_+^N \mid x_i \leq \delta_i,\ 
\forall i \in [N]\}\), and the feasible set of second-stage decisions is
\(
\mathcal Y_{(x,\boldsymbol z)}
=
\{
(y,w)\in\mathbb R^{N\times N}_+\times\mathbb R^N_+
\mid
x_i+w_i+\sum_{j\in[N]}y_{ji}-\sum_{j\in[N]}y_{ij}-z_i\geq 0,\ \forall i\in[N]
\}.
\)
The first-stage decision $x$ specifies the initial ordering for the pre-positioned inventory across stores, while the second-stage decisions $(y,w)$ capture the adaptive replenishment response after demand is realized. The objective is to determine an initial stocking decision $x$ that minimizes the expected total cost (consisting of the initial-ordering cost $c^\top x$ and the operational cost $\min_{(y,w)\in\mathcal Y_{(x,\boldsymbol z)}}
\{
    \sum_{i\in[N]} d_i^\top y_i+l^\top w
\}$) under possible distributional shifts in $\boldsymbol z$. The tractable approximate reformulations for solving the DRO and RS optimizers are provided in Appendix~\ref{appd:network lotsizing reformulation}.

We instantiate the network lot-sizing problem as follows. Consider $N=10$ stores. For each store $i\in[N]$, the initial ordering quantity $x_i\geq0$ is constrained by the capacity limit $\delta_i=40$. We define the source distribution $P_S$ as the joint distribution of the random demand vector $\boldsymbol z=(z_i)_{i\in[N]}$, where customer demands $z_i$, $i\in[N]$, are independently distributed as $\mathcal N(20,50)$.
When  sampling demand realizations, we truncate negative values  to ensure non-negativity.\footnote{Sampling negative values occurs rarely, with probability approximately $0.0024$.} 
We draw $M$ i.i.d.\ source-demand realizations
$\{\widehat{\boldsymbol z}_S^{(m)}\}_{m\in[M]}$ from $P_S$, where $\widehat{\boldsymbol z}_S^{(m)}
=(\hat z_{S,i}^{(m)})_{i\in[N]}\in\mathbb R_+^N$
denotes the $m$-th joint demand realization across all stores. 
These source realizations serve as the training samples for computing the initial ordering decisions under ERM, DRO, and RS.
We assume a common unit initial-order price $c_i$ for all stores and vary it over $c_i\in\{5,10,20\}$. The emergency-order price is fixed at $l_i=30$. Store locations are randomly assigned on a $10\times10$ grid. Let $D_{ij}$ denote the Euclidean distance between stores $i$ and $j$; the per-unit transportation cost $d_{ij}$ is set proportional to this distance, with $d_{ij}=2D_{ij}$ for all $i,j\in[N]$.

\subsubsection*{Positive demand mean shifts.}
We consider distributional shifts in which the demand distribution shifts along a mean-increasing direction. In this setting, the baseline ERM model is trained on a source distribution with lower average demand than that of the target environment where the demand is realized. It thus tends to underestimate true demand, leading to insufficient initial ordering and a higher likelihood of incurring additional operational costs from lateral transportation and emergency orders.  The robust  methods DRO and RS are applied to  account for such upward demand shifts.

Specifically, the true target demand over all stores follows a multivariate normal distribution
\(P_T:=
    \mathcal N(20\cdot\mathbf 1_N+t_T d,\ 50I_N),
\)
where the shift direction $d=(d_1,\ldots,d_N)$ is randomly sampled from the positive orthant such that $d_i\ge0$ for all $i$ and $\|d\|_2=1$. The true shift magnitude is fixed at $t_T=4$. We assume that the shift direction is known, while the shift magnitude remains unknown.

Both robust models choose their hyperparameters based on a nominal target distribution
\(P_t:=
    \mathcal N(20\cdot\mathbf 1_N+t d,\ 50I_N)
\)
with nominal shift magnitude $t$ as in Section~\ref{sec:known_direction}.
The learned decisions are evaluated using an independent test set $\{\widehat{\boldsymbol z}_T^{(k)}\}_{k\in[K]}$, where $K$ denotes the
number of target-demand realizations used for evaluation and $\widehat{\boldsymbol z}_T^{(k)}$, $k\in[K]$, are i.i.d.\ realizations
from $P_T$. We set $K=1000$.
For each learned decision $\hat x$, we compute the corresponding test costs $\{f(\hat x,\widehat{\boldsymbol z}_T^{(k)})\}_{k\in[K]}$.
In the experiments, we vary the specified shift magnitude $t$ from values below the true shift magnitude $t_T=4$ (representing under-specified shifts) to values above $t_T=4$ (representing over-specified shifts). 

\begin{figure}[t]
  \centering
  \includegraphics[width=\linewidth]{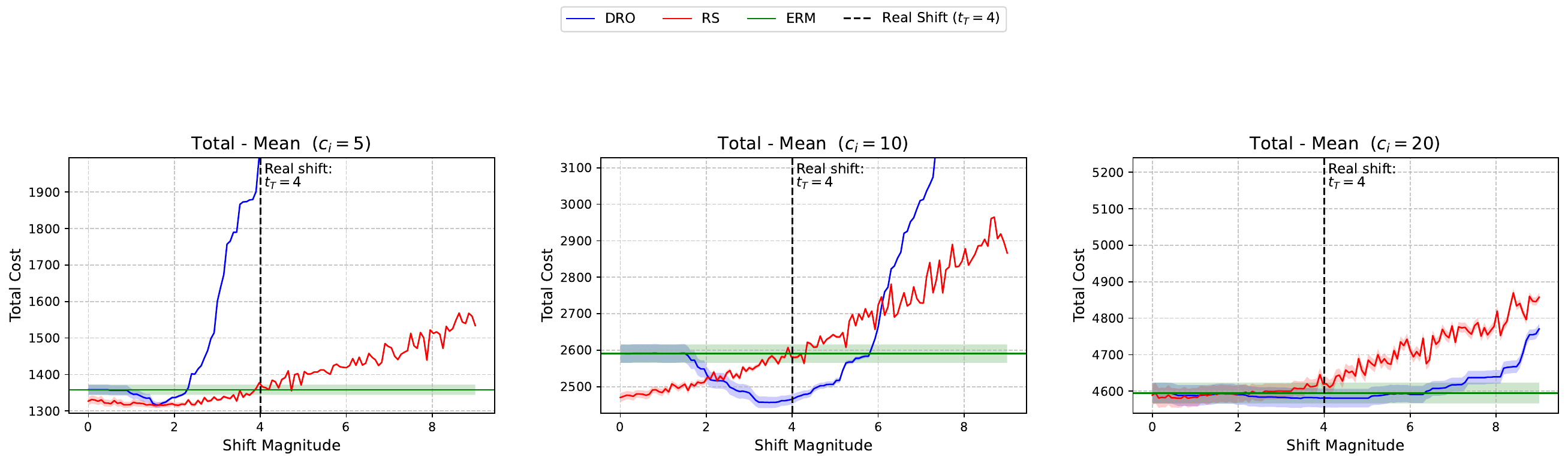}
    \caption{Mean total costs of different methods for the network lot-sizing problem across different  shift magnitude specifications ($t_T=4$ corresponds to the truth). The shift direction is known and used to calibrate DRO and RS hyperparameters, while the shift magnitude is unknown and specified along the horizontal axis. We consider three different levels of per-unit initial-ordering cost $c_i \in \{5, 10, 20\}$. }
  \label{fig:total-mean}
\end{figure}

\subsection{Model Performance}\label{sec:positive mean shift}
\subsubsection*{
Total cost.}

Figure~\ref{fig:total-mean} reports the mean total cost of ERM, DRO, and RS. In the under-specified regime \((t<t_T)\), robust methods generally improve upon ERM because the source-trained ERM decision tends to under-order for upward demand shifts. With large under-specification, RS attains a lower total cost than DRO. This pattern is  consistent with  Proposition~\ref{prop:underspecified}: when the nominal shift magnitude is small, the DRO ambiguity set does not  cover the true target distribution, whereas RS can still reduce the effective sensitivity to the shift through its fragility measure.

The performance changes in the over-specified regime \((t>t_T)\). RS gradually orders more inventory as the nominal demand level increases and can eventually become worse than ERM. DRO exhibits a more cost-dependent pattern. When the unit initial-ordering cost \(c_i\) is small, DRO becomes highly conservative and performs substantially worse than both RS and ERM. As \(c_i\) increases, this disadvantage diminishes, and DRO can eventually outperform RS.
We shall further examine this initial-ordering-cost-dependent performance in the subsequent sections.

Appendix~\ref{app: experiment} provides additional numerical results on quantile-based metrics. 
Results show that RS achieves better \(95\)th-percentile performance than DRO for  upward demand shifts, consistent with the empirical findings of \citet{long2023robust} in settings without distributional shifts.

\subsubsection*{Cost decomposition: initial ordering vs.~operational costs.}\label{sec:initial vs operational}

\begin{figure}[t]
    \centering
    \includegraphics[width=0.9\linewidth]{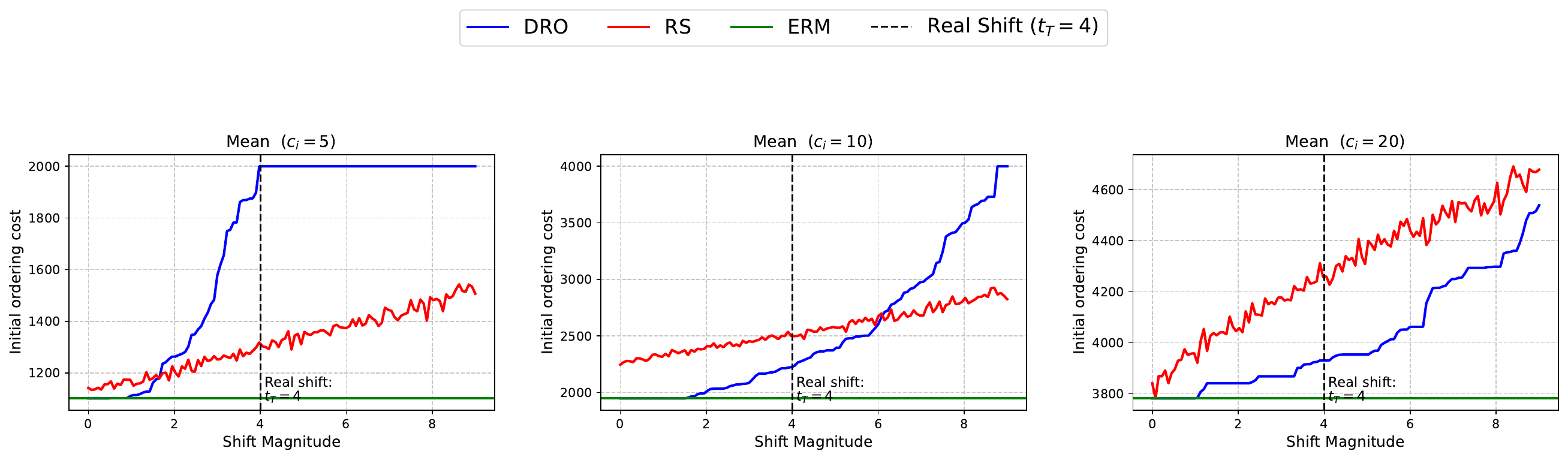}\\[0.5em]
    \includegraphics[width=0.9\linewidth]{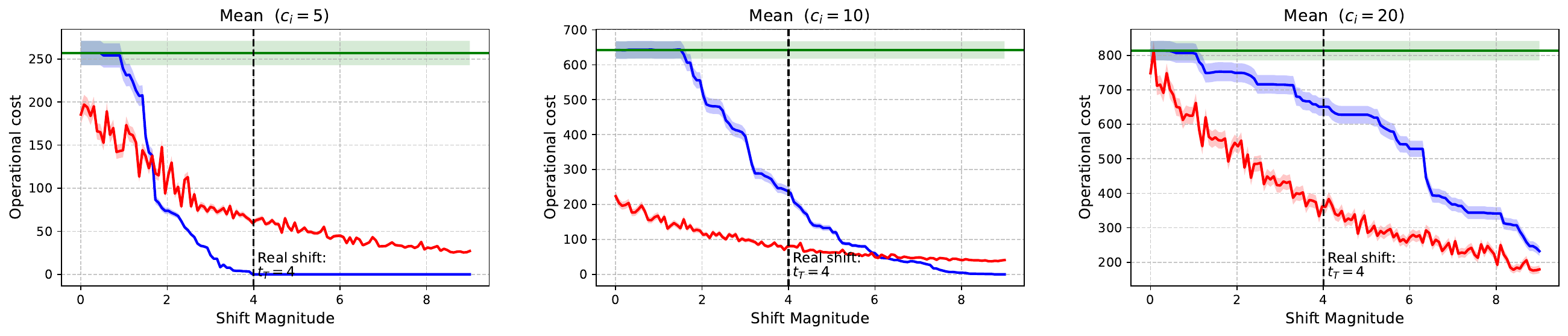}
    \caption{Mean cost decomposition for the network lot-sizing application. The top row reports initial-ordering costs, and the bottom row reports operational recourse costs (transshipment costs plus emergency-order costs). Together, they decompose the total costs in Figure~\ref{fig:total-mean}.}
    \label{fig:initial and operational panels}
\end{figure}

Figure~\ref{fig:initial and operational panels} decomposes the mean total cost in Figure~\ref{fig:total-mean} into  initial-ordering costs and  operational costs after demand is realized.\footnote{The corresponding \(95\%\) quantiles, reported in Appendix~\ref{appendix:95 Quantiles of Initial Ordering and Operational Costs}, show similar patterns.} Relative to ERM, both DRO and RS typically incur higher initial-ordering costs but lower operational costs, indicating that robust methods substitute preventive upfront inventory for corrective second-stage transshipment and emergency ordering.
The two methods differ in how this substitution responds to the nominal shift magnitude \(t\). RS changes almost linearly: as \(t\) increases, its initial-ordering cost rises steadily and its operational cost falls steadily. DRO is more sensitive to the unit ordering cost \(c_i\). 
When \(c_i\) is small, DRO rapidly front-loads inventory; in the \(c_i=5\) panels, the initial-ordering cost quickly reaches a high plateau, whereas the operational cost is already close to zero, so further increases in \(t\) mostly reflect over-protection rather than additional recourse-cost savings.
When \(c_i\) is large, DRO reacts more cautiously because the cost of excess initial inventory becomes more important.

This decomposition explains the total-cost patterns in Figure~\ref{fig:total-mean}. In the under-specified regime, RS achieves a better balance by moderately increasing initial orders while substantially reducing operational costs. For small or intermediate \(c_i\), DRO may become too aggressive, especially in the over-specified regime: its initial-ordering cost rises sharply without a proportional operational-cost reduction. As \(c_i\) increases, however, DRO's upfront ordering becomes less aggressive, while RS continues to scale approximately linearly with \(t\); consequently, DRO can become preferable at higher unit ordering costs.

\subsubsection*{Optimization correspondence and shift-calibrated hyperparameters.}

We now compare two ways of relating the DRO radius \(r\) and the RS reference value \(\tau\). First, following \citet{wang2023equivalence}, we derive the optimization-based correspondence between \(r\) and \(\tau\), where the two methods produce the same solution. Second, we overlay the shift-calibrated hyperparameter pairs generated by the calibration rule in Section~\ref{sec:known_direction}. Figure~\ref{6.2 hyperparameter correspondence} uses this comparison to help understand how the relative performance of DRO and RS changes with the initial-ordering cost \(c\). Let \(\tau^{\mathrm{eq}}(r)\) denote the optimization-equivalent RS threshold on the blue curve, i.e., the RS threshold that would reproduce the DRO solution with different radius \(r\).
Let
$\tau^{\mathrm{shift}}(r)$ denote the shift-calibrated RS threshold on the
orange curve, obtained directly from each nominal target distribution
$P_t$ that induces the DRO radius $r = P_W(P_S, P_t)$, over the same range of $r$.
For small $c_i=5$, the blue curve rises sharply and quickly saturates, while the orange
curve remains much lower over most shift magnitudes. Hence, for the same
specified magnitude, the shift-calibrated RS solution is substantially less
conservative than the solution-equivalent RS representation of DRO. This
explains why DRO overreacts to large specified shifts and incurs high total cost
in Figure~\ref{fig:total-mean}, whereas RS changes more gradually.
For large $c_i=20$, the ordering of the two curves is largely reversed: the
shift-calibrated RS threshold is above the solution-equivalent threshold
corresponding to DRO. In this case, RS induces a larger increase in initial
stocking as the specified magnitude grows, while DRO reacts more moderately.
This explains why RS becomes more costly than DRO for large $c_i$. The case with moderate $c_i=10$ is intermediate, with the two curves
crossing, which is consistent with the mixed performance pattern in Figure~\ref{fig:total-mean}.

\begin{figure}[t]
    \centering
    \includegraphics[width=\linewidth,height=0.2\textheight]{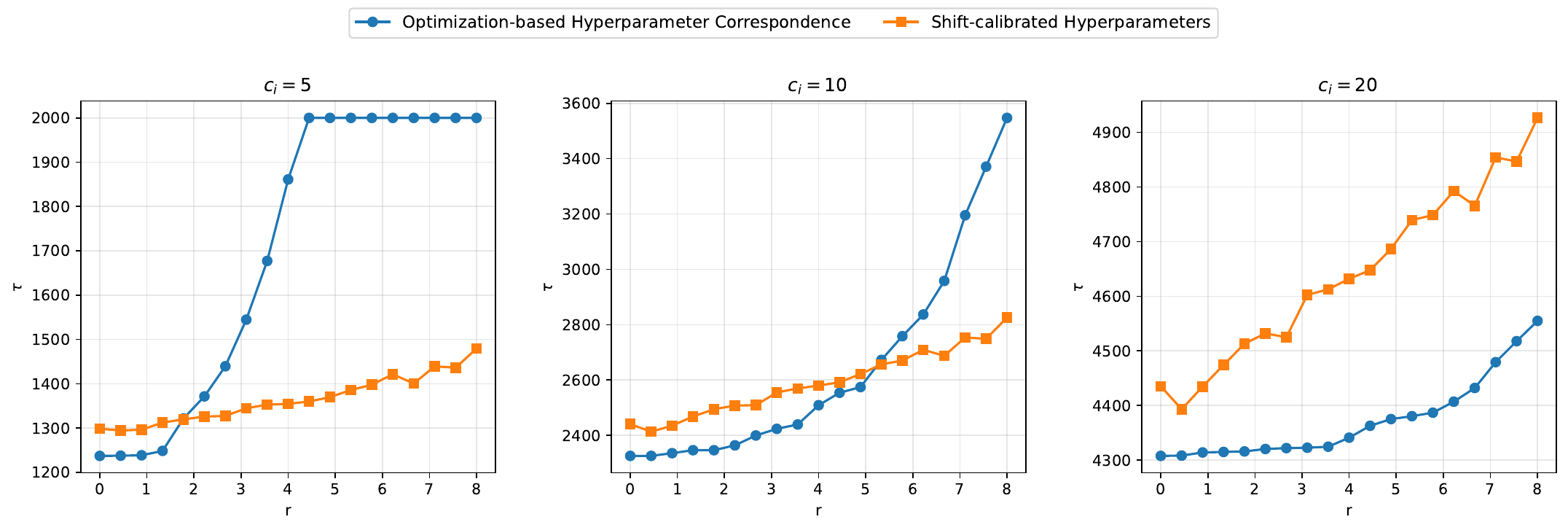}
    \caption{Optimization correspondence and shift-calibrated hyperparameter pairs for DRO and RS. The blue curve reports the optimization-based correspondence of \citet{wang2023equivalence}, where each \((r,\tau)\) pair yields identical DRO and RS solutions. The orange curve reports the parameter pairs induced by our shift-information-directed calibration, where \(r\) and \(\tau\) are both chosen from the same specified shift magnitude as in Section~\ref{sec:known_direction}. } 
    \label{6.2 hyperparameter correspondence}
  \end{figure}
\section{Conclusion}\label{sec:conclusion}

In this paper, we study how Distributionally Robust Optimization (DRO) and Robust Satisficing (RS) address distributional shifts.
We derive generalization error bounds for both methods under distributional shifts, explicitly characterizing the trade-off between reduced sensitivity to shift and regularization cost as a function of their respective hyperparameters.
Building on these results, we provide shift-information-directed hyperparameter calibration and method comparison under scenarios with partial shift information.
Under such hyperparameter calibration, when the shift magnitude is known but the direction is unknown, the resulting trade-off bound favors DRO because the ambiguity radius can directly encode the known magnitude.
When the shift direction is known but the magnitude is unknown, RS can be favored when the magnitude is sufficiently under-specified.
Further simulation studies show target-environment generalization patterns that largely align with these theoretical comparisons, and the network lot-sizing application gives an operational interpretation through the trade-off between preventive initial ordering and corrective operational recourse costs.

 We hope that this work contributes to a more refined, shift-explicit characterization of the finite-sample behavior of DRO and RS and represents a step toward connecting empirical evaluations of target-environment performance under distribution shifts.
 A natural future direction is to investigate whether the analysis can be extended to applications such as medical imaging diagnostics, autonomous-driving perception, supply-chain, disaster-response planning, and fairness-sensitive decision-making, where distribution shifts are well documented and robust learning methods are increasingly studied or deployed. Our network lot-sizing application and numerical experiments illustrate how theory can guide target-environment guarantees and hyperparameter selection in a structured setting. Extending such guidance to these more complex environments remains an important open problem.

\newpage

\bibliographystyle{informs2014} 
\bibliography{reference} 

@inproceedings{an2021generalization,
  author    = {An, Yang and Gao, Rui},
  title     = {Generalization Bounds for ({Wasserstein}) Robust Optimization},
  booktitle = {Advances in Neural Information Processing Systems},
  volume    = {34},
  pages     = {10382--10392},
  year      = {2021},
  url       = {https://proceedings.neurips.cc/paper/2021/hash/55fd1368113e5a675e868c5653a7bb9e-Abstract.html}
}

@inproceedings{azizian2023exact,
  author    = {Azizian, Wa{\"\i}ss and Iutzeler, Franck and Malick, J{\'e}r{\^o}me},
  title     = {Exact Generalization Guarantees for (Regularized) {Wasserstein} Distributionally Robust Models},
  booktitle = {Advances in Neural Information Processing Systems},
  volume    = {36},
  pages     = {14584--14596},
  year      = {2023},
  publisher = {Curran Associates, Inc.},
  doi       = {10.52202/075280-0641},
  url       = {https://proceedings.neurips.cc/paper_files/paper/2023/hash/2f060912eacace9ce61ef339205ec54c-Abstract-Conference.html}
}

@incollection{bayraksan2015data,
  author    = {Bayraksan, G{\"u}zin and Love, David K.},
  title     = {Data-Driven Stochastic Programming Using Phi-Divergences},
  booktitle = {The Operations Research Revolution},
  editor    = {Aleman, Dionne M. and Thiele, Aur{\'e}lie C.},
  series    = {{INFORMS} TutORials in Operations Research},
  pages     = {1--19},
  year      = {2015},
  publisher = {INFORMS},
  address   = {Catonsville, MD},
  doi       = {10.1287/educ.2015.0134}
}

@article{blanchet2019robust,
  author  = {Blanchet, Jose and Kang, Yang and Murthy, Karthyek},
  title   = {Robust {Wasserstein} Profile Inference and Applications to Machine Learning},
  journal = {Journal of Applied Probability},
  volume  = {56},
  number  = {3},
  pages   = {830--857},
  year    = {2019},
  doi     = {10.1017/jpr.2019.49}
}

@article{blanchet2025distributionally,
  author  = {Blanchet, Jose and Li, Jiajin and Lin, Sirui and Zhang, Xuhui},
  title   = {Distributionally Robust Optimization and Robust Statistics},
  journal = {Statistical Science},
  volume  = {40},
  number  = {3},
  pages   = {351--377},
  year    = {2025},
  doi     = {10.1214/24-STS955}
}

@article{chen2023designing,
  author  = {Chen, Shengjie and Chen, Yanju},
  title   = {Designing a Resilient Supply Chain Network under Ambiguous Information and Disruption Risk},
  journal = {Computers \& Chemical Engineering},
  volume  = {179},
  pages   = {108428},
  year    = {2023},
  doi     = {10.1016/j.compchemeng.2023.108428}
}

@article{dai2025assured,
  author  = {Dai, Tinglong and Simchi-Levi, David and Wu, Michelle Xiao and Xie, Yao},
  title   = {Assured Autonomy: How Operations Research Powers and Orchestrates Generative {AI} Systems},
  journal = {Production and Operations Management},
  year    = {2026},
  note    = {Published online May 18, 2026},
  doi     = {10.1177/10591478261455127}
}

@inproceedings{dong2023robustness,
  author    = {Dong, Yinpeng and Kang, Caixin and Zhang, Jinlai and Zhu, Zijian and Wang, Yikai and Yang, Xiao and Su, Hang and Wei, Xingxing and Zhu, Jun},
  title     = {Benchmarking Robustness of {3D} Object Detection to Common Corruptions in Autonomous Driving},
  booktitle = {2023 {IEEE/CVF} Conference on Computer Vision and Pattern Recognition ({CVPR})},
  pages     = {1022--1032},
  year      = {2023},
  publisher = {IEEE},
  doi       = {10.1109/CVPR52729.2023.00105}
}

@article{delage2010distributionally,
  author  = {Delage, Erick and Ye, Yinyu},
  title   = {Distributionally Robust Optimization under Moment Uncertainty with Application to Data-Driven Problems},
  journal = {Operations Research},
  volume  = {58},
  number  = {3},
  pages   = {595--612},
  year    = {2010},
  doi     = {10.1287/opre.1090.0741}
}

@article{deng2023distributionally,
  author  = {Deng, Menghua and Bian, Bomin and Zhou, Yanlin and Ding, Jianpeng},
  title   = {Distributionally Robust Production and Replenishment Problem for Hydrogen Supply Chains},
  journal = {Transportation Research Part E: Logistics and Transportation Review},
  volume  = {179},
  pages   = {103293},
  year    = {2023},
  doi     = {10.1016/j.tre.2023.103293}
}

@article{duchi2021statistics,
  author  = {Duchi, John C. and Glynn, Peter W. and Namkoong, Hongseok},
  title   = {Statistics of Robust Optimization: A Generalized Empirical Likelihood Approach},
  journal = {Mathematics of Operations Research},
  volume  = {46},
  number  = {3},
  pages   = {946--969},
  year    = {2021},
  doi     = {10.1287/moor.2020.1085}
}

@article{esfahani2015data,
  author  = {Mohajerin Esfahani, Peyman and Kuhn, Daniel},
  title   = {Data-Driven Distributionally Robust Optimization Using the {Wasserstein} Metric: Performance Guarantees and Tractable Reformulations},
  journal = {Mathematical Programming},
  volume  = {171},
  number  = {1--2},
  pages   = {115--166},
  year    = {2018},
  doi     = {10.1007/s10107-017-1172-1}
}

@article{fournier2015rate,
  author  = {Fournier, Nicolas and Guillin, Arnaud},
  title   = {On the Rate of Convergence in {Wasserstein} Distance of the Empirical Measure},
  journal = {Probability Theory and Related Fields},
  volume  = {162},
  number  = {3--4},
  pages   = {707--738},
  year    = {2015},
  doi     = {10.1007/s00440-014-0583-7}
}

@article{gao2023finite,
  author  = {Gao, Rui},
  title   = {Finite-Sample Guarantees for {Wasserstein} Distributionally Robust Optimization: Breaking the Curse of Dimensionality},
  journal = {Operations Research},
  volume  = {71},
  number  = {6},
  pages   = {2291--2306},
  year    = {2023},
  doi     = {10.1287/opre.2022.2326}
}

@article{gao2023distributionally,
  author  = {Gao, Rui and Kleywegt, Anton},
  title   = {Distributionally Robust Stochastic Optimization with {Wasserstein} Distance},
  journal = {Mathematics of Operations Research},
  volume  = {48},
  number  = {2},
  pages   = {603--655},
  year    = {2023},
  doi     = {10.1287/moor.2022.1275}
}

@inproceedings{garg2020unified,
  author    = {Garg, Saurabh and Wu, Yifan and Balakrishnan, Sivaraman and Lipton, Zachary C.},
  title     = {A Unified View of Label Shift Estimation},
  booktitle = {Advances in Neural Information Processing Systems},
  volume    = {33},
  pages     = {3290--3300},
  year      = {2020},
  url       = {https://proceedings.neurips.cc/paper/2020/hash/219e052492f4008818b8adb6366c7ed6-Abstract.html}
}

@article{goh2010distributionally,
  author  = {Goh, Joel and Sim, Melvyn},
  title   = {Distributionally Robust Optimization and Its Tractable Approximations},
  journal = {Operations Research},
  volume  = {58},
  number  = {4, Part 1},
  pages   = {902--917},
  year    = {2010},
  doi     = {10.1287/opre.1090.0795}
}

@inproceedings{gulrajani2021search,
  author    = {Gulrajani, Ishaan and Lopez-Paz, David},
  title     = {In Search of Lost Domain Generalization},
  booktitle = {International Conference on Learning Representations},
  year      = {2021},
  url       = {https://openreview.net/forum?id=lQdXeXDoWtI}
}

@inproceedings{hashimoto2018fairness,
  author    = {Hashimoto, Tatsunori and Srivastava, Megha and Namkoong, Hongseok and Liang, Percy},
  title     = {Fairness without Demographics in Repeated Loss Minimization},
  booktitle = {Proceedings of the 35th International Conference on Machine Learning},
  series    = {Proceedings of Machine Learning Research},
  volume    = {80},
  pages     = {1929--1938},
  year      = {2018},
  publisher = {PMLR},
  url       = {https://proceedings.mlr.press/v80/hashimoto18a.html}
}

@techreport{hu2013kullback,
  author      = {Hu, Zhaolin and Hong, L. Jeff},
  title       = {{Kullback--Leibler} Divergence Constrained Distributionally Robust Optimization},
  institution = {The Hong Kong University of Science and Technology},
  year        = {2013},
  note        = {Technical report; first posted on Optimization Online in November 2012},
  url         = {https://optimization-online.org/2012/11/3677/}
}

@article{huang2023distributionally,
  author  = {Huang, Hongxu and Li, Zhengmao and Gooi, Hoay Beng and Qiu, Haifeng and Zhang, Xiaotong and Lv, Chaoxian and Liang, Rui and Gong, Dunwei},
  title   = {Distributionally Robust Energy-Transportation Coordination in Coal Mine Integrated Energy Systems},
  journal = {Applied Energy},
  volume  = {333},
  pages   = {120577},
  year    = {2023},
  doi     = {10.1016/j.apenergy.2022.120577}
}

@article{kantorovich1958space,
  author  = {Kantorovich, Leonid V. and Rubinshtein, S. G.},
  title   = {On a Space of Totally Additive Functions},
  journal = {Vestnik Leningradskogo Universiteta. Seriya Matematiki, Mekhaniki i Astronomii},
  volume  = {13},
  number  = {7},
  pages   = {52--59},
  year    = {1958},
  note    = {In Russian}
}

@inproceedings{koh2021wilds,
  author    = {Koh, Pang Wei and Sagawa, Shiori and Marklund, Henrik and Xie, Sang Michael and Zhang, Marvin and Balsubramani, Akshay and Hu, Weihua and Yasunaga, Michihiro and Phillips, Richard Lanas and Gao, Irena and Lee, Tony and David, Etienne and Stavness, Ian and Guo, Wei and Earnshaw, Berton and Haque, Imran and Beery, Sara M. and Leskovec, Jure and Kundaje, Anshul and Pierson, Emma and Levine, Sergey and Finn, Chelsea and Liang, Percy},
  title     = {{WILDS}: A Benchmark of In-the-Wild Distribution Shifts},
  booktitle = {Proceedings of the 38th International Conference on Machine Learning},
  series    = {Proceedings of Machine Learning Research},
  volume    = {139},
  pages     = {5637--5664},
  year      = {2021},
  publisher = {PMLR},
  url       = {https://proceedings.mlr.press/v139/koh21a.html}
}

@article{lam2019recovering,
  author  = {Lam, Henry},
  title   = {Recovering Best Statistical Guarantees via the Empirical Divergence-Based Distributionally Robust Optimization},
  journal = {Operations Research},
  volume  = {67},
  number  = {4},
  pages   = {1090--1105},
  year    = {2019},
  doi     = {10.1287/opre.2018.1786}
}

@inproceedings{lee2018minimax,
  author    = {Lee, Jaeho and Raginsky, Maxim},
  title     = {Minimax Statistical Learning with {Wasserstein} Distances},
  booktitle = {Advances in Neural Information Processing Systems},
  volume    = {31},
  pages     = {2687--2696},
  year      = {2018},
  publisher = {Curran Associates, Inc.},
  url       = {https://proceedings.neurips.cc/paper_files/paper/2018/hash/ea8fcd92d59581717e06eb187f10666d-Abstract.html}
}

@article{li2023data,
  author  = {Li, Yang and Han, Meng and Shahidehpour, Mohammad and Li, Jiazheng and Long, Chao},
  title   = {Data-Driven Distributionally Robust Scheduling of Community Integrated Energy Systems with Uncertain Renewable Generations Considering Integrated Demand Response},
  journal = {Applied Energy},
  volume  = {335},
  pages   = {120749},
  year    = {2023},
  doi     = {10.1016/j.apenergy.2023.120749}
}

@inproceedings{li2024statistical,
  author    = {Li, Zhiyi and Xu, Yunbei and Zhan, Ruohan},
  title     = {Statistical Properties of Robust Satisficing},
  booktitle = {Proceedings of the 41st International Conference on Machine Learning},
  series    = {Proceedings of Machine Learning Research},
  volume    = {235},
  pages     = {29112--29127},
  year      = {2024},
  publisher = {PMLR},
  url       = {https://proceedings.mlr.press/v235/li24cc.html}
}

@article{long2023robust,
  author  = {Long, Daniel Zhuoyu and Sim, Melvyn and Zhou, Minglong},
  title   = {Robust Satisficing},
  journal = {Operations Research},
  volume  = {71},
  number  = {1},
  pages   = {61--82},
  year    = {2023},
  doi     = {10.1287/opre.2021.2238}
}

@book{quinonero2022dataset,
  editor    = {Qui{\~n}onero-Candela, Joaquin and Sugiyama, Masashi and Schwaighofer, Anton and Lawrence, Neil D.},
  title     = {Dataset Shift in Machine Learning},
  publisher = {MIT Press},
  address   = {Cambridge, MA},
  year      = {2022},
  note      = {Paperback edition},
  url       = {https://mitpress.mit.edu/9780262545877/dataset-shift-in-machine-learning/}
}

@techreport{ramachandra2021robust,
  author      = {Ramachandra, Arjun and Rujeerapaiboon, Napat and Sim, Melvyn},
  title       = {Robust Conic Satisficing},
  institution = {Indian Institute of Management Bangalore},
  type        = {{IIMB} Working Paper},
  number      = {708},
  year        = {2025},
  url         = {https://research.iimb.ac.in/work_papers/10/}
}

@inproceedings{ruan2023robust,
  author    = {Ruan, Haolin and Zhou, Siyu and Chen, Zhi and Ho, Chin Pang},
  title     = {Robust Satisficing {MDP}s},
  booktitle = {Proceedings of the 40th International Conference on Machine Learning},
  series    = {Proceedings of Machine Learning Research},
  volume    = {202},
  pages     = {29232--29258},
  year      = {2023},
  publisher = {PMLR},
  url       = {https://proceedings.mlr.press/v202/ruan23a.html}
}

@inproceedings{saday2023robust,
  author    = {Saday, Artun and Y{\i}ld{\i}r{\i}m, Ya{\c{s}}ar Cahit and Tekin, Cem},
  title     = {Robust {Bayesian} Satisficing},
  booktitle = {Advances in Neural Information Processing Systems},
  volume    = {36},
  pages     = {69253--69269},
  year      = {2023},
  doi       = {10.52202/075280-3032},
  url       = {https://proceedings.neurips.cc/paper_files/paper/2023/hash/daa098aa8e1fc718943ff1ab7b5b30c9-Abstract-Conference.html}
}

@inproceedings{sakaridis2021acdc,
  author    = {Sakaridis, Christos and Dai, Dengxin and Van Gool, Luc},
  title     = {{ACDC}: The Adverse Conditions Dataset with Correspondences for Semantic Driving Scene Understanding},
  booktitle = {2021 {IEEE/CVF} International Conference on Computer Vision ({ICCV})},
  pages     = {10745--10755},
  year      = {2021},
  publisher = {IEEE},
  doi       = {10.1109/ICCV48922.2021.01059}
}

@inproceedings{sagawa2020distributionally,
  author    = {Sagawa, Shiori and Koh, Pang Wei and Hashimoto, Tatsunori B. and Liang, Percy},
  title     = {Distributionally Robust Neural Networks for Group Shifts: On the Importance of Regularization for Worst-Case Generalization},
  booktitle = {International Conference on Learning Representations},
  year      = {2020},
  url       = {https://openreview.net/forum?id=ryxGuJrFvS}
}

@article{shafieezadeh2019regularization,
  author  = {Shafieezadeh-Abadeh, Soroosh and Kuhn, Daniel and Mohajerin Esfahani, Peyman},
  title   = {Regularization via Mass Transportation},
  journal = {Journal of Machine Learning Research},
  volume  = {20},
  number  = {103},
  pages   = {1--68},
  year    = {2019},
  url     = {https://jmlr.org/papers/v20/17-633.html}
}

@book{sugiyama2012machine,
  author    = {Sugiyama, Masashi and Kawanabe, Motoaki},
  title     = {Machine Learning in Non-Stationary Environments: Introduction to Covariate Shift Adaptation},
  publisher = {MIT Press},
  address   = {Cambridge, MA},
  year      = {2012},
  doi       = {10.7551/mitpress/9780262017091.001.0001}
}

@inproceedings{sutter2021robust,
  author    = {Sutter, Tobias and Krause, Andreas and Kuhn, Daniel},
  title     = {Robust Generalization Despite Distribution Shift via Minimum Discriminating Information},
  booktitle = {Advances in Neural Information Processing Systems},
  volume    = {34},
  pages     = {29754--29767},
  year      = {2021},
  url       = {https://proceedings.neurips.cc/paper/2021/hash/f86890095c957e9b949d11d15f0d0cd5-Abstract.html}
}

@article{wang2023equivalence,
  author  = {Wang, Zhiyuan and Ran, Lun and Zhou, Minglong and He, Long},
  title   = {On the Equivalence and Performance of Distributionally Robust Optimization and Robust Satisficing Models},
  journal = {Manufacturing \& Service Operations Management},
  volume  = {27},
  number  = {4},
  pages   = {1295--1312},
  year    = {2025},
  doi     = {10.1287/msom.2023.0531}
}

@article{wang2023risk,
  author  = {Wang, Duo and Yang, Kai and Yang, Lixing},
  title   = {Risk-Averse Two-Stage Distributionally Robust Optimisation for Logistics Planning in Disaster Relief Management},
  journal = {International Journal of Production Research},
  volume  = {61},
  number  = {2},
  pages   = {668--691},
  year    = {2023},
  doi     = {10.1080/00207543.2021.2013559}
}

@article{liu2024rethinking,
  author  = {Wang, Tianyu and Liu, Jiashuo and Cui, Peng and Namkoong, Hongseok},
  title   = {Rethinking Distribution Shifts: Empirical Analysis and Modeling for Tabular Data},
  journal = {Management Science},
  year    = {2026},
  url     = {https://arxiv.org/abs/2307.05284}
}

@article{yu2022external,
  author  = {Yu, Alice C. and Mohajer, Bahram and Eng, John},
  title   = {External Validation of Deep Learning Algorithms for Radiologic Diagnosis: A Systematic Review},
  journal = {Radiology: Artificial Intelligence},
  volume  = {4},
  number  = {3},
  pages   = {e210064},
  year    = {2022},
  doi     = {10.1148/ryai.210064}
}

@article{zech2018variable,
  author  = {Zech, John R. and Badgeley, Marcus A. and Liu, Manway and Costa, Anthony B. and Titano, Joseph J. and Oermann, Eric Karl},
  title   = {Variable Generalization Performance of a Deep Learning Model to Detect Pneumonia in Chest Radiographs: A Cross-Sectional Study},
  journal = {PLOS Medicine},
  volume  = {15},
  number  = {11},
  pages   = {e1002683},
  year    = {2018},
  doi     = {10.1371/journal.pmed.1002683}
}

@book{zhang2023mathematical,
  author    = {Zhang, Tong},
  title     = {Mathematical Analysis of Machine Learning Algorithms},
  publisher = {Cambridge University Press},
  address   = {Cambridge, UK},
  year      = {2023},
  doi       = {10.1017/9781009093057}
}

\newpage
\begin{APPENDICES}

    \section{Supplementary Theoretical Analyses under an  Adversarial Setting}\label{sec:adversarial}

    We now isolate the case in which the known shift direction coincides with the worst-case direction
    of the DRO ambiguity set. For each radius $r\ge 0$, define
    \[
        \mathcal A(r)
        \;:=\;\arg\max_{P \in \mathcal{B}(P_S, r)} \,\E_{P}\big[f(x_S, z)\big],
    \]
    the set of distributions in the Wasserstein ball that maximize the loss of the source-population
    optimizer $x_S$. We call a radius-indexed family $\{P_r^{\mathrm{adv}}\}_{r\ge 0}$ an adversarial
    direction if
    \[
        P_r^{\mathrm{adv}}
        \in
        \arg\max_{P\in \mathcal A(r)} d_W(P,P_S),
    \]
    that is, among the maximizers, we choose the distribution farthest from the source distribution $P_S$ in the $1$-Wasserstein metric $d_W(\cdot,\cdot)$ (using a fixed tie-breaking rule whenever the maximizer is non-unique). 
    We say that the known direction is adversarial over the range of magnitudes
    under consideration if the anticipated distributions satisfy $ P_t = P^{\mathrm{adv}}_{r_t}$, $t\ge 0$
    for a nondecreasing map $t\mapsto r_t$. The true target distribution is then
    $P_T=P_{t_T}$ for some unknown $t_T$.
    
    In this adversarial scenario, the specific anticipated
    target $P_t$ is also a worst-case distribution in the ball $\mathcal{B}(P_S,r_t)$, so the regularization gap $\sup_{P\in \mathcal{B}(P_S,r_t)}\mathbb{E}_P[f(x_S,z)]-\mathbb{E}_{P_t}\!\big[f(x_S,z)\big]$ in
    Proposition~\ref{prop:reg_comparison} vanishes.
     Thus, the RS regularization penalty is asymptotically
    comparable to the DRO regularization penalty. The remaining difference lies in the sensitivity-to-shift components. As the radius \(r_t\) increases along the
    adversarial path, the Wasserstein ball \(\mathcal B(P_S,r_t)\) becomes closer to covering the true target distribution \(P_T=P_{t_T}\), and fully contains it
    when \(t\geq t_T\). Hence, the DRO sensitivity term becomes small or even
    vanishes, while the RS sensitivity term \(k_{\tau_t}d_W(P_S,P_T)\) remains
    non-negligible. This relationship is formalized in the following proposition.
    
    \begin{proposition}\label{prop:well_overspefified_case}
    In the adversarial scenario, under Assumptions~\ref{assump:regularity},
    \ref{assump:strong convex and lipschitz}, and~\ref{monotonicity}, suppose either:
    (i) \(t<t_T\) and
    \(\inf_{0\leq s\leq t} d_W(P_T,P_s)
    \leq \frac{k_{\tau_t}}{L}d_W(P_S,P_T)\), or
    (ii) \(t\geq t_T\). Then, with probability at least \(1-\delta\), 
    \begin{align*}
        \mathrm{Reg}_{\mathrm{DRO}}(r_t)
        \leq
        \mathrm{Reg}_{\mathrm{RS}}(\tau_t)
        +
        \bar{\rho}_n(\delta),  ~~ 
        \mathrm{Sen}_{\mathrm{DRO}}(r_t)
        \leq
        \mathrm{Sen}_{\mathrm{RS}}(\tau_t).
    \end{align*}
    where \(\bar{\rho}_n(\delta)\) is the statistical error defined in
    Proposition~\ref{prop:reg_comparison}.
    \end{proposition}
    
    Combining the two component-wise inequalities, Proposition~\ref{prop:well_overspefified_case}
    implies the aggregate comparison
    \(\mathrm{TO}_{\mathrm{DRO}}(r_t)\leq \mathrm{TO}_{\mathrm{RS}}(\tau_t)+\bar{\rho}_n(\delta)\).
    Thus, in the adversarial scenario, DRO yields a tighter trade-off bound in
    mildly under-specified, well-specified, and over-specified regimes.
    Similarly, the condition
    \(\inf_{0\leq s\leq t} d_W(P_T,P_s)
    \leq \frac{k_{\tau_t}}{L}d_W(P_S,P_T)\)
    for \(t<t_T\) involves \(t\) on both sides and can be tricky to verify.
    Still, if \(\mathbb E_{P_t}[f(\hat x_{\mathrm{ERM}},z)]\) increases
    monotonically in \(t\), then a simple sufficient condition is
    \(\inf_{0\leq s\leq t} d_W(P_T,P_s)
    \leq \frac{k_{\tau_{t_T}}}{L}d_W(P_S,P_T)\).

    Nevertheless, this result does not imply that one should always enlarge the DRO radius to intentionally make the well-specified or over-specified cases occur. Although increasing the ambiguity radius may improve DRO's generalization guarantee relative to RS in adversarial settings, real distributional shifts are often not aligned with the worst-case direction. In  non-adversarial scenarios, the gap $\sup_{P\in \mathcal{B}(P_s,r_t)}\mathbb{E}_P[f(x_S,z)]-\mathbb{E}_{P_t}\!\big[f(x_S,z)\big]$ in Proposition~\ref{prop:reg_comparison} may be non-negligible, meaning that DRO pays for conservative worst-case distributions. This is consistent with prior studies showing that, as the ambiguity set grows, DRO's focus on worst-case distributions can lead to performance deterioration and diminishing regularization benefits~\citep{esfahani2015data,long2023robust}. In such cases, RS may avoid excessive conservatism and achieve a better balance between robustness and performance than DRO, as we shall shortly show in simulations.

    \section{Additional Examples Satisfying the Monotonicity Assumption \ref{monotonicity}}\label{sec:monotonicity_assump}

\begin{example}[Interpolation between two distributions]
Consider the convex combination of two distributions $P_0$ and $P_1$: 
\[
P_t = (1 - t)P_0 + tP_1 \quad \text{for } t \in [0, 1].
\]  
It follows that
$d_{W}(P_t, P_0) = t \cdot d_{W}(P_0, P_1)$,
i.e., the type-I Wasserstein distance from $P_t$ to $P_0$ grows linearly in $t$ with slope equal to the distance between $P_0$ and $P_1$. This example corresponds to the two-domain special case of one empirical benchmark for distribution shifts in  \cite{koh2021wilds}. In this sense, our result fills a small theoretical gap by providing an analysis of this two-support setting through the Wasserstein lens, while also helping bridge the ML and OR literatures. 
\end{example}

\begin{example}[Stochastic Differential Equations]
We can also consider certain stochastic differential equations (SDEs) as generating monotone distributional shifts. This example is related to \citet{dai2025assured}, who use flow-based generative models to characterize distributional shifts through ordinary differential equations (without a Wiener process). Consider the SDE
\begin{equation*}
    dX_t = g(t, X_t)dt + \sigma(t, X_t)dW_t,
\end{equation*} 
where $g(t, X_t)$ is the drift term, $\sigma(t, X_t)$ the diffusion term, and $W_t$ a Wiener process.  
An initial distribution $X_0 \sim P_0$ evolves into $X_t \sim P_t$ according to the dynamics of this SDE.  
In such cases, Assumption~\ref{monotonicity} reduces to a condition imposed on the drift term $g(t, X_t)$ and diffusion term $\sigma(t,X_t)$, which is satisfied for many common stochastic processes. 
As an illustration, suppose $X_0 \sim \mathcal{N}(1, \sigma^2)$ and $X_t$ follows the Ornstein–Uhlenbeck-type SDE
\begin{equation*}
    dX_t = (2 + t - X_t)dt + \sqrt{2} \sigma  dW_t.
\end{equation*} 
It can be shown that $X_t \sim \mathcal{N}(1+t, \sigma^2)$,  
which describes a distributional shift where the variance remains constant while the mean grows linearly with time $t$.

\end{example}
    \section{Supplementary Experimental Results}\label{app: experiment}
    
    \subsection{Dual Formulation of Robust Methods}\label{appd:network lotsizing reformulation}
    We present the tractable approximations used to determine the initial ordering decision under the two robust frameworks.
    The original adaptive formulations involve recourse decisions that depend on the realized demand and are generally intractable.
    Following \citet{long2023robust} and \citet{wang2023equivalence}, we adopt the scenario-wise lifted affine recourse adaptation, under which the recourse policies are affine functions of the lifted uncertainty vector $(z,u)$.
    The resulting robust counterparts can be solved efficiently by off-the-shelf optimization solvers.
    In our experiments, we use \texttt{Gurobi}.
    
    \paragraph{Distributionally robust optimization.}
    For a Wasserstein radius $r$, the DRO-LDR approximation is given by
    \begin{align*}
    \min \quad 
    & c^\top x + kr + \frac{1}{S}\sum_{s\in[S]} v_s \\
    \text{s.t.} \quad 
    & \sum_{i\in[N]}
    \left( d_i^\top y_i^{(s)}(z,u) + l^\top w^{(s)}(z,u) \right)
    - ku - v_s \leq 0,
    && \forall (z,u)\in \bar{\mathcal Z}_s,\ \forall s\in[S],\\
    & x_i + w_i^{(s)}(z,u)
    + \sum_{j\in[N]} y_{ji}^{(s)}(z,u)
    - \sum_{j\in[N]} y_{ij}^{(s)}(z,u)
    - z_i \geq 0,
    && \forall (z,u)\in \bar{\mathcal Z}_s,\ \forall s\in[S],\ \forall i\in[N],\\
    & y^{(s)}(z,u)\geq 0,\quad w^{(s)}(z,u)\geq 0,
    && \forall (z,u)\in \bar{\mathcal Z}_s,\ \forall s\in[S],\\
    & 0\leq x_i\leq \delta_i,
    && \forall i\in[N],\\
    & y^{(s)}\in \mathcal L^{N+1,N\times N},\quad
    w^{(s)}\in \mathcal L^{N+1,N},
    && \forall s\in[S].
    \end{align*}
    
    Here
    \[
    \bar{\mathcal Z}_s :=
    \left\{
    (z,u)\in \mathcal Z\times\mathbb R
    \ \middle|\ 
    \|z-\hat z_s\|_1\leq u
    \right\}.
    \]
    
    \paragraph{Robust satisficing.}
    For a target value $\tau$, the RS-LDR approximation is given by
    \begin{align*}
    \min \quad 
    & k \\
    \text{s.t.} \quad
    & c^\top x + \frac{1}{S}\sum_{s\in[S]} v_s \leq \tau,\\
    & \sum_{i\in[N]}
    \left( d_i^\top y_i^{(s)}(z,u) + l^\top w^{(s)}(z,u) \right)
    - ku - v_s \leq 0,
    && \forall (z,u)\in \bar{\mathcal Z}_s,\ \forall s\in[S],\\
    & x_i + w_i^{(s)}(z,u)
    + \sum_{j\in[N]} y_{ji}^{(s)}(z,u)
    - \sum_{j\in[N]} y_{ij}^{(s)}(z,u)
    - z_i \geq 0,
    && \forall (z,u)\in \bar{\mathcal Z}_s,\ \forall s\in[S],\ \forall i\in[N],\\
    & y^{(s)}(z,u)\geq 0,\quad w^{(s)}(z,u)\geq 0,
    && \forall (z,u)\in \bar{\mathcal Z}_s,\ \forall s\in[S],\\
    & 0\leq x_i\leq \delta_i,
    && \forall i\in[N],\\
    & y^{(s)}\in \mathcal L^{N+1,N\times N},\quad
    w^{(s)}\in \mathcal L^{N+1,N},
    && \forall s\in[S].
    \end{align*}
    
    \subsection{95\% Quantiles of Total Costs}

    Figure~\ref{fig:total-q95} reports the $95\%$ quantile of total costs under the same experimental setting as Figure~\ref{fig:total-mean}. This tail-risk comparison is not directly covered by the expected-risk generalization bounds developed in Sections~\ref{sec: theoretical bounds} and~\ref{sec: partial}, which do not provide explicit tail bounds. Nevertheless, it provides a useful robustness check. In contrast to the mean-cost comparison, where the relative performance of DRO and RS changes with the specified shift magnitude, the tail-risk comparison is more favorable to RS. Across the three values of $c_i$, RS is generally below DRO over most shift-magnitude specifications. This suggests that RS provides more stable upper-tail performance in this lot-sizing experiment, even in regimes where DRO may be competitive in terms of mean cost.

    \begin{figure}[t]
      \centering
      \includegraphics[width=\linewidth]{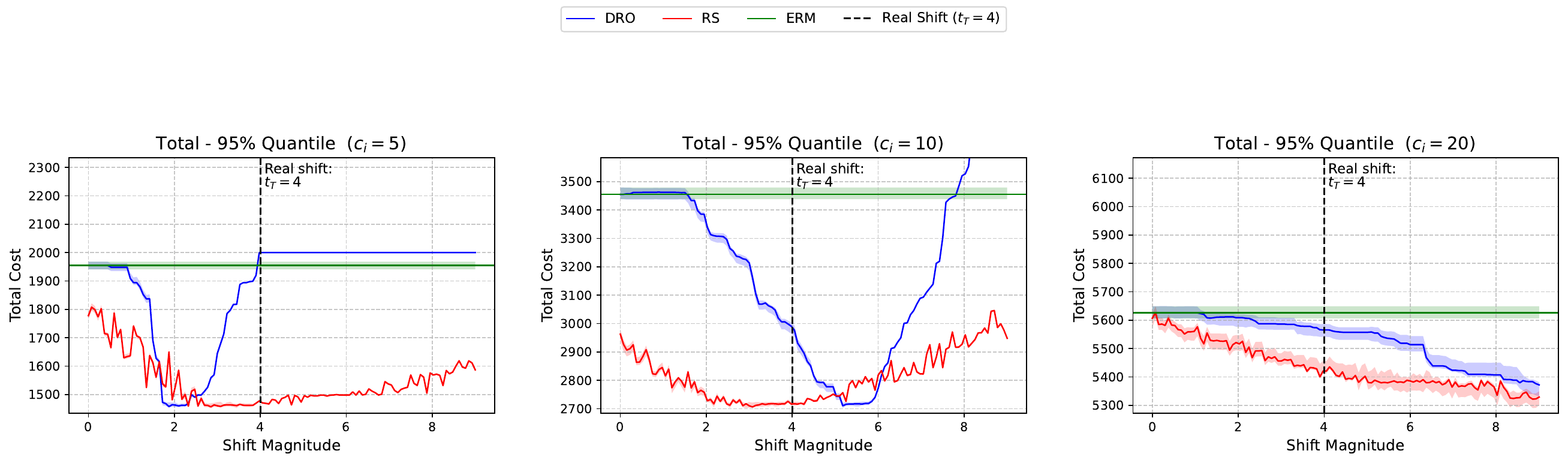}
      \caption{$95\%$ quantile of total costs of different methods for the network lot-sizing problem across different shift magnitude specifications ($t_T=4$ corresponds to the truth). The shift direction is known and used to calibrate DRO and RS hyperparameters, while the shift magnitude is unknown and specified along the horizontal axis. We consider three different levels of per-unit initial-ordering cost $c_i\in\{5,10,20\}$.}
      \label{fig:total-q95}
    \end{figure}

    \subsection{95\% Quantiles of Initial Ordering and Operational Costs}\label{appendix:95 Quantiles of Initial Ordering and Operational Costs}

In Section~\ref{sec:initial vs operational}, we decompose the mean total cost into initial-ordering and operational costs. Figure~\ref{fig:initial and operational panels (95 quantile)} reports the corresponding 95\% quantile decomposition and shows a similar high-level pattern: robust methods tend to incur higher initial-ordering costs but lower operational costs, with RS less sensitive to changes in $c_i$.
    
\begin{figure}[htbp]
    \centering
    \includegraphics[width=0.9\linewidth]{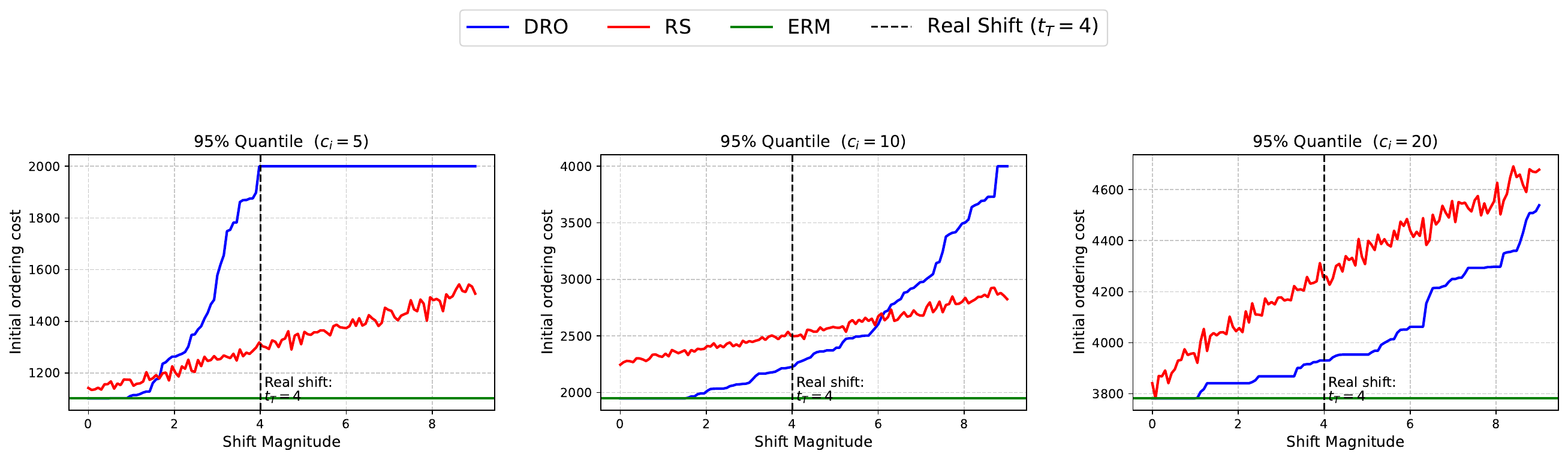}\\[0.5em]
    \includegraphics[width=0.9\linewidth]{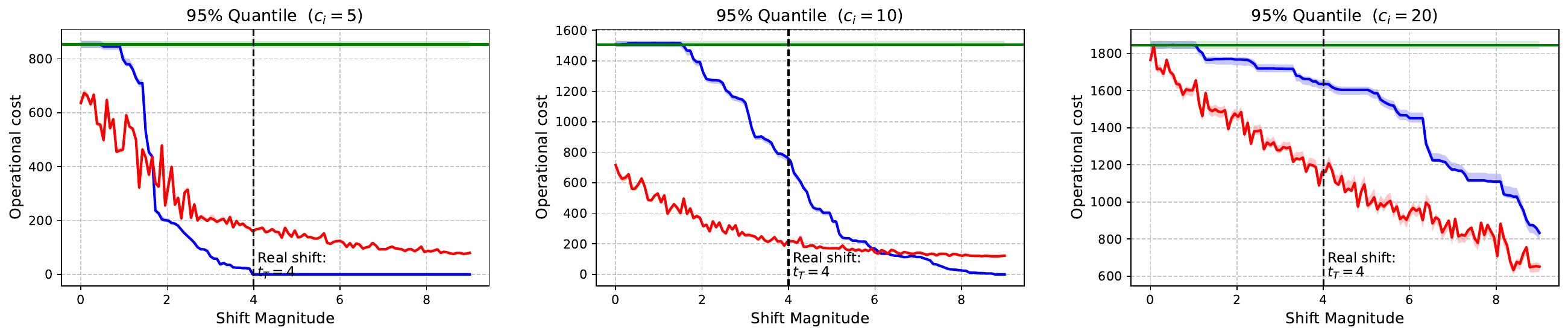}
    \caption{95\% quantile cost decomposition for the network lot-sizing application. The top panel reports initial-ordering costs, and the bottom panel reports operational costs (transshipment costs plus emergency-order costs).}
    \label{fig:initial and operational panels (95 quantile)}
\end{figure}
    \subsection{Dominance of Emergency Order Costs over Transportation Costs}
    \label{appendix:operation_cost_decomposition}
    \begin{figure}[htbp]
      \centering
      \includegraphics[width=0.8\linewidth]{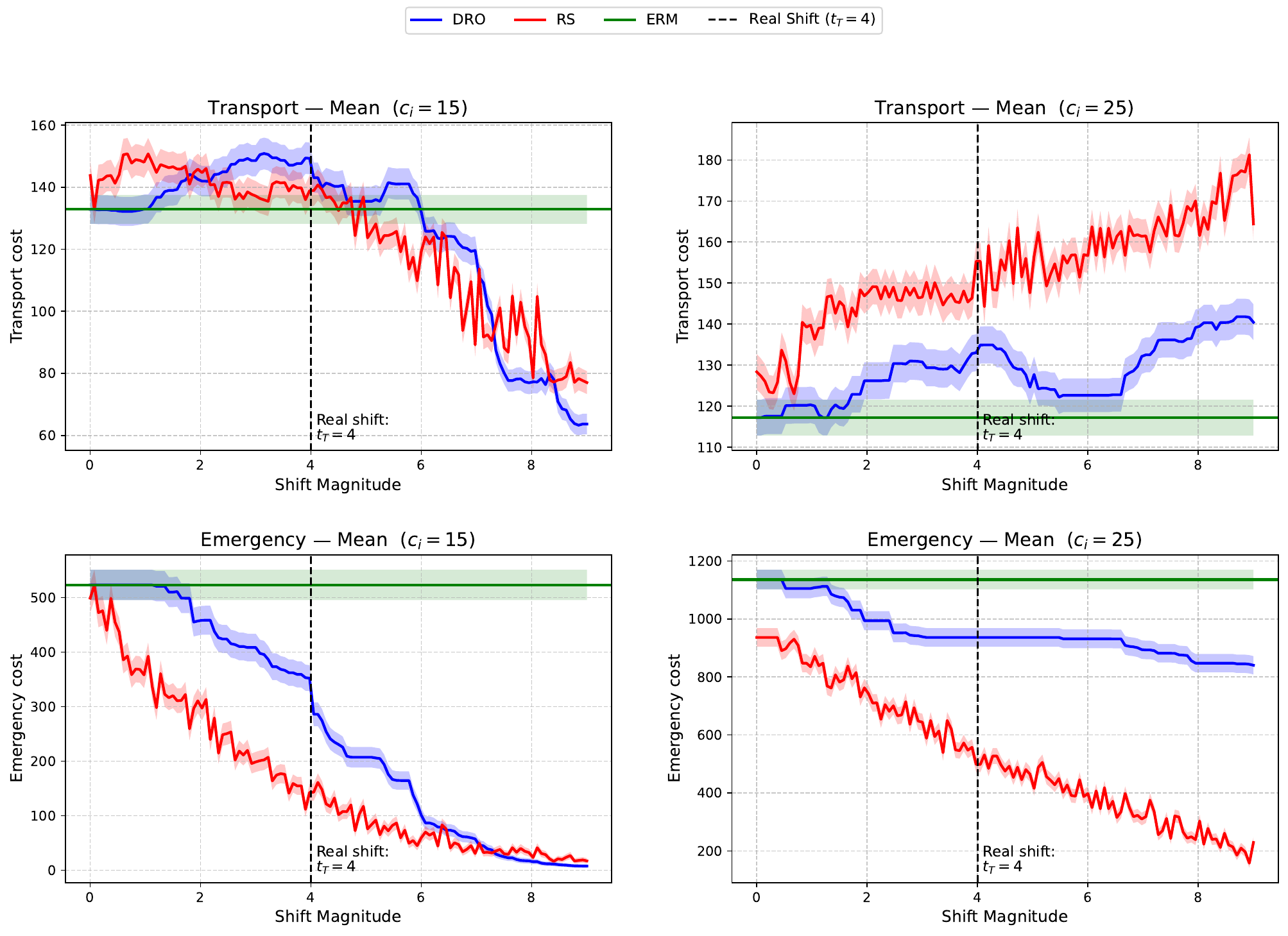}
      \caption{Mean operational-cost decomposition for the network lot-sizing application. The top panel reports transportation costs, and the bottom panel reports emergency-order costs, under two larger levels of per-unit initial-ordering cost \(c_i\in\{15,25\}\). Together, they further decompose the operational costs.}
      \label{fig:c10c20-transport-emergency}
    \end{figure}

    In Section \ref{sec:lot_sizing_setup},  for a given  initial allocation $\hat{x}$ and a realized random demand $\bar{z}$, 
    we decompose the \emph{total cost}  $f(\hat{x},\bar{z})$  
    into  the \emph{initial-ordering cost} $c^T\hat{x}$ and the \emph{operational cost}, 
    which corresponds to the optimal value of the following second-stage problem:
    \[
    \min_{(y,w)\in \mathcal{Y}_{(\hat{x},\bar{z})}}\Big\{\sum_{i\in [N]} d_i^Ty_i + l^T w\Big\}.
    \]
    Now let $(\bar{y},\bar{w})$ denote the optimal solution to this problem. 
    We can further decompose the operational cost into two parts: the \emph{transportation cost}  $\sum_{i\in [N]} d_i^T \bar{y}_i$, 
    and the \emph{emergency-order cost}  $l^T\bar{w}$. 
    Hence, the total cost can be expressed as
    \[
    f(\hat{x},\bar{z}) = c^T\hat{x} + \sum_{i\in [N]} d_i^T \bar{y}_i + l^T\bar{w}.
    \]
    
    Using the cases 
    $c_i=15$ and $c_i=25$ as illustrative examples, we observe the following patterns.
    First, in these supplementary lot-sizing instances, RS achieves a stronger reduction in emergency-order costs than DRO, indicating better control of this high-cost corrective component.
    Second, an interesting observation arises when comparing the scales of the two cost components: transportation costs are substantially smaller and thus play a relatively minor role. In fact, variations in emergency-order costs almost entirely account for the overall changes in operational costs, especially when $c$ is large.
    This phenomenon can be explained by the differing scaling behaviors of the two components. The transportation cost reflects local redistribution and scales with spatial variability, whereas the emergency-order cost scales with the aggregate demand surplus relative to the planned supply. Consequently, under the mean-increasing shift scenario, the emergency-order cost begins dominant and shows significant changes, while the transportation cost remains comparatively small.
    
    \subsection{Replicated Experiments}

We further conduct a replicated experiment under the same network lot-sizing setup, but with the true shift magnitude changed to \(t_T=3\). The results reported below show that the main qualitative observations remain broadly consistent.

    \begin{figure}[htbp]
      \centering
      \begin{subfigure}{\linewidth}
        \centering
        \includegraphics[width=0.9\linewidth]{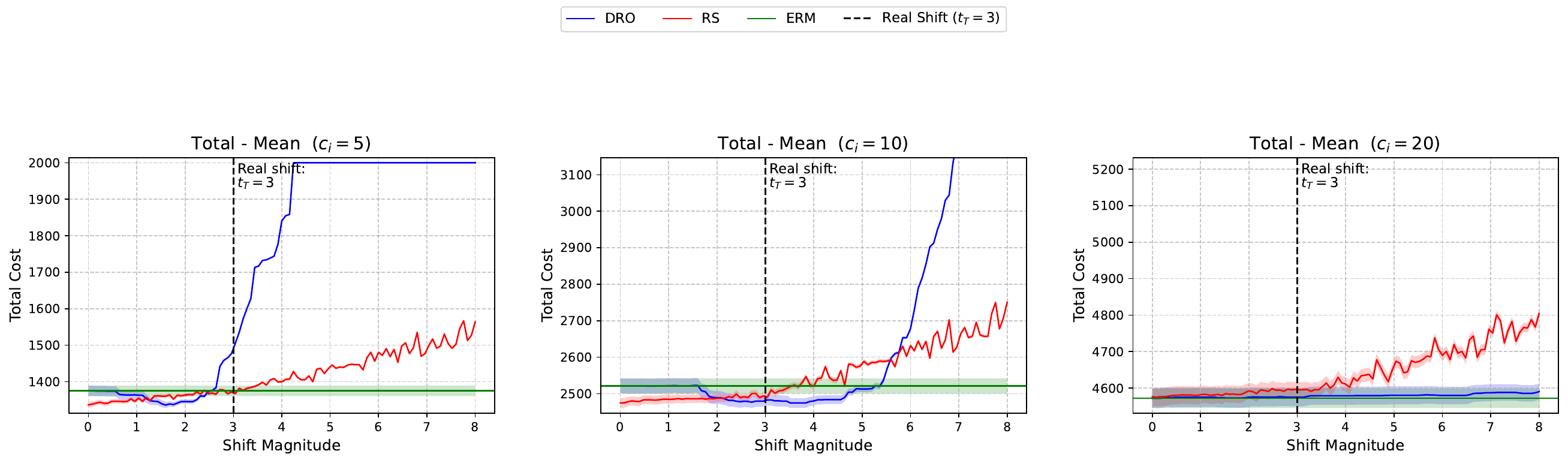}
        \captionsetup{font=footnotesize}
        \caption{Mean of total costs across shift magnitude $t$}
        \label{fig:(additional)total-mean}
      \end{subfigure}
    
      \vspace{0.8em}
    
      \begin{subfigure}{\linewidth}
        \centering
        \includegraphics[width=0.9\linewidth]{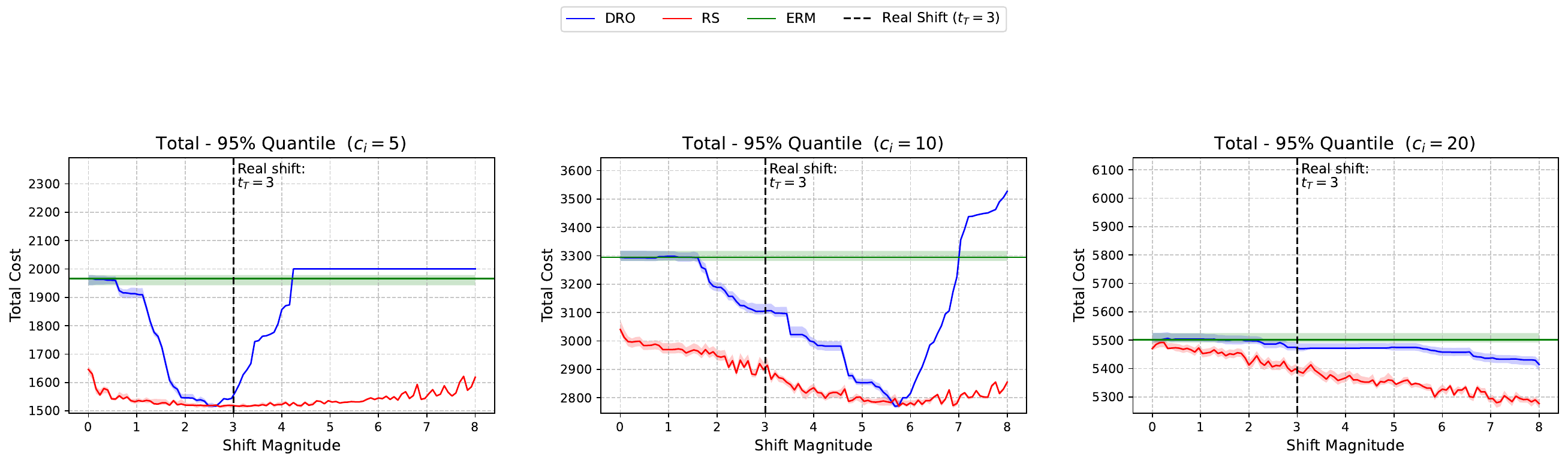}
        \captionsetup{font=footnotesize}
        \caption{95\% quantile of total costs across shift magnitude $t$}
        \label{fig:(additional)total-q95}
      \end{subfigure}
    
     \caption{Mean and 95\% quantile of total costs for the replicated network lot-sizing experiment with true shift magnitude \(t_T=3\). The top panel reports mean total costs, and the bottom panel reports the 95\% quantile of total costs, across different shift magnitude specifications. The shift direction is known and used to calibrate DRO and RS hyperparameters, while the shift magnitude is unknown and specified along the horizontal axis. We consider three different levels of per-unit initial-ordering cost \(c_i\in\{5,10,20\}\).}
      \label{fig:(additional)total-panels}
    \end{figure}

    \begin{figure}[htbp]
    \centering
    \includegraphics[width=0.9\linewidth]{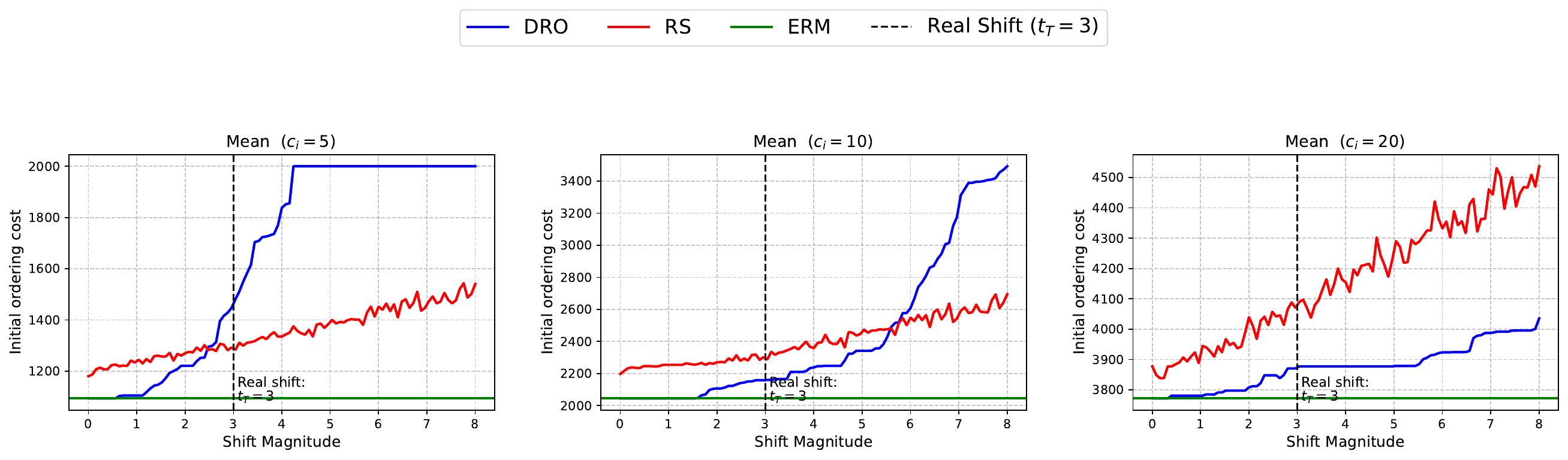}\\[0.5em]
    \includegraphics[width=0.9\linewidth]{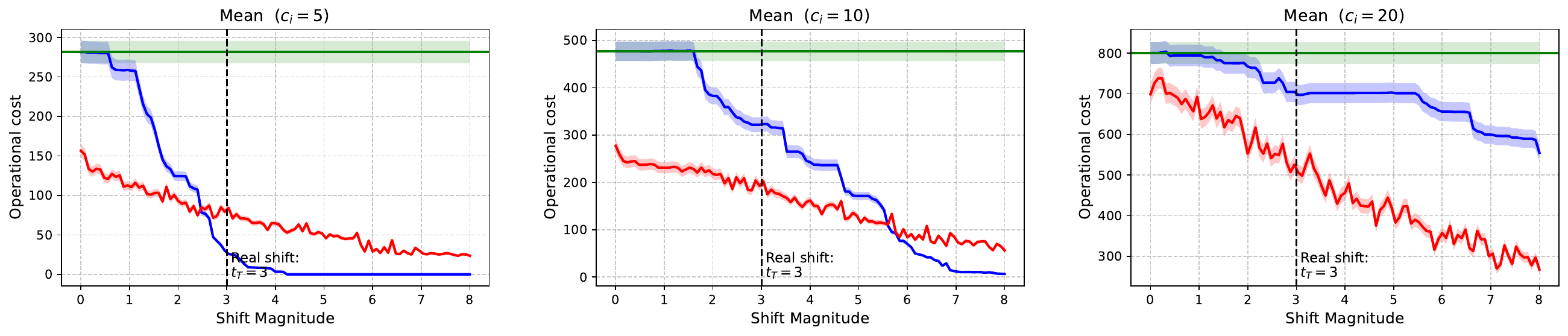}
    \caption{Mean cost decomposition for the network lot-sizing application. The top panel reports initial-ordering costs, and the bottom panel reports operational recourse costs (transshipment costs plus emergency-order costs). }
\end{figure}

    \begin{figure}[htbp]
    \centering
    \includegraphics[width=0.9\linewidth]{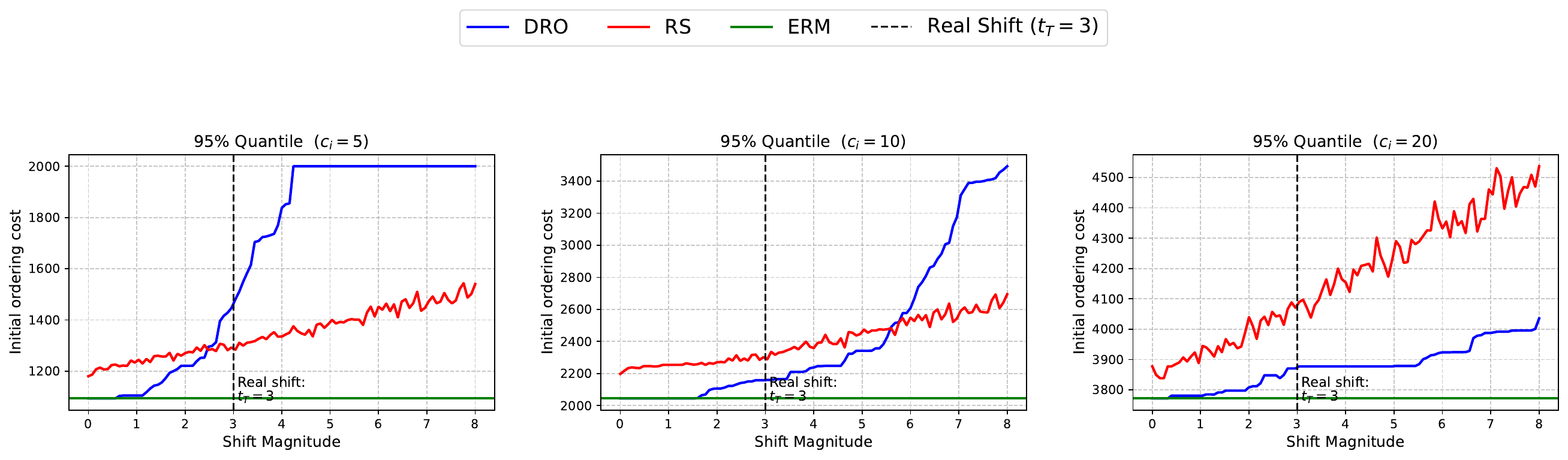}\\[0.5em]
    \includegraphics[width=0.9\linewidth]{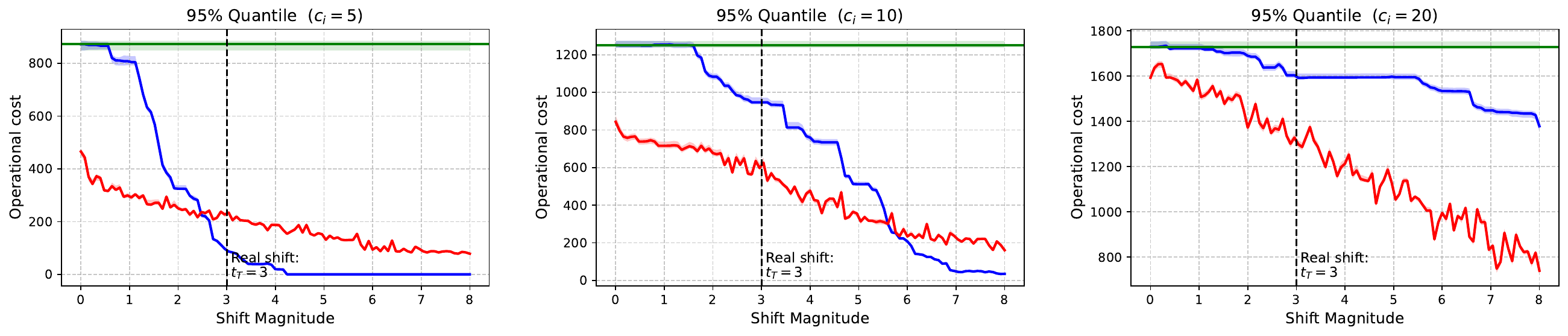}
    \caption{95\% quantile cost decomposition for the network lot-sizing application. The top panel reports initial-ordering costs, and the bottom panel reports operational costs (transshipment costs plus emergency-order costs).}
\end{figure}
    
    \begin{figure}[htbp]
      \centering
      \includegraphics[width=\linewidth]{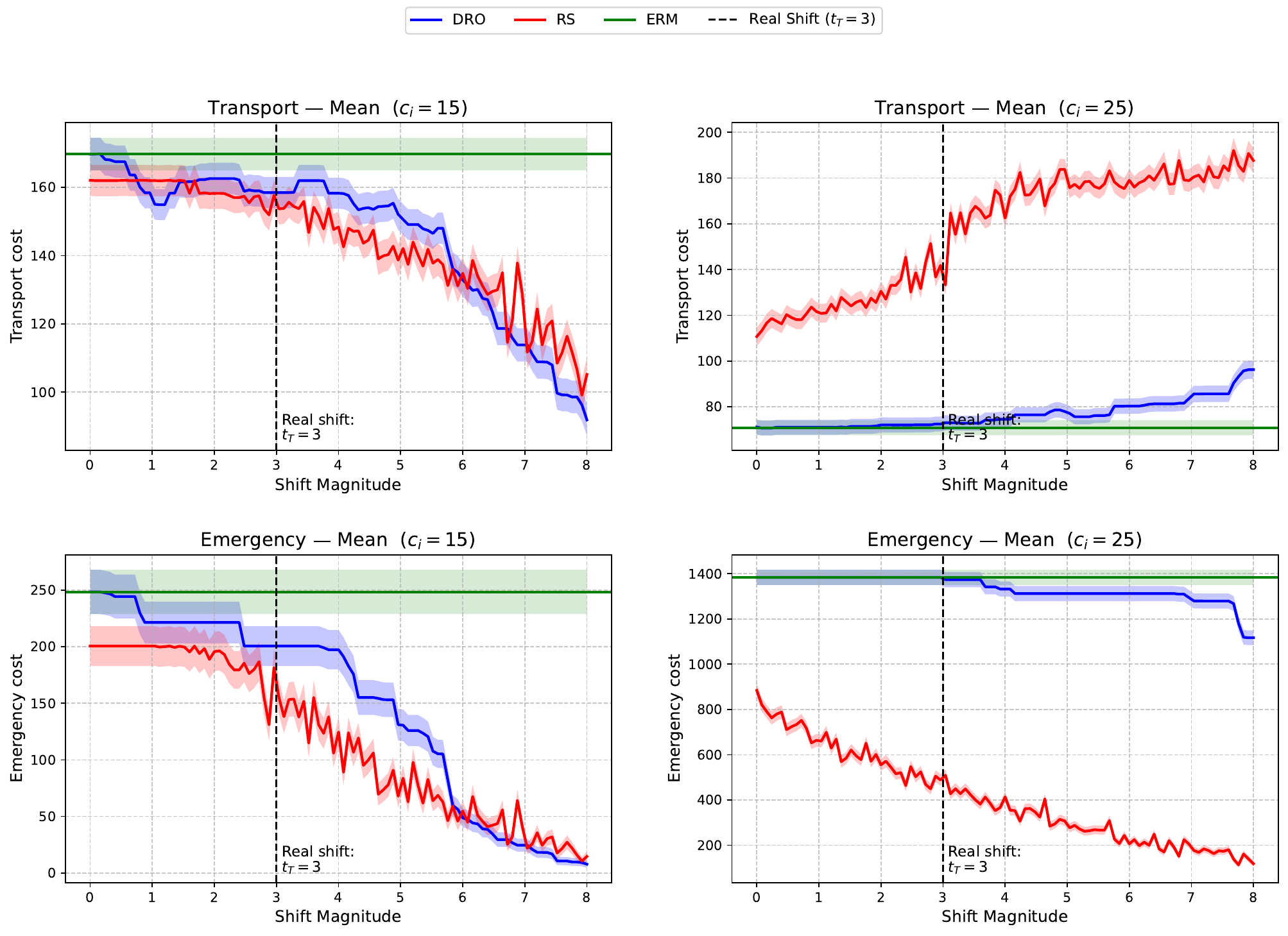}
      \caption{Mean operational-cost decomposition for the network lot-sizing application. The top panel reports transportation costs, and the bottom panel reports emergency-order costs, under two larger levels of per-unit initial-ordering cost \(c_i\in\{15,25\}\). Together, they further decompose the operational costs.}
      \label{fig:(additional)c15c25-transport-emergency}
    \end{figure}

\section{Verification of Assumptions for Common Losses}
\label{app:examples_assumptions}

This section gives simple sufficient conditions under which Assumptions~\ref{assump:regularity} and~\ref{assump:strong convex and lipschitz} hold for several common losses. The goal is not to provide the weakest possible conditions, but rather to show that the assumptions are standard and can be verified for some common loss functions. Throughout this section, we assume that the feasible set \(\mathcal X\) is compact. When Assumption~\ref{assump:regularity} requires a bounded instance space, we also assume that the instance space \(\mathcal Z\) is compact, or equivalently that the distribution is supported on a compact subset of the original sample space. For unbounded parametric families such as Gaussian demands, this can be interpreted as working with a truncated Gaussian approximation on a sufficiently large compact set. This convention guarantees boundedness of continuous losses; hence the main points to verify are the Lipschitz conditions and the second-order growth condition.

\paragraph{Newsvendor loss.}
Consider the \(d\)-product newsvendor loss
\[
f(x,z)=\sum_{j=1}^d h_j(x_j-z_j)_+ + b_j(z_j-x_j)_+,
\]
where \(h_j,b_j>0\). This loss is continuous in \((x,z)\). If \(\mathcal X\) and \(\mathcal Z\) are compact, then \(f\) is uniformly bounded, so Assumption~\ref{assump:regularity}(b) holds. Moreover, for any fixed \(x\),
\[
|f(x,z)-f(x,z')|
\le
\sum_{j=1}^d \max\{h_j,b_j\}|z_j-z'_j|
\le
L_z\|z-z'\|_2,
\]
where \(L_z=(\sum_{j=1}^d \max\{h_j,b_j\}^2)^{1/2}\). Hence Assumption~\ref{assump:regularity}(c) holds. Similarly, for any fixed \(z\),
\[
|f(x,z)-f(x',z)|
\le
\sum_{j=1}^d \max\{h_j,b_j\}|x_j-x'_j|
\le
L_x\|x-x'\|_2,
\]
with the same choice \(L_x=(\sum_{j=1}^d \max\{h_j,b_j\}^2)^{1/2}\). Thus the parameter-Lipschitz condition in Assumption~\ref{assump:strong convex and lipschitz}(a) holds.

For the second-order growth condition. For
\(
\mathcal E(x)=\mathbb E_{P_S}[f(x,z)].
\)
Suppose the marginal distribution of \(Z_j\) under \(P_S\) has density \(p_j\), and there exists \(m_j>0\) such that
\[
p_j(t)\ge m_j
\]
for all \(t\) in the projection of \(\mathcal X\) onto coordinate \(j\). Then, wherever the derivative exists,
\[
\frac{\partial^2 \mathcal E(x)}{\partial x_j^2}
=
(h_j+b_j)p_j(x_j).
\]
Hence
\[
\nabla^2 \mathcal E(x)
\succeq
\alpha I,
\qquad
\alpha:=\min_{1\le j\le d}(h_j+b_j)m_j>0.
\]
Therefore \(\mathcal E\) is strongly convex on \(\mathcal X\), and its restriction to any convex feasible set, such as the capacity set
\[
\mathcal X=\{x\in\mathbb R_+^d:\mathbf 1^\top x\le C\},
\]
satisfies
\[
\mathcal E(x)\ge \mathcal E(x_S)+\frac{\alpha}{2}\|x-x_S\|_2^2,
\qquad \forall x\in\mathcal X.
\]
This verifies Assumption~\ref{assump:strong convex and lipschitz}(b). In particular, for the Gaussian mean-shift newsvendor model used in the numerical study, the Gaussian density is strictly positive on every compact interval, so the above density lower bound holds on the bounded feasible region. 

\paragraph{Logistic loss.}
Consider binary logistic loss
\[
f(x;(a,y))=\log(1+\exp(-y a^\top x)),
\qquad y\in\{-1,1\}.
\]
Assume \(\mathcal{X}\) is compact and the covariates are bounded, \(\|a\|_2\le B\). Then \(f\) is continuous and uniformly bounded on the compact set \(\mathcal{X}\times\mathcal Z\). Moreover,
\[
\|\nabla_x f(x;(a,y))\|_2\le \|a\|_2\le B,
\]
so the parameter-Lipschitz condition holds with \(L'=B\). The loss is also Lipschitz in the instance variable under any metric that combines the Euclidean distance between covariates and the discrete distance between labels; boundedness of \(\mathcal{X}\) gives a uniform Lipschitz constant.

For the second-order growth condition, suppose the source design is nondegenerate:
\[
\mathbb E_{P_S}[aa^\top]\succeq \lambda I
\]
for some \(\lambda>0\). Since \(\mathcal{X}\) and the covariates are bounded, there exists \(c>0\) such that
\[
\sigma(a^\top x)(1-\sigma(a^\top x))\ge c,
\qquad \forall x\in\mathcal{X},
\]
where \(\sigma(t)=(1+\exp(-t))^{-1}\). Therefore
\[
\nabla^2 \mathcal E(x)
=
\mathbb E_{P_S}\!\left[
\sigma(a^\top x)(1-\sigma(a^\top x))aa^\top
\right]
\succeq
c\lambda I.
\]
Thus the source risk satisfies the second-order growth condition on \(\mathcal{X}\). The condition can fail without such nondegeneracy, for example if the covariates are rank deficient.

\paragraph{Huber loss.}
Consider the Huber regression loss
\[
f(x;(a,y))=\rho_\delta(y-a^\top x),
\]
where \(\rho_\delta\) is the Huber loss with threshold \(\delta>0\). Assume \(\mathcal{X}\) is compact and the instance space is bounded, e.g., \(\|a\|_2\le B\) and \(|y|\le Y\). Then \(f\) is continuous and uniformly bounded. Since \(|\rho_\delta'(r)|\le \delta\), we have
\[
\|\nabla_x f(x;(a,y))\|_2
=
|\rho_\delta'(y-a^\top x)|\|a\|_2
\le
\delta B,
\]
so the parameter-Lipschitz condition holds. The same boundedness also implies Lipschitz continuity in the instance variable \((a,y)\), uniformly over \(x\in\mathcal{X}\).

For the second-order growth condition, it is sufficient to assume that the source distribution has enough mass in the quadratic region of the Huber loss and that the corresponding weighted design is nondegenerate. For example, suppose there exists \(\kappa>0\) such that
\[
\mathbb E_{P_S}\!\left[
\mathbf 1\{|Y-a^\top x|\le \delta\}aa^\top
\right]
\succeq
\kappa I,
\qquad \forall x\in \mathcal{X}.
\]
Then the source risk has curvature at least \(\kappa\) on \(\mathcal{X}\), and hence satisfies
\[
\mathcal E(x)\ge
\mathcal E(x_S)
+\frac{\kappa}{2}\|x-x_S\|_2^2,
\qquad \forall x\in\mathcal{X}.
\]
Thus Assumption~\ref{assump:strong convex and lipschitz}(b) holds. Intuitively, this condition rules out flat directions and the degenerate case where almost all residuals lie in the linear tail of the Huber loss.

\section{Proofs of Main Results}

\subsection{Proof of Proposition \ref{f-space, ERM bound}}

\textit{Proof.}
Let
\[
\mathcal A:=\{z\mapsto f(x,z): x\in\mathcal X\}.
\]
Recall that
\[
J_S:=\inf_{x\in\mathcal X}E_{P_S}[f(x,z)],
\qquad
J_T:=\inf_{x\in\mathcal X}E_{P_T}[f(x,z)].
\]
Let
\[
x_S\in\arg\min_{x\in\mathcal X}E_{P_S}[f(x,z)].
\]

By the Lipschitz continuity of $z\mapsto f(x,z)$ uniformly over
$x\in\mathcal X$, together with the Kantorovich--Rubinstein duality \ref{dual w-dis}, for any
$x\in\mathcal X$ we have
\[
\left|
E_{P_T}[f(x,z)]-E_{P_S}[f(x,z)]
\right|
\le
L\,d_W(P_S,P_T).
\]
Applying this inequality to $x=\hat x_{\mathrm{ERM}}$ gives
\begin{equation}
\label{eq:erm-shift-control}
E_{P_T}[f(\hat x_{\mathrm{ERM}},z)]
\le
E_{P_S}[f(\hat x_{\mathrm{ERM}},z)]
+
L\,d_W(P_S,P_T).
\end{equation}

We now control the source risk of $\hat x_{\mathrm{ERM}}$. Adding and
subtracting empirical risks gives
\begin{align*}
E_{P_S}[f(\hat x_{\mathrm{ERM}},z)]
&=
\left[
E_{P_S}f(\hat x_{\mathrm{ERM}},z)
-
E_{\hat P_n}f(\hat x_{\mathrm{ERM}},z)
\right] \\
&\quad+
\left[
E_{\hat P_n}f(\hat x_{\mathrm{ERM}},z)
-
E_{\hat P_n}f(x_S,z)
\right] \\
&\quad+
\left[
E_{\hat P_n}f(x_S,z)
-
E_{P_S}f(x_S,z)
\right]
+
E_{P_S}f(x_S,z).
\end{align*}
The middle term is nonpositive by the optimality of $\hat x_{\mathrm{ERM}}$.
Therefore,
\begin{equation}
\label{eq:erm-source-risk-decomp}
E_{P_S}[f(\hat x_{\mathrm{ERM}},z)]
\le
J_S
+
G_n
+
\left[
E_{\hat P_n}f(x_S,z)
-
E_{P_S}f(x_S,z)
\right],
\end{equation}
where
\[
G_n
:=
\sup_{x\in\mathcal X}
\left\{
E_{P_S}f(x,z)-E_{\hat P_n}f(x,z)
\right\}.
\]

By symmetrization and bounded-difference concentration, with probability at
least $1-\delta/2$,
\begin{equation}
\label{eq:erm-Gn-bound}
G_n
\le
2\mathfrak R_n(\mathcal A)
+
M\sqrt{\frac{\log(2/\delta)}{2n}}.
\end{equation}
For the fixed function $x_S$, Hoeffding's inequality gives, with probability at
least $1-\delta/2$,
\begin{equation}
\label{eq:erm-fixed-fS}
E_{\hat P_n}f(x_S,z)
-
E_{P_S}f(x_S,z)
\le
M\sqrt{\frac{\log(2/\delta)}{2n}}.
\end{equation}
Combining \eqref{eq:erm-source-risk-decomp}--\eqref{eq:erm-fixed-fS} by a union
bound, with probability at least $1-\delta$,
\[
E_{P_S}[f(\hat x_{\mathrm{ERM}},z)]
\le
J_S
+
2\mathfrak R_n(\mathcal A)
+
2M\sqrt{\frac{\log(2/\delta)}{2n}}.
\]
By Dudley's entropy integral bound,
\[
\mathfrak R_n(\mathcal A)
\le
\frac{12}{\sqrt n}\mathcal C(\mathcal A),
\]
where
\[
\mathcal C(\mathcal A)
:=
\int_0^\infty
\sqrt{
\log\mathcal N(\mathcal A,\|\cdot\|_\infty,u)
}\,du.
\]
Thus, with probability at least $1-\delta$,
\[
E_{P_S}[f(\hat x_{\mathrm{ERM}},z)]
\le
J_S
+
\frac{24}{\sqrt n}\mathcal C(\mathcal A)
+
2M\sqrt{\frac{\log(2/\delta)}{2n}}.
\]
Substituting this inequality into \eqref{eq:erm-shift-control} and subtracting
$J_T$ from both sides yields
\[
\mathcal{R}_{P_T}(\hat x_{\mathrm{ERM}})
\le
J_S-J_T
+
L\,d_W(P_S,P_T)
+
\frac{24}{\sqrt n}\mathcal C(\mathcal A)
+
2M\sqrt{\frac{\log(2/\delta)}{2n}}.
\]
This completes the proof.

\begin{proposition}[\cite{kantorovich1958space}]\label{dual w-dis}
For any distributions $\mathbb{Q}_1$, $\mathbb{Q}_2\in\mathcal{M}(\Xi)$, we have
\begin{align}\label{repre thm}  d_\mathrm{W}\big(\mathbb{Q}_1,\mathbb{Q}_2\big)=\sup_{f\in\mathcal{L}}\Big\{\:\int_{\Xi}f(\xi)\:\mathbb{Q}_1(\mathrm{d}\xi)-\int_{\Xi}f(\xi)\:\mathbb{Q}_2(\mathrm{d}\xi)\Big\}, 
\end{align}
where $\mathcal{L}$ denotes the space of all Lipschitz functions with 
$|f(\xi)-f(\xi')|\le\|\xi-\xi'\|\quad \text{for all } \xi,\xi'\in\Xi.$

\end{proposition}

\subsection{Proof of Proposition \ref{LB for ERM}}

\textit{Proof.}
Consider the $\ell_1$ loss
\[
f(x,z)=|1-zx|,
\qquad \mathcal X=[0,1],
\]
and the source distribution $P_S=\delta_1$. Since all source samples equal $1$,
we have $\hat P_n=\delta_1$ almost surely. The loss is Lipschitz in $z$,
uniformly over \(x\in[0,1]\), with Lipschitz constant \(L=1\).

The ERM solution is
\[
\hat x_{\mathrm{ERM}}
\in
\arg\min_{x\in[0,1]} E_{\hat P_n}|1-Zx|
=
\arg\min_{x\in[0,1]} |1-x|
=
\{1\}.
\]
Thus \(\hat x_{\mathrm{ERM}}=1\). Moreover,
\[
J_S
=
\inf_{x\in[0,1]} E_{P_S}|1-Zx|
=
\inf_{x\in[0,1]} |1-x|
=
0.
\]

For any target distribution $P_T$, the target loss of ERM is
\[
E_{P_T}\big[|1-Z\hat x_{\mathrm{ERM}}|\big]
=
E_{P_T}[|1-Z|].
\]
Since \(P_S=\delta_1\), every coupling between \(P_T\) and \(P_S\) must couple
\(Z\sim P_T\) with the constant random variable \(1\). Therefore,
\[
d_W(P_T,P_S)
=
d_W(P_T,\delta_1)
=
\int |z-1|\,dP_T(z)
=
E_{P_T}[|Z-1|].
\]
Hence
\[
E_{P_T}\big[|1-Z\hat x_{\mathrm{ERM}}|\big]
=
d_W(P_T,P_S).
\]

Finally, by definition,
\[
J_T
=
\inf_{x\in[0,1]} E_{P_T}|1-Zx|,
\]
and therefore
\[
\mathcal{R}_{P_T}(\hat x_{\mathrm{ERM}})
=
E_{P_T}\big[|1-Z\hat x_{\mathrm{ERM}}|\big]-J_T
=
d_W(P_T,P_S)-J_T.
\]
Since \(L=1\) and \(J_S=0\), this can be written as
\[
\mathcal{R}_{P_T}(\hat x_{\mathrm{ERM}})
=
J_S-J_T+L\,d_W(P_S,P_T).
\]
This proves the claim.

\subsection{Proof of Theorem \ref{DRO bound}}

\textit{Proof.}
Fix $r\ge 0$ and recall that
\[
\mathcal B(Q,r)
:=
\{P\in\mathcal P(\mathcal Z): d_W(P,Q)\le r\}.
\]
Let
\[
\hat x_{\mathrm{DRO}}
\in
\arg\min_{x\in\mathcal X}
\sup_{P\in\mathcal B(\hat P_n,r)}
E_P[f(x,z)],
\]
and let
\[
x_S\in\arg\min_{x\in\mathcal X}E_{P_S}[f(x,z)],
\qquad
x_T\in\arg\min_{x\in\mathcal X}E_{P_T}[f(x,z)].
\]
For any $Q\in\mathcal P(\mathcal Z)$, define
\[
 R_r(Q,x)
:=
\sup_{P\in\mathcal B(Q,r)}E_P[f(x,z)].
\]
Define the finite-sample discrepancy
\[
\tilde\Delta(x)
:=
\left|
 R_r(\hat P_n,x)
-
 R_r(P_S,x)
\right|.
\]

By adding and subtracting the empirical robust risk and using the optimality of
$\hat x_{\mathrm{DRO}}$, we have
\begin{align*}
&
 R_r(P_S,\hat x_{\mathrm{DRO}})
-
 R_r(P_S,x_S)
\\
&=
\left[
 R_r(P_S,\hat x_{\mathrm{DRO}})
-
 R_r(\hat P_n,\hat x_{\mathrm{DRO}})
\right]
+
\left[
 R_r(\hat P_n,\hat x_{\mathrm{DRO}})
-
\ R_r(\hat P_n,x_S)
\right]
\\
&\quad+
\left[
\ R_r(\hat P_n,x_S)
-
\ R_r(P_S,x_S)
\right] \\
&\le
\sup_{x\in\mathcal X}\tilde\Delta(x)
+
\tilde\Delta(x_S),
\end{align*}
because the middle term is nonpositive.

Next, decompose the target-environment excess loss:
\begin{align}
R_{P_T}(\hat x_{\mathrm{DRO}})
&=
E_{P_T}[f(\hat x_{\mathrm{DRO}},z)]
-
E_{P_T}[f(x_T,z)]
\nonumber\\
&=
\left[
E_{P_T}[f(\hat x_{\mathrm{DRO}},z)]
-
R_r(P_S,\hat x_{\mathrm{DRO}})
\right]
\nonumber\\
&\quad+
\left[
R_r(P_S,\hat x_{\mathrm{DRO}})
-
R_r(P_S,x_S)
\right]
\nonumber\\
&\quad+
\left[
R_r(P_S,x_S)
-
E_{P_S}[f(x_S,z)]
\right]
\nonumber\\
&\quad+
\left[
E_{P_S}[f(x_S,z)]
-
E_{P_T}[f(x_T,z)]
\right].
\label{eq:dro-target-decomposition}
\end{align}
The last term is $J_S-J_T$, and the third term is
\[
\Lambda_r(P_S,x_S)
:=
R_r(P_S,x_S)-E_{P_S}[f(x_S,z)].
\]
For the first term in \eqref{eq:dro-target-decomposition}, for any
$P\in\mathcal B(P_S,r)$,
\[
R_r(P_S,\hat x_{\mathrm{DRO}})
\ge
E_P[f(\hat x_{\mathrm{DRO}},z)].
\]
Therefore, by the Lipschitz continuity of $z\mapsto f(x,z)$ uniformly over
$x\in\mathcal X$,
\[
E_{P_T}[f(\hat x_{\mathrm{DRO}},z)]
-
R_r(P_S,\hat x_{\mathrm{DRO}})
\le
E_{P_T}[f(\hat x_{\mathrm{DRO}},z)]
-
E_P[f(\hat x_{\mathrm{DRO}},z)]
\le
L d_W(P_T,P).
\]
Taking the infimum over $P\in\mathcal B(P_S,r)$ gives
\[
E_{P_T}[f(\hat x_{\mathrm{DRO}},z)]
-
R_r(P_S,\hat x_{\mathrm{DRO}})
\le
L\inf_{P\in\mathcal B(P_S,r)}d_W(P_T,P).
\]
Combining the above inequalities yields
\begin{align}
R_{P_T}(\hat x_{\mathrm{DRO}})
&\le
J_S-J_T
+
L\inf_{P\in\mathcal B(P_S,r)}d_W(P_T,P)
+
\Lambda_r(P_S,x_S)
\nonumber\\
&\quad+
\sup_{x\in\mathcal X}\tilde\Delta(x)
+
\tilde\Delta(x_S).
\label{eq:dro-before-empirical-process}
\end{align}

It remains to control the two finite-sample terms. For $x\in\mathcal X$ and
$0\le k\le L$, define
\[
\phi_{x,k}(z)
:=
\sup_{z'\in\mathcal Z}
\left\{
f(x,z')-k\|z-z'\|
\right\},
\]
and define the envelope class
\[
\Phi
:=
\{\phi_{x,k}: x\in\mathcal X,\ 0\le k\le L\}.
\]
By the Kantorovich dual formulation of the Wasserstein robust risk,
\[
R_r(Q,f)
=
\inf_{k\ge 0}
\left\{
kr+E_Q[\phi_{x,k}(z)]
\right\}.
\]
Since $z\mapsto f(x,z)$ is $L$-Lipschitz, the infimum over $k\ge0$ can be
restricted to $0\le k\le L$. Indeed, for any $k\ge L$,
\[
f(x,z')-k\|z-z'\|
\le
f(x,z)+(L-k)\|z-z'\|
\le
f(x,z),
\]
and equality is attained at $z'=z$. Hence $\phi_{x,k}(z)=f(x,z)$ for all
$k\ge L$, so values $k>L$ cannot improve the dual objective.

Thus, for every $x\in\mathcal X$,
\[
\tilde\Delta(x)
\le
\sup_{0\le k\le L}
\left|
(E_{P_S}-E_{\hat P_n})\phi_{x,k}
\right|.
\]
Consequently,
\[
\sup_{x\in\mathcal X}\tilde\Delta(x)
\le
Z_n^{\mathrm{unif}},
\]
where
\[
Z_n^{\mathrm{unif}}
:=
\sup_{x\in\mathcal X,\ 0\le k\le L}
\left|
(E_{P_S}-E_{\hat P_n})\phi_{x,k}
\right|.
\]

Since $0\le \phi_{x,k}\le M$, symmetrization and bounded-difference concentration
give, with probability at least $1-\delta/2$,
\[
Z_n^{\mathrm{unif}}
\le
2\mathcal R_n(\Phi)
+
M\sqrt{\frac{\log(4/\delta)}{2n}}.
\]
By the envelope-class Rademacher complexity bound,
\[
\mathcal R_n(\Phi)
\le
\frac{24}{\sqrt n}\mathcal C(\mathcal A)
+
\frac{24L\,\operatorname{diam}(\mathcal Z)}{\sqrt n}.
\]
Therefore, with probability at least $1-\delta/2$,
\begin{equation}
\label{eq:dro-uniform-delta}
\sup_{x\in\mathcal X}\tilde\Delta(x)
\le
\frac{48}{\sqrt n}\mathcal C(\mathcal A)
+
\frac{48L\,\operatorname{diam}(\mathcal Z)}{\sqrt n}
+
M\sqrt{\frac{\log(4/\delta)}{2n}}.
\end{equation}

We now control the fixed term $\tilde\Delta(x_S)$. Define
\[
\Phi_{x_S}
:=
\{\phi_{x_S,k}:0\le k\le L\}.
\]
For any $k,k'\in[0,L]$ and $z\in\mathcal Z$,
\[
|\phi_{x_S,k}(z)-\phi_{x_S,k'}(z)|
\le
\operatorname{diam}(\mathcal Z)|k-k'|.
\]
Also, since $0\le f\le M$, we have $0\le \phi_{x_S,k}\le M$. Hence
$\Phi_{x_S}$ is a one-dimensional class indexed by $k\in[0,L]$. Dudley's entropy
integral gives
\[
\mathcal R_n(\Phi_{x_S})
\le
\frac{12L\,\operatorname{diam}(\mathcal Z)}{\sqrt n}.
\]
Again by symmetrization and bounded-difference concentration, with probability
at least $1-\delta/2$,
\begin{equation}
\label{eq:dro-fixed-delta}
\tilde\Delta(x_S)
\le
\frac{24L\,\operatorname{diam}(\mathcal Z)}{\sqrt n}
+
M\sqrt{\frac{\log(4/\delta)}{2n}}.
\end{equation}

Combining \eqref{eq:dro-uniform-delta} and \eqref{eq:dro-fixed-delta} by a
union bound, with probability at least $1-\delta$,
\[
\sup_{x\in\mathcal X}\tilde\Delta(x)+\tilde\Delta(x_S)
\le
\frac{48}{\sqrt n}\mathcal C(\mathcal A)
+
\frac{72L\,\operatorname{diam}(\mathcal Z)}{\sqrt n}
+
2M\sqrt{\frac{\log(4/\delta)}{2n}}.
\]
Substituting this bound into \eqref{eq:dro-before-empirical-process} gives
\[
\mathcal R_{P_T}(\hat x_{\mathrm{DRO}})
\le
J_S-J_T
+
L\inf_{P\in\mathcal B(P_S,r)}d_W(P_T,P)
+
\Lambda_r(P_S,x_S)
+
\frac{48}{\sqrt n}\mathcal C(\mathcal A)
+
\frac{72L\,\operatorname{diam}(\mathcal Z)}{\sqrt n}
+
2M\sqrt{\frac{\log(4/\delta)}{2n}}.
\]
This completes the proof.

\subsection{Proof of Theorem~\ref{RS bound}}

\textit{Proof.}
Recall that
\[
\tau_\epsilon
:=
\inf_{x\in\mathcal X}E_{\hat P_n}[f(x,z)]+\epsilon.
\]
Let $(\hat x_{\mathrm{RS}},k_{\tau_\epsilon})$ be an optimal solution of the RS
problem with reference value $\tau_\epsilon$. By feasibility of the RS solution,
for all $P\in\mathcal P(\mathcal Z)$,
\begin{equation}
\label{eq:rs-feasibility}
E_P[f(\hat x_{\mathrm{RS}},z)]
-
\tau_\epsilon
\le
k_{\tau_\epsilon}d_W(P,\hat P_n).
\end{equation}
Moreover, by Lemma~\ref{k leq L}, we have
\[
0\le k_{\tau_\epsilon}\le L.
\]

We first rewrite the feasibility condition in a dual form. Equation
\eqref{eq:rs-feasibility} implies
\[
\sup_{P\in\mathcal P(\mathcal Z)}
\left\{
E_P[f(\hat x_{\mathrm{RS}},z)]
-
k_{\tau_\epsilon}d_W(P,\hat P_n)
\right\}
\le
\tau_\epsilon.
\]
By the penalized Kantorovich dual representation, this is equivalent to
\begin{equation}
\label{eq:rs-dual-empirical}
E_{\hat P_n}
\left[
\sup_{y\in\mathcal Z}
\left\{
f(\hat x_{\mathrm{RS}},y)
-
k_{\tau_\epsilon}\|z-y\|
\right\}
\right]
\le
\tau_\epsilon.
\end{equation}

For $x\in\mathcal X$ and $0\le k\le L$, define
\[
\phi_{x,k}(z)
:=
\sup_{y\in\mathcal Z}
\left\{
f(x,y)-k\|z-y\|
\right\},
\]
and define the one-sided uniform deviation
\[
Z_n^{\mathrm{RS}}
:=
\sup_{x\in\mathcal X,\ 0\le k\le L}
\left\{
E_{P_S}[\phi_{x,k}(z)]
-
E_{\hat P_n}[\phi_{x,k}(z)]
\right\}.
\]
Since $(\hat x_{\mathrm{RS}},k_{\tau_\epsilon})$ belongs to this index set,
\eqref{eq:rs-dual-empirical} gives
\begin{equation}
\label{eq:rs-dual-source}
E_{P_S}[\phi_{\hat x_{\mathrm{RS}},k_{\tau_\epsilon}}(z)]
\le
\tau_\epsilon
+
Z_n^{\mathrm{RS}}.
\end{equation}

Applying the same penalized dual representation with nominal distribution
$P_S$, for any $P\in\mathcal P(\mathcal Z)$,
\[
E_P[f(\hat x_{\mathrm{RS}},z)]
-
k_{\tau_\epsilon}d_W(P,P_S)
\le
E_{P_S}[\phi_{\hat x_{\mathrm{RS}},k_{\tau_\epsilon}}(z)].
\]
Combining this inequality with \eqref{eq:rs-dual-source}, we obtain, for all
$P\in\mathcal P(\mathcal Z)$,
\begin{equation}
\label{eq:rs-general-P}
E_P[f(\hat x_{\mathrm{RS}},z)]
\le
\tau_\epsilon
+
k_{\tau_\epsilon}d_W(P,P_S)
+
Z_n^{\mathrm{RS}}.
\end{equation}
Taking $P=P_T$ gives
\begin{equation}
\label{eq:rs-target-loss}
E_{P_T}[f(\hat x_{\mathrm{RS}},z)]
\le
\tau_\epsilon
+
k_{\tau_\epsilon}d_W(P_T,P_S)
+
Z_n^{\mathrm{RS}}.
\end{equation}

Next, we use the definition of $\tau_\epsilon$. Let
\[
x_S\in\arg\min_{x\in\mathcal X}E_{P_S}[f(x,z)].
\]
Since
\[
\tau_\epsilon
=
\inf_{x\in\mathcal X}E_{\hat P_n}[f(x,z)]+\epsilon,
\]
we have
\begin{equation}
\label{eq:tau-upper}
\tau_\epsilon
\le
E_{\hat P_n}[f(x_S,z)]+\epsilon
=
J_S
+
\left[
E_{\hat P_n}f(x_S,z)
-
E_{P_S}f(x_S,z)
\right]
+
\epsilon.
\end{equation}
Substituting \eqref{eq:tau-upper} into \eqref{eq:rs-target-loss} and subtracting
\[
J_T:=\inf_{x\in\mathcal X}E_{P_T}[f(x,z)]
\]
from both sides yields
\begin{align}
R_{P_T}(\hat x_{\mathrm{RS}})
&=
E_{P_T}[f(\hat x_{\mathrm{RS}},z)]-J_T
\nonumber\\
&\le
J_S-J_T
+
k_{\tau_\epsilon}d_W(P_S,P_T)
+
\epsilon
\nonumber\\
&\quad+
\left[
E_{\hat P_n}f(x_S,z)
-
E_{P_S}f(x_S,z)
\right]
+
Z_n^{\mathrm{RS}}.
\label{eq:rs-before-stat}
\end{align}

It remains to control the two finite-sample terms in
\eqref{eq:rs-before-stat}. Since $0\le f\le M$, we have
$0\le \phi_{x,k}\le M$ for all $x\in\mathcal X$ and $0\le k\le L$.
By symmetrization and bounded-difference concentration, with probability at
least $1-\delta/2$,
\begin{equation}
\label{eq:rs-uniform-dev}
Z_n^{\mathrm{RS}}
\le
2\mathcal R_n(\Phi)
+
M\sqrt{\frac{\log(2/\delta)}{2n}},
\end{equation}
where
\[
\Phi
:=
\{\phi_{x,k}: x\in\mathcal X,\ 0\le k\le L\}.
\]
By the envelope-class Rademacher complexity bound,
\[
\mathcal R_n(\Phi)
\le
\frac{24}{\sqrt n}\mathcal C(\mathcal A)
+
\frac{24L\,\operatorname{diam}(\mathcal Z)}{\sqrt n}.
\]
Therefore, with probability at least $1-\delta/2$,
\begin{equation}
\label{eq:rs-uniform-final}
Z_n^{\mathrm{RS}}
\le
\frac{48}{\sqrt n}\mathcal C(\mathcal A)
+
\frac{48L\,\operatorname{diam}(\mathcal Z)}{\sqrt n}
+
M\sqrt{\frac{\log(2/\delta)}{2n}}.
\end{equation}

For the fixed function $x_S$, Hoeffding's inequality gives, with probability at
least $1-\delta/2$,
\begin{equation}
\label{eq:rs-fixed-fS}
E_{\hat P_n}f(x_S,z)
-
E_{P_S}f(x_S,z)
\le
M\sqrt{\frac{\log(2/\delta)}{2n}}.
\end{equation}
Combining \eqref{eq:rs-uniform-final} and \eqref{eq:rs-fixed-fS} by a union
bound, and substituting them into \eqref{eq:rs-before-stat}, we obtain, with
probability at least $1-\delta$,
\[
\mathcal R_{P_T}(\hat x_{\mathrm{RS}})
\le
J_S-J_T
+
k_{\tau_\epsilon}d_W(P_S,P_T)
+
\epsilon
+
\frac{48}{\sqrt n}\mathcal C(\mathcal A)
+
\frac{48L\,\operatorname{diam}(\mathcal Z)}{\sqrt n}
+
2M\sqrt{\frac{\log(2/\delta)}{2n}}.
\]
This completes the proof.

\subsection{Proof of Proposition \ref{prop:known_shift}}

\textit{Proof.}
Recall that in Scenario I the shift magnitude is known, so the DRO radius is
chosen as
\[
r=d_W(P_S,P_T).
\]
Therefore \(P_T\in\mathcal B(P_S,r)\), and hence
\[
\mathrm{Sen}_{\mathrm{DRO}}(r)
=
L\inf_{P\in\mathcal B(P_S,r)}d_W(P_T,P)
=
0.
\]
On the other hand,
\[
\mathrm{Sen}_{\mathrm{RS}}(\tau_r)
=
k_{\tau_r}d_W(P_S,P_T)
=
k_{\tau_r}r.
\]
This proves the sensitivity comparison.

We now compare the regularization terms. For any
\(Q\in\mathcal P(\mathcal Z)\), \(x\in\mathcal X\), and \(r\ge0\), define
\[
\Delta_r(Q,f)
:=
\sup_{P\in\mathcal B(Q,r)}E_P[f(x,z)]
-
E_Q[f(x,z)].
\]
Then
\[
\mathrm{Reg}_{\mathrm{DRO}}(r)
=
\Delta_r(P_S,x_S).
\]
Moreover, under the calibration
\[
\tau_r
:=
\sup_{P\in\mathcal B(\hat P_n,r)}
E_P[f(\hat x_{\mathrm{ERM}},z)],
\]
and since \(\hat x_{\mathrm{ERM}}\) minimizes the empirical risk,
\[
\mathrm{Reg}_{\mathrm{RS}}(\tau_r)
=
\tau_r-E_{\hat P_n}[f(\hat x_{\mathrm{ERM}},z)]
=
\Delta_r(\hat P_n,\hat x_{\mathrm{ERM}}).
\]
Therefore it suffices to control
\[
\left|
\Delta_r(P_S,x_S)
-
\Delta_r(\hat P_n,\hat x_{\mathrm{ERM}})
\right|.
\]

By the triangle inequality,
\begin{align}
&
\left|
\Delta_r(P_S,x_S)
-
\Delta_r(\hat P_n,\hat x_{\mathrm{ERM}})
\right|
\nonumber\\
&\le
\left|
\Delta_r(P_S,x_S)-\Delta_r(\hat P_n,x_S)
\right|
+
\left|
\Delta_r(\hat P_n,x_S)
-
\Delta_r(\hat P_n,\hat x_{\mathrm{ERM}})
\right|.
\label{eq:prop3-triangle}
\end{align}

We first control the first term in \eqref{eq:prop3-triangle}. Write
\[
R_r(Q,x)
:=
\sup_{P\in\mathcal B(Q,r)}E_P[f(x,z)].
\]
Then
\[
\Delta_r(Q,x)=R_r(Q,x)-E_Q[f(x,z)].
\]
Thus
\begin{align}
&
\left|
\Delta_r(P_S,x_S)-\Delta_r(\hat P_n,x_S)
\right|
\nonumber\\
&\le
\left|
R_r(P_S,x_S)-R_r(\hat P_n,x_S)
\right|
+
\left|
(E_{P_S}-E_{\hat P_n})f(x_S,z)
\right|.
\label{eq:prop3-first-term}
\end{align}
For the robust-risk discrepancy, define
\[
\phi_{x_S,k}(z)
:=
\sup_{y\in\mathcal Z}
\left\{
f(x_S,y)-k\|z-y\|
\right\},
\qquad 0\le k\le L.
\]
By the Kantorovich dual formulation of the Wasserstein robust risk,
\[
\left|
R_r(P_S,x_S)-R_r(\hat P_n,x_S)
\right|
\le
\sup_{0\le k\le L}
\left|
(E_{P_S}-E_{\hat P_n})\phi_{x_S,k}
\right|.
\]
Since \(0\le \phi_{x_S,k}\le M\) and
\[
|\phi_{x_S,k}(z)-\phi_{x_S,k'}(z)|
\le
\operatorname{diam}(\mathcal Z)|k-k'|,
\]
the same one-dimensional entropy argument as in the proof of
Theorem~\ref{DRO bound} gives, with probability at least \(1-\delta/4\),
\[
\left|
R_r(P_S,x_S)-R_r(\hat P_n,x_S)
\right|
\le
\frac{24L\,\operatorname{diam}(\mathcal Z)}{\sqrt n}
+
M\sqrt{\frac{\log(8/\delta)}{2n}}.
\]
Also, by Hoeffding's inequality, with probability at least \(1-\delta/4\),
\[
\left|
(E_{P_S}-E_{\hat P_n})f(x_S,z)
\right|
\le
M\sqrt{\frac{\log(8/\delta)}{2n}}.
\]
Combining the last two displays, with probability at least \(1-\delta/2\),
\begin{equation}
\label{eq:prop3-fixed-Delta}
\left|
\Delta_r(P_S,x_S)-\Delta_r(\hat P_n,x_S)
\right|
\le
\frac{24L\,\operatorname{diam}(\mathcal Z)}{\sqrt n}
+
2M\sqrt{\frac{\log(8/\delta)}{2n}}.
\end{equation}

We next control the second term in \eqref{eq:prop3-triangle}. By the
parameter-Lipschitz condition,
\[
\left|
R_r(\hat P_n,x_S)
-
R_r(\hat P_n,\hat x_{\mathrm{ERM}})
\right|
\le
L'\|x_S-\hat x_{\mathrm{ERM}}\|_{\mathcal X},
\]
and
\[
\left|
E_{\hat P_n}f(x_S,z)
-
E_{\hat P_n}f(\hat x_{\mathrm{ERM}},z)
\right|
\le
L'\|x_S-\hat x_{\mathrm{ERM}}\|_{\mathcal X}.
\]
Therefore,
\begin{equation}
\label{eq:prop3-param}
\left|
\Delta_r(\hat P_n,x_S)
-
\Delta_r(\hat P_n,\hat x_{\mathrm{ERM}})
\right|
\le
2L'\|\hat x_{\mathrm{ERM}}-x_S\|_{\mathcal X}.
\end{equation}

It remains to control \(\|\hat x_{\mathrm{ERM}}-x_S\|_{\mathcal X}\). Define
\[
\mathfrak G_n(\delta)
:=
\frac{24}{\sqrt n}\mathcal C(\mathcal A)
+
M\sqrt{\frac{\log(8/\delta)}{2n}}.
\]
By uniform convergence, with probability at least \(1-\delta/4\),
\[
\sup_{x\in\mathcal X}
\left|
E_{P_S}f(x,z)-E_{\hat P_n}f(x,z)
\right|
\le
\mathfrak G_n(\delta).
\]
On this event, the optimality of \(\hat x_{\mathrm{ERM}}\) gives
\[
E_{P_S}f(\hat x_{\mathrm{ERM}},z)
-
E_{P_S}f(x_S,z)
\le
2\mathfrak G_n(\delta).
\]
By the strong convexity condition in Assumption~\ref{assump:strong convex and lipschitz},
\[
\frac{\alpha}{2}
\|\hat x_{\mathrm{ERM}}-x_S\|_{\mathcal X}^2
\le
E_{P_S}f(\hat x_{\mathrm{ERM}},z)
-
E_{P_S}f(x_S,z).
\]
Hence
\begin{equation}
\label{eq:prop3-erm-stability}
\|\hat x_{\mathrm{ERM}}-x_S\|_{\mathcal X}
\le
2\sqrt{\frac{\mathfrak G_n(\delta)}{\alpha}}.
\end{equation}
Combining \eqref{eq:prop3-param} and \eqref{eq:prop3-erm-stability}, we get
\[
\left|
\Delta_r(\hat P_n,x_S)
-
\Delta_r(\hat P_n,\hat x_{\mathrm{ERM}})
\right|
\le
4L'\sqrt{\frac{\mathfrak G_n(\delta)}{\alpha}}.
\]

Finally, combining this inequality with \eqref{eq:prop3-fixed-Delta} in
\eqref{eq:prop3-triangle}, and applying a union bound over the events above,
we obtain, with probability at least \(1-\delta\),
\[
\left|
\mathrm{Reg}_{\mathrm{DRO}}(r)
-
\mathrm{Reg}_{\mathrm{RS}}(\tau_r)
\right|
\le
\rho_n(\delta),
\]
where
\[
\rho_n(\delta)
:=
4L'\sqrt{\frac{\mathfrak G_n(\delta)}{\alpha}}
+
\frac{24L\,\operatorname{diam}(\mathcal Z)}{\sqrt n}
+
2M\sqrt{\frac{\log(8/\delta)}{2n}}.
\]
Equivalently,
\[
\mathrm{Reg}_{\mathrm{RS}}(\tau_r)-\rho_n(\delta)
\le
\mathrm{Reg}_{\mathrm{DRO}}(r)
\le
\mathrm{Reg}_{\mathrm{RS}}(\tau_r)+\rho_n(\delta).
\]
This proves the regularization comparison and completes the proof.

\subsection{Proof of Proposition~\ref{prop:reg_comparison}}

\textit{Proof.}
Recall that in Scenario II the calibrated DRO radius and RS threshold are
\[
r_t=d_W(P_S,P_t),
\qquad
\tau_t=E_{P_t}[f(\hat x_{\mathrm{ERM}},z)].
\]
Define
\[
\Gamma_t
:=
\sup_{P\in\mathcal B(P_S,r_t)}
E_P[f(x_S,z)]
-
E_{P_t}[f(x_S,z)].
\]
Since \(P_t\in\mathcal B(P_S,r_t)\), we have \(\Gamma_t\ge 0\).

By definition,
\[
\mathrm{Reg}_{\mathrm{DRO}}(r_t)
=
\sup_{P\in\mathcal B(P_S,r_t)}
E_P[f(x_S,z)]
-
E_{P_S}[f(x_S,z)].
\]
Therefore,
\begin{equation}
\label{eq:prop4-dro-minus-gap}
\mathrm{Reg}_{\mathrm{DRO}}(r_t)-\Gamma_t
=
E_{P_t}[f(x_S,z)]
-
E_{P_S}[f(x_S,z)].
\end{equation}
On the other hand, under the calibration
\(\tau_t=E_{P_t}[f(\hat x_{\mathrm{ERM}},z)]\), the RS regularization term is
\begin{equation}
\label{eq:prop4-rs-reg}
\mathrm{Reg}_{\mathrm{RS}}(\tau_t)
=
\tau_t-E_{\hat P_n}[f(\hat x_{\mathrm{ERM}},z)]
=
E_{P_t}[f(\hat x_{\mathrm{ERM}},z)]
-
E_{\hat P_n}[f(\hat x_{\mathrm{ERM}},z)].
\end{equation}

Combining \eqref{eq:prop4-dro-minus-gap} and \eqref{eq:prop4-rs-reg}, we have
\begin{align}
&
\mathrm{Reg}_{\mathrm{RS}}(\tau_t)
-
\left(
\mathrm{Reg}_{\mathrm{DRO}}(r_t)-\Gamma_t
\right)
\nonumber\\
&=
E_{P_t}[f(\hat x_{\mathrm{ERM}},z)-f(x_S,z)]
+
E_{P_S}[f(x_S,z)]
-
E_{\hat P_n}[f(\hat x_{\mathrm{ERM}},z)].
\label{eq:prop4-main-diff}
\end{align}

We now upper bound the two terms on the right-hand side. By the
parameter-Lipschitz condition in Assumption~\ref{assump:strong convex and lipschitz},
\begin{equation}
\label{eq:prop4-pt-term}
E_{P_t}[f(\hat x_{\mathrm{ERM}},z)-f(x_S,z)]
\le
L'\|\hat x_{\mathrm{ERM}}-x_S\|_{\mathcal X}.
\end{equation}
Next define
\[
U_n
:=
\sup_{x\in\mathcal X}
\left|
E_{P_S}[f(x,z)]-E_{\hat P_n}[f(x,z)]
\right|.
\]
Since \(x_S\) minimizes the population risk,
\[
E_{P_S}[f(x_S,z)]
-
E_{P_S}[f(\hat x_{\mathrm{ERM}},z)]
\le 0.
\]
Thus
\begin{equation}
\label{eq:prop4-emp-term}
E_{P_S}[f(x_S,z)]
-
E_{\hat P_n}[f(\hat x_{\mathrm{ERM}},z)]
\le
E_{P_S}[f(\hat x_{\mathrm{ERM}},z)]
-
E_{\hat P_n}[f(\hat x_{\mathrm{ERM}},z)]
\le
U_n.
\end{equation}

By symmetrization, bounded-difference concentration, and Dudley's entropy
integral bound, with probability at least \(1-\delta\),
\begin{equation}
\label{eq:prop4-uniform-conv}
U_n
\le
\bar{\mathfrak G}_n(\delta)
:=
\frac{24}{\sqrt n}\mathcal C(\mathcal A)
+
M\sqrt{\frac{\log(2/\delta)}{2n}}.
\end{equation}
On the same event, the empirical optimality of \(\hat x_{\mathrm{ERM}}\) implies
\[
E_{P_S}[f(\hat x_{\mathrm{ERM}},z)]
-
E_{P_S}[f(x_S,z)]
\le
2\bar{\mathfrak G}_n(\delta).
\]
By the strong convexity condition in Assumption~\ref{assump:strong convex and lipschitz},
\[
\frac{\alpha}{2}
\|\hat x_{\mathrm{ERM}}-x_S\|_{\mathcal X}^2
\le
E_{P_S}[f(\hat x_{\mathrm{ERM}},z)]
-
E_{P_S}[f(x_S,z)].
\]
Therefore,
\begin{equation}
\label{eq:prop4-stability}
\|\hat x_{\mathrm{ERM}}-x_S\|_{\mathcal X}
\le
2\sqrt{\frac{\bar{\mathfrak G}_n(\delta)}{\alpha}}.
\end{equation}

Combining \eqref{eq:prop4-main-diff}--\eqref{eq:prop4-stability}, with
probability at least \(1-\delta\),
\[
\mathrm{Reg}_{\mathrm{RS}}(\tau_t)
-
\left(
\mathrm{Reg}_{\mathrm{DRO}}(r_t)-\Gamma_t
\right)
\le
2L'\sqrt{\frac{\bar{\mathfrak G}_n(\delta)}{\alpha}}
+
\bar{\mathfrak G}_n(\delta).
\]
Define
\[
\bar\rho_n(\delta)
:=
2L'\sqrt{\frac{\bar{\mathfrak G}_n(\delta)}{\alpha}}
+
\bar{\mathfrak G}_n(\delta).
\]
Then
\[
\mathrm{Reg}_{\mathrm{RS}}(\tau_t)+\Gamma_t
\le
\mathrm{Reg}_{\mathrm{DRO}}(r_t)+\bar\rho_n(\delta),
\]
which proves the claim.

\subsection{Proof of Proposition~\ref{prop:underspecified}}

\textit{Proof.}
Recall that
\[
r_t=d_W(P_S,P_t),
\]
and
\[
\mathrm{Sen}_{\mathrm{DRO}}(r_t)
=
L\inf_{P\in\mathcal B(P_S,r_t)}d_W(P_T,P),
\qquad
\mathrm{Sen}_{\mathrm{RS}}(\tau_t)
=
k_{\tau_t}d_W(P_S,P_T).
\]

We first prove the under-specified case. For any
\(P\in\mathcal B(P_S,r_t)\), the triangle inequality gives
\[
d_W(P_T,P)
\ge
d_W(P_T,P_S)-d_W(P,P_S)
\ge
d_W(P_T,P_S)-r_t.
\]
Therefore,
\[
\mathrm{Sen}_{\mathrm{DRO}}(r_t)
\ge
L\left(d_W(P_T,P_S)-r_t\right).
\]
Since \(r_t=d_W(P_S,P_t)\), the condition
\[
\frac{d_W(P_t,P_S)}{d_W(P_T,P_S)}
\le
1-\frac{k_{\tau_t}}{L}
\]
is equivalent to
\[
L\left(d_W(P_T,P_S)-r_t\right)
\ge
k_{\tau_t}d_W(P_T,P_S).
\]
Hence
\[
\mathrm{Sen}_{\mathrm{DRO}}(r_t)
\ge
k_{\tau_t}d_W(P_T,P_S)
=
\mathrm{Sen}_{\mathrm{RS}}(\tau_t),
\]
which proves
\[
\mathrm{Sen}_{\mathrm{RS}}(\tau_t)
\le
\mathrm{Sen}_{\mathrm{DRO}}(r_t).
\]

We next prove the well-specified or over-specified case. Suppose \(t\ge t_T\).
Since \(P_T=P_{t_T}\), Assumption~\ref{monotonicity} gives
\[
d_W(P_S,P_T)
=
d_W(P_S,P_{t_T})
\le
d_W(P_S,P_t)
=
r_t.
\]
Therefore \(P_T\in\mathcal B(P_S,r_t)\). Hence
\[
\mathrm{Sen}_{\mathrm{DRO}}(r_t)
=
L\inf_{P\in\mathcal B(P_S,r_t)}d_W(P_T,P)
=
0.
\]
Since \(\mathrm{Sen}_{\mathrm{RS}}(\tau_t)\ge0\), we have
\[
\mathrm{Sen}_{\mathrm{DRO}}(r_t)
\le
\mathrm{Sen}_{\mathrm{RS}}(\tau_t).
\]
This completes the proof.

\subsection{Proof of Proposition~\ref{prop:well_overspefified_case}}

\textit{Proof.}
We first compare the regularization terms. Recall that
\[
r_t=d_W(P_S,P_t),
\qquad
\tau_t=E_{P_t}[f(\hat x_{\mathrm{ERM}},z)].
\]
In the adversarial scenario, \(P_t\) is a worst-case distribution in
\(\mathcal B(P_S,r_t)\) for the source optimizer \(x_S\). Hence
\[
\sup_{P\in\mathcal B(P_S,r_t)}E_P[f(x_S,z)]
=
E_{P_t}[f(x_S,z)].
\]
Therefore,
\[
\mathrm{Reg}_{\mathrm{DRO}}(r_t)
=
E_{P_t}[f(x_S,z)]
-
E_{P_S}[f(x_S,z)].
\]
On the other hand,
\[
\mathrm{Reg}_{\mathrm{RS}}(\tau_t)
=
\tau_t-E_{\hat P_n}[f(\hat x_{\mathrm{ERM}},z)]
=
E_{P_t}[f(\hat x_{\mathrm{ERM}},z)]
-
E_{\hat P_n}[f(\hat x_{\mathrm{ERM}},z)].
\]
Thus,
\begin{align}
&
\mathrm{Reg}_{\mathrm{DRO}}(r_t)
-
\mathrm{Reg}_{\mathrm{RS}}(\tau_t)
\nonumber\\
&=
E_{P_t}[f(x_S,z)-f(\hat x_{\mathrm{ERM}},z)]
+
E_{\hat P_n}[f(\hat x_{\mathrm{ERM}},z)]
-
E_{P_S}[f(x_S,z)].
\label{eq:prop6-reg-diff}
\end{align}
By the parameter-Lipschitz condition,
\[
E_{P_t}[f(x_S,z)-f(\hat x_{\mathrm{ERM}},z)]
\le
L'\|\hat x_{\mathrm{ERM}}-x_S\|_{\mathcal X}.
\]
Moreover, by the empirical optimality of \(\hat x_{\mathrm{ERM}}\),
\[
E_{\hat P_n}[f(\hat x_{\mathrm{ERM}},z)]
\le
E_{\hat P_n}[f(x_S,z)].
\]
Hence, with
\[
U_n
:=
\sup_{x\in\mathcal X}
\left|
E_{P_S}[f(x,z)]-E_{\hat P_n}[f(x,z)]
\right|,
\]
we have
\[
E_{\hat P_n}[f(\hat x_{\mathrm{ERM}},z)]
-
E_{P_S}[f(x_S,z)]
\le
E_{\hat P_n}[f(x_S,z)]
-
E_{P_S}[f(x_S,z)]
\le
U_n.
\]
Using the uniform convergence and stability bounds from
\eqref{eq:prop4-uniform-conv} and \eqref{eq:prop4-stability}, with probability
at least \(1-\delta\),
\[
U_n\le \bar{\mathfrak G}_n(\delta),
\qquad
\|\hat x_{\mathrm{ERM}}-x_S\|_{\mathcal X}
\le
2\sqrt{\frac{\bar{\mathfrak G}_n(\delta)}{\alpha}}.
\]
Substituting these two bounds into \eqref{eq:prop6-reg-diff}, we obtain
\[
\mathrm{Reg}_{\mathrm{DRO}}(r_t)
-
\mathrm{Reg}_{\mathrm{RS}}(\tau_t)
\le
2L'\sqrt{\frac{\bar{\mathfrak G}_n(\delta)}{\alpha}}
+
\bar{\mathfrak G}_n(\delta)
=
\bar\rho_n(\delta).
\]
Therefore,
\[
\mathrm{Reg}_{\mathrm{DRO}}(r_t)
\le
\mathrm{Reg}_{\mathrm{RS}}(\tau_t)+\bar\rho_n(\delta).
\]

It remains to compare the sensitivity terms. Recall that
\[
\mathrm{Sen}_{\mathrm{DRO}}(r_t)
=
L\inf_{P\in\mathcal B(P_S,r_t)}d_W(P_T,P),
\qquad
\mathrm{Sen}_{\mathrm{RS}}(\tau_t)
=
k_{\tau_t}d_W(P_S,P_T).
\]

For case \emph{(i)}, suppose \(t<t_T\) and
\[
\inf_{0\le s\le t}d_W(P_T,P_s)
\le
\frac{k_{\tau_t}}{L}d_W(P_S,P_T).
\]
For any \(0\le s\le t\), Assumption~\ref{monotonicity} gives
\[
d_W(P_S,P_s)\le d_W(P_S,P_t)=r_t,
\]
so \(P_s\in\mathcal B(P_S,r_t)\). Therefore,
\[
\inf_{P\in\mathcal B(P_S,r_t)}d_W(P_T,P)
\le
\inf_{0\le s\le t}d_W(P_T,P_s).
\]
It follows that
\[
\mathrm{Sen}_{\mathrm{DRO}}(r_t)
\le
L\inf_{0\le s\le t}d_W(P_T,P_s)
\le
k_{\tau_t}d_W(P_S,P_T)
=
\mathrm{Sen}_{\mathrm{RS}}(\tau_t).
\]

For case \emph{(ii)}, suppose \(t\ge t_T\). Since \(P_T=P_{t_T}\), Assumption~\ref{monotonicity}
implies
\[
d_W(P_S,P_T)=d_W(P_S,P_{t_T})
\le
d_W(P_S,P_t)
=
r_t.
\]
Thus \(P_T\in\mathcal B(P_S,r_t)\), and hence
\[
\mathrm{Sen}_{\mathrm{DRO}}(r_t)
=
L\inf_{P\in\mathcal B(P_S,r_t)}d_W(P_T,P)
=
0
\le
\mathrm{Sen}_{\mathrm{RS}}(\tau_t).
\]
This proves the sensitivity comparison in both cases.

Combining the regularization and sensitivity comparisons completes the proof.

\end{APPENDICES}

\end{document}